\documentclass{article}

\PassOptionsToPackage{numbers, compress}{natbib}
\usepackage[main, final]{neurips_2026}

\usepackage[utf8]{inputenc}
\usepackage[T1]{fontenc}
\usepackage[hidelinks]{hyperref}
\usepackage{bm}
\usepackage{url}
\usepackage{booktabs}
\usepackage{comment}
\usepackage{amsfonts}
\usepackage{amsmath}
\usepackage{amssymb}
\usepackage{amsthm}
\usepackage{nicefrac}
\usepackage{microtype}
\usepackage{xcolor}
\usepackage{graphicx}
\usepackage{subcaption}
\usepackage{algorithm}
\usepackage{multirow}
\usepackage{array}
\usepackage{pifont}
\usepackage{helvet}
\usepackage[colorinlistoftodos]{todonotes}
\usepackage{enumitem}

\usepackage{stfloats}
\usepackage{algpseudocode}
\usepackage{booktabs, tabularx}

\usepackage[acronym, sort=none, hyperfirst=false]{glossaries}
\makeatletter

\newcommand*{\TIDE@acranchor}[2]{%
  \ifcsdef{TIDE@acranchor@#1}{#2}{\global\csdef{TIDE@acranchor@#1}{}\glstarget{#1}{#2}}}
\renewcommand*{\glstextformat}[1]{\ifglsused{\glslabel}{#1}{\TIDE@acranchor{\glslabel}{#1}}}

\AtBeginDocument{\ifdefined\AtCaptionSingleLineCheck
  \AtCaptionSingleLineCheck{\let\TIDE@acranchor\@secondoftwo}\fi}
\makeatother

\newacronym{dl}{DL}{Deep Learning}
\newacronym{nn}{NN}{Neural Network}
\newacronym{ml}{ML}{Machine Learning}
\newacronym{gm}{GM}{Generative Model}

\newacronym{vit}{ViT}{Vision Transformer}
\newacronym{ei}{E-I}{Excitatory-Inhibitory}
\newacronym{ctm}{CTM}{Continuous Thought Machine}
\newacronym{cnn}{CNN}{Convolutional Neural Network}
\newacronym{hrf}{HRF}{Hierarchical Receptive Field}
\newacronym{lds}{LDS}{Lyapunov Diagonal Stability}
\newacronym{tide}{TIDE}{Temporal Inhibitory-Excitatory Dynamic Engine}

\newacronym{nlm}{NLM}{Neuron-level Model}
\newacronym{ln}{LN}{LayerNorm}
\newacronym{rl}{RL}{Reinforcement Learning}
\newacronym{mdp}{MDP}{Markov Decision Process}
\newacronym{mlp}{MLP}{Multilayer Perceptron}
\newacronym{bptt}{BPTT}{Backpropagation Through Time}
\newacronym{dog}{DoG}{Difference-of-Gaussian}
\newacronym{ood}{OOD}{Out-of-Distribution}
\newacronym{ode}{ODE}{Ordinary Differential Equation}
\newacronym{ddp}{DDP}{Distributed Data Parallel}
\newacronym{wta}{WTA}{Winner-Take-All}
\newacronym{svd}{SVD}{Single Value Decomposition}
\newacronym{rnd}{RND}{Random Network Distillation}
\newacronym{td}{TD}{Temporal Difference}
\newacronym{mse}{MSE}{Mean Squared Error}
\newacronym{sota}{SOTA}{State of the Art}

\usepackage{tikz}
\usepackage{standalone}
\usetikzlibrary{shapes.geometric, arrows.meta, positioning, calc, fit, backgrounds, decorations.pathreplacing}

\usepackage{dsfont}
\IfFileExists{bbm.sty}{\usepackage{bbm}}{}
\usepackage{nicefrac}
\DeclareMathOperator*{\argmin}{arg\,min}
\DeclareMathOperator*{\argmax}{arg\,max}
\DeclareMathSizes{6.2}{6.2}{5}{5}

\definecolor{excitgreen}{RGB}{34,139,34}
\definecolor{inhibred}{RGB}{178,34,34}
\definecolor{midpurple}{RGB}{128,0,128}
\definecolor{darkpurple}{RGB}{75,0,130}
\definecolor{midcyan}{RGB}{0,139,139}
\definecolor{midred}{RGB}{205,92,92}
\definecolor{midmagenta}{RGB}{199,21,133}
\definecolor{darkyellow}{RGB}{184,134,11}
\definecolor{midblue}{RGB}{70,130,180}

\theoremstyle{plain}
\theoremstyle{remark}
\newtheorem{theorem}{Theorem}[section]
\newtheorem{lemma}[theorem]{Lemma}
\newtheorem{remark}[theorem]{Remark}
\newtheorem{corollary}[theorem]{Corollary}
\newtheorem{definition}[theorem]{Definition}
\newtheorem{proposition}[theorem]{Proposition}

\newcommand{\RR}{\mathbb{R}}

\newcommand{\one}{\mathbbm{1}}
\newcommand{\T}{\!^{\top}\!}

\newcommand{\Proj}{\Pi}
\newcommand{\Wee}{W_{\!E\!E}}
\newcommand{\Wei}{W_{\!E\!I}}
\newcommand{\Wie}{W_{\!I\!E}}
\newcommand{\Wii}{W_{\!I\!I}}

\newcommand{\perr}{\lambda_{\mathrm{P}}}
\newcommand{\persurr}{\widehat\lambda_{\mathrm{P}}}
\newcommand{\justif}[1]{\quad\text{\scriptsize$\because$\ #1}}

\title{TIDE: Asymmetric Neural Circuits for Stabilized \\Temporal Inhibitory-Excitatory Dynamics}

\author{
  Alexander Kyuroson \\
  Luleå University of Technology,\\
  Luleå, Sweden \\ \\
  \texttt{akyuroson@gmail.com} \\
  \And
  Denis Kleyko \\
  Örebro University \&  \\
  RISE Research  \\
  Institutes of Sweden  \\
  \texttt{denis.kleyko@oru.se} \\
  \And
  Marcus Liwicki\\
  Luleå University of Technology,\\
  Luleå, Sweden \\ \\
  \texttt{marcus.liwicki@ltu.se} \\
}

\begin{document}
\maketitle

\begin{abstract}

Recent Continuous Thought Machine (CTM) architecture decouples internal computation from external inputs via neural dynamics, but relies on unconstrained multi-layer perceptrons for its recurrent computations. We propose to model neural dynamics using asymmetric Excitatory-Inhibitory (E-I) networks, which can be stabilized via principles from network theory, while obeying Dale's principle and can be expressed as energy-based systems optimized through a game-theoretic loss.
Building on this perspective, we introduce Temporal Inhibitory-Excitatory Dynamic Engine (TIDE), a neuro-inspired architecture that computes internal representations through neural dynamics stabilized by incorporating the Wilson-Cowan dynamics and lateral inhibition. TIDE balances biological realism by, for instance, using Hierarchical Receptive Fields and enforcing Dale's principle to ensure a realistic $80\!:\!20$ E-I balance ratio with an end-to-end trainable architecture. The aim of this paper is to introduce a new architecture that brings neuro-inspired learning and architectural principles to the forefront. 
We present proofs of convergence, stability, and complexity bounds, along with empirical ablation studies. Overall, TIDE achieves higher \texttt{top-1} accuracy than CTM on the majority of the investigated datasets while requiring a smaller training budget. Furthermore, across varying training budgets, TIDE improves \texttt{top-1} accuracy on ImageNet-C by an average of $+1.65$\% to $+6.42$\% under different perturbation severity levels.

\end{abstract}
\glsresetall

\section{Introduction}
\label{sec:intro}
Recent advances in machine learning driven by \glspl{cnn}~\cite{lecun1998gradient, krizhevsky2012imagenet, he2016deep}, Transformers~\cite{vaswani2017attention}, and \glspl{vit}~\cite{dosovitskiy2021image} have significantly improved representation learning, scalability, and cross-domain generalization.
At the same time, these architectures rely on computational principles that differ widely from those of biological neural systems, which are commonly characterized by sparsity, event-driven activation, population-level dynamics, and asymmetric \gls{ei} networks~\cite{ChiYan2022CVNN}.
The absence of such principles could also be associated with shortcomings of modern deep learning models, like their brittleness under distribution shift, limited cross-modal generalization, and sample inefficiency relative to their biological counterparts~\cite{lillicrap2016random, scellier2017equilibrium, payeur2021burst}.

It has been suggested that structure and asymmetrical connections can facilitate more efficient learning in neural systems~\cite{zador2019critique, isaacson2011inhibition}. Recent work~\cite{haber2022daleian, cornford2021dales} has demonstrated that Daleian-based networks are more robust against perturbations and are capable of efficient learning via single-neuron feature optimization. This indicates that biological mechanisms may play a significant role in designing novel architectures. Furthermore, it has been shown that the $80\!:\!20$ ratio for \gls{ei} population is vital not only for energy efficiency but also for robustness in continual learning and dynamic stability~\cite{haider2006neocortical}.

By introducing neural dynamics as a main primitive, \gls{ctm} architecture (Figure~\ref{fig:tide-vs-ctm}) attempts to decouple an internal reasoning process from the external input data, thus introducing recurrence via internal computation steps that replace static activations with per-neuron \glspl{nlm}, which elevates synchronization from an emergent statistic to an explicit latent representation for continuous reasoning~\cite{darlow2025continuous}. Although \gls{ctm} uses recurrence for internal computation, it does not leverage \gls{ei} dynamics and only relies on standard \gls{mlp} to represent its recurrent connections, which typically do not obey Dale's principle~\cite{dale1935pharmacology,eccles1954cholinergic}.

We introduce \gls{tide} architecture that evolves the ideas of \gls{ctm} by using the Wilson-Cowan model for neural dynamics~\cite{betteti2025competition}, leading to a recurrent model in which every connection is Dale-constrained, and neuron populations interact based on an asynchronous asymmetric model with distinct time constants for different types of neurons. 
Our main contributions within \gls{tide} can be summarized as follows: 
\begin{itemize}[itemsep=1pt, topsep=0pt]
    \item A novel neuro-inspired recurrent architecture with non-negative weights via the enforcement of the Wilson-Cowan dynamics.
    \item \gls{hrf} implementation and its adoption as a feature extractor for end-to-end training and inference based on a neuro-inspired model.
    \item Integration of a differentiable spectral stabilizer based on the Perron-eigenvalue sum-ratio to enable online \gls{lds} verification during training.
    \item Inclusion of population-specific \glspl{nlm} with distinct exponential temporal kernels for excitatory and inhibitory neurons given by $\tau_E{=}20$ ms, and $\tau_I{=}5$ ms, respectively. 
    \item Comprehensive analysis and ablation studies to investigate \gls{tide}'s performance against \gls{ctm} in the presence of corrupted or \gls{ood} data. Experimental results demonstrate that \gls{tide} consistently outperforms \gls{ctm} in robustness.
\end{itemize}

The remainder of this paper is structured as follows. Section~\ref{sec:related_works} presents related work. Thereafter, Section~\ref{sec:preliminary} introduces the theoretical background and notation adopted throughout the paper, covering the \gls{ctm} formulation and the Dale-constrained \gls{ei} dynamics. Section~\ref{sec:methodology} presents the proposed architecture, its components, and their overall theoretical properties. Section~\ref{sec:result} introduces the evaluation setup, presents the results, and provides both qualitative and quantitative assessments of the proposed architecture. Section~\ref{sec:conclusions} concludes by discussing results and future work.

\section{Related Work}
\label{sec:related_works}

\begin{figure*}[t!]
  \centering
  \includegraphics[width=1\linewidth]{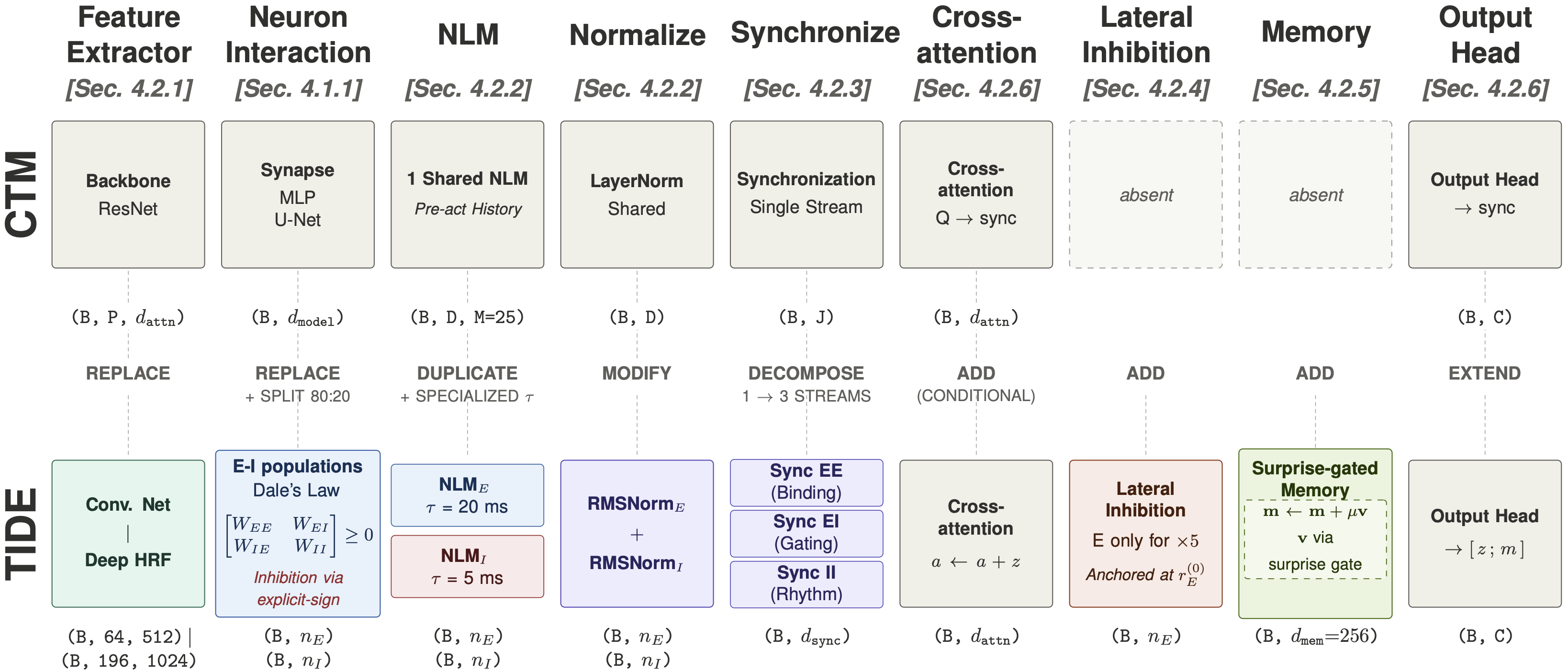}
  \caption{Schematic architectural comparison between CTM and TIDE. TIDE's architectural components are further described in Section~\ref{sec:arch}; Appendix~\ref{app:algo} provides additional technical details.
  }
  \label{fig:tide-vs-ctm}
\end{figure*}

\textbf{Continuous reasoning with recurrency:} In contrast to existing sequence- and iteration-based architectures such as recurrent-depth networks~\cite{schwarzschild2021can}, deep equilibrium models~\cite{bai2019deep}, neural ordinary
differential equations~\cite{chen2018neural}, and adaptive compute~\cite{graves2016adaptive,banino2021pondernet}, \gls{ctm}~\cite{darlow2025continuous} adopts recurrence for internal computation steps and leverages it for reasoning via certainty-based halting. The proposed \gls{tide} retains the internal computation steps and synchronization read-out blocks proposed in \gls{ctm} but replaces the unconstrained recurrent weights with the Wilson-Cowan dynamics under Dale constraints.

\textbf{Wilson-Cowan dynamics \& \gls{ei} balance:} \gls{tide} incorporates a neuro-inspired circuit with separate excitatory and inhibitory neuron populations, each with its own time constants and connectivity, following the Wilson-Cowan model~\cite{wilson1972excitatory,wilson1973mathematical}. This allows operating in the asynchronous-irregular regime, where excitatory and inhibitory activations can cancel while the \gls{ei} network exhibits irregular yet dynamically stable patterns~\cite{vanvreeswijk1996chaos,brunel2000dynamics}. Stability is further promoted through homeostasis via inhibitory plasticity~\cite{vogels2011inhibitory,turrigiano2008homeostatic} and
balanced-network mechanisms~\cite{harris2013cortical}, acting as direct regularizers on network activity.

\textbf{Dale's principle:} Dale's principle~\cite{dale1935pharmacology,eccles1954cholinergic} states that each neuron releases a single neurotransmitter at every given outgoing connection, implying that the recurrent connectivity matrix is sign-indefinite and weights for \gls{ei} network must satisfy $\Wei \ne \Wie^\top$. In contrast to~\cite{cornford2021dales}, the existing models ignore this principle. By parameterizing recurrent weights as $W = \Sigma \odot |M|$, where $\Sigma \in \{+1,-1\}$ is the fixed per-neuron sign vector assigned via Dale's principle, \gls{tide} allows direct usage of the existing optimization algorithms for a neuro-inspired architecture. At each optimization step, $M$ is projected onto the non-negative orthant via the clipping $M \leftarrow \max(M, 0)$ to preserve the assigned sign.

\textbf{Energy-based \& game-theoretic loss:} The use of an energy-based optimization method that leverages Nash equilibria to replace the single scalar loss in homogeneous networks, with generalization to the reinterpretation of dynamics in heterogeneous networks based on \gls{ei} networks, is proposed in~\cite{betteti2025competition}. Therefore, by using the per-neuron gradient and its optimization via the energy function, it is feasible to address the main shortcoming of Hopfield networks~\cite{hopfield1982neural}, thereby providing a principled objective that is consistent, while ensuring the convergence of internal dynamics. Furthermore, this reformulation provides stability under strong convex-concavity of the local energies in the consensus regime during the recurrent updates, replacing the implicit symmetry-based stability argument in~\cite{hopfield1982neural}.

\textbf{Cortex-aligned backbone \& memory model:} It has been demonstrated that an untrained \glspl{cnn} with multi-scale center-surround filters and subsequent spatial compression resembles the primate ventral-stream responses~\cite{kazemian2024cortex}. Based on hierarchical models~\cite{riesenhuber1999hierarchical} and data-driven cortex-alignment studies~\cite{yamins2014performance,schrimpf2020integrative}, \gls{tide} uses \gls{hrf} as its backbone for feature extraction. Moreover, given our adoption of a surprise-gate (Section~\ref{sec:arch:memory}), which differs in its memory formulation from Titans~\cite{behrouz2024titans} and MIRAS~\cite{behrouz2025miras}, and is represented as a state buffer updated iteratively during each internal computation step, without reliance on gradient updates, the timing of memory updates can be coupled to the homeostatic states of the network. \gls{tide}'s inductive bias, achieved via the \gls{ei} network and its temporal mixing operation, is later conditioned by constraints from Dale's principle, thereby removing the need for internal state-space memory and hidden states as used in, e.g., Mamba~\cite{gu2023mamba} and xLSTM~\cite{beck2024xlstm}.

\textbf{Normalization \& optimization:} Since the mean-subtraction of \gls{ln} would inhibit the absolute \gls{ei} balance signaling, which must be preserved, \texttt{RMSNorm}~\cite{zhang2019rmsnorm} per-population is used. Moreover, by using AdamW~\cite{loshchilov2019decoupled} with a linear warm-up followed by cosine decay~\cite{loshchilov2017sgdr}, we stabilize any further issues raised by the mixed-precision training~\cite{micikevicius2018mixed} to improve the computational efficiency of \gls{tide}. 

\section{Preliminaries}
\label{sec:preliminary}
This section presents the notation and background concepts used in the paper. We use superscript~$^{(t)}$ to represent internal computation steps. Populations in \gls{ei} networks are denoted by $E$ and $I$, respectively. The activations are $r_E\!\in\!\RR^{n_E},\,r_I\!\in\!\RR^{n_I}$, with $n_E=\lfloor 0.8 \cdot d_{\mathrm{model}}\rfloor$ and $n_I=d_{\mathrm{model}}-n_E$ while population time constants and their corresponding Euler coefficients are $\tau_E,\,\tau_I$ and $\alpha_E,\,\alpha_I$, respectively.
$\Wee,\Wei,\Wie,\Wii$ are recurrent connectivity matrices whose components are non-negative, in accordance with Dale's principle. 
$P$ denotes the number of spatial backbone tokens, $d_{\mathrm{model}}$ is model dimension while $d_{\mathrm{attn}} {=} n_{\mathrm{heads}}\, d_{\mathrm{head}}$ and $d_{\mathrm{sync}} {=} d_{EE}{+}d_{EI}{+}d_{II}$ are attention and synchronization dimensions, respectively. Remaining notations are introduced throughout the text.

Given a fixed input, \gls{ctm} processes it over $T$ internal computation steps, representing thinking by a recurrence whose state consists of a collection of per-neuron activation histories~\cite{darlow2025continuous}. Two of the main pillars of \gls{ctm}, namely \gls{nlm} and latent representation via synchronization, are adopted in \gls{tide}. Each neuron $i$ has a small MLP $g_{\theta_i}\!:\RR^M\!\to\!\RR$ that maps the last $M$ pre-activations of that neuron as a temporal buffer with a window size $M$, to its next post-activation $r^\tau_i = g_{\theta_i}(a^{(\tau-M+1):(\tau)}_i)$. The synchronization, where the vector of post-activations is fed to the readout, is the second-order statistic of the post-activation history:
    $S^{(t)}_{ij} = \frac{\sum_{\tau \le t} \lambda^{t-\tau}\, r^\tau_i\, r^\tau_j}
                    {\sqrt{\sum_{\tau \le t} \lambda^{t-\tau}\, (r^\tau_i)^2
                    \;\sum_{\tau \le t} \lambda^{t-\tau}\, (r^\tau_j)^2}}$,
where $\lambda \in (0, 1]$ denotes the decay factor that down-weights distant history~\cite{darlow2025continuous}. Note that \gls{ctm} does not follow Dale's principle; therefore, its formulation and weights are fully unconstrained.

\textbf{Wilson-Cowan dynamics:} The two-population continuous-time rate \gls{ei} networks are described by
\begin{equation}
\tau_E\,\dot r_E = -r_E + \varphi\,\!\bigl(\Wee\,r_E - \Wei\,r_I + u_E\bigr); \quad
\tau_I\,\dot r_I = -r_I + \varphi\,\!\bigl(\Wie\,r_E - \Wii\,r_I + u_I\bigr), 
\label{eq:wc-cont-I}
\end{equation}
where $\varphi$, $u_E$, and $u_I$ denote the activation function, and external excitatory and inhibitory inputs, respectively~\cite{wilson1972excitatory}. Excitation decays by $\tau_E \approx 20$ ms while inhibition relaxes with $\tau_I \approx\!5$ ms, thus the receptor kinetics is imposed such that $\tau_I\!<\!\tau_E$. A first-order forward-Euler discretization with time step $\Delta t = 1$ ms yields the recurrence:
\begin{equation}
r_E^{(t)} = (1-\alpha_E)\,r_E^{(t-1)} + \alpha_E\,\varphi(h_E^{(t)}); \quad
r_I^{(t)} = (1-\alpha_I)\,r_I^{(t-1)} + \alpha_I\,\varphi(h_I^{(t)}), 
\label{eq:wc-euler-I}
\end{equation}
with $\alpha_\bullet\!=\!\Delta t/\tau_\bullet$ and $h_\bullet^{(t)}$ representing the pre-activations that are computed as
\begin{equation}
h_E^{(t)} = W_{EE}\,r_E^{(t-1)} - W_{EI}\,r_I^{(t-1)} + W_E^{\mathrm{in}}\,a^{(t)}; \quad
h_I^{(t)} = W_{IE}\,r_E^{(t-1)} - W_{II}\,r_I^{(t-1)} + W_I^{\mathrm{in}}\,a^{(t)}, 
\label{eq:h-I-main}
\end{equation}
where $a^{(t)}\!\in\!\mathbb R^{d_{\mathrm{sync}}}$ denotes the cross-attention among latent vector $z^{(t)}$, see Section~\ref{sec:arch:sync}. 

\textbf{Dale's principle:} Under Dale's principle~\cite{dale1935pharmacology,eccles1954cholinergic}, a neuron's effect on its post-synaptic neighbors has a fixed sign determined by its neurotransmitter. In an artificial network, this is formulated as recurrent connectivity matrices with non-negative components $W_{EE}, W_{EI}, W_{IE}, W_{II}$. The inhibitory population's contribution to downstream activity is given an explicit minus sign as in~\eqref{eq:h-I-main}.

\textbf{Game-theoretic formulation:} Since an \gls{ei} network violates the assumption of Hopfield networks, $W\!=\!W^\top$, per-neuron energies $\{E_i\}_{i=1}^{d_{model}}$ are defined such that a single scalar Lyapunov function can be adopted for asymmetric populations. Following~\cite{betteti2025competition}, we assign a per-neuron local energy, $E_i$, with neurons partitioned into excitatory subset, $n_E$, and inhibitory subset, $n_I$. Thus, the local energy is defined as $E_i(x,u_i)$, where the joint network state, $x \in \RR^{d_{model}}$, is the stacked vector of post-Euler firing rates, $x^{(t)} = [\, r_E^{(t)};\; r_I^{(t)} \,]$. The firing rate of neuron $i$ for $r_E^{(t)}$ and $r^{(t)}_I$ at computation step, $t$, is denoted by $x_i^{(t)} \in \RR$, while $u_i \in \RR$ denotes external input driving neuron~$i$, for a given state of neighboring neurons, and is given by $W_\bullet^{\mathrm{in}} a^{(t)}$; $x^\ast$ is the joint state that denotes the Nash equilibrium in a Zero-sum game where the stationary conditions are satisfied iff: $\nabla_{x_i}E_i(x^\ast,u_i)=0\quad\forall i$.

\section{TIDE Architecture \& Its Foundations}
\label{sec:methodology}

\subsection{Theoretical Foundations}
\label{sec:theory}

Here we provide theoretical foundations of \gls{tide}; complete proofs are presented in Appendices~\ref{app:discrete}--\ref{app:cont}.

\subsubsection{Dale-constrained Gradient Optimization}
\label{sec:theory:dale}

A projection $W\!\leftarrow\!\Proj_{\mathrm{Dale}}(W)\!=\!\max(W,0)$ is used to enforce Dale's principle during optimization updates. It is component-wise, idempotent, and $1$-Lipschitz in the Frobenius norm (Proposition~\ref{prop:proj}), and under \gls{ddp} produces identical weights for gradients across the model.

\begin{theorem}[Dale-compatible gradient, informal]
\label{thm:main-pgd}
Let $\mathcal L$ be $L$-smooth and $\mu$-strongly convex over the
non-negative orthant $\mathcal W_\mathrm{Dale}\!=\!\{W\!\ge\!0\}$, while considering the iteration $W^{(k+1)}\!=\!\Proj_\mathrm{Dale}(W^{(k)}-\eta\,\nabla\mathcal L(W^{(k)}))$ with $\eta\!=\!1/L$. Then $W^{(k)}\!\in\!\mathcal W_\mathrm{Dale}$ for every $k\!\ge\!0$, and $\|W^{(k)}-W^\ast\|_F^2\!\le\!(1-\mu/L)^k\|W^{(0)}-W^\ast\|_F^2$ for the every unique is projected minimizer $W^\ast$ over $\mathcal W_\mathrm{Dale}$.
\end{theorem}

The training objective defined in Section~\ref{sec:arch:loss} is non-convex. Thus, the theorem applies only at a strict local minimum at which the standard condition, i.e., Lipschitz-continuous gradient, is met. Moreover, \gls{bptt} through the constraint is addressed in Remark~\ref{rem:bptt}.

\subsubsection{Game-theoretic Formulation \& Lyapunov Diagonal Stability}
\label{sec:theory:game}

Let $x = [\, r_E;\; r_I \,] \in \RR^{d_{model}}$, and $W_{\mathrm{eff}} \in \RR^{d_{model} \times d_{model}}$ be the signed effective recurrent matrix obtained by composing the four Dale-constrained matrices as defined in \eqref{eq:Weff-repeat} while defining the per-neuron energies following \cite{betteti2025competition} as
\begin{equation}
E_i(x,u_i) = -x_i\!\sum_{j=1}^{d_{model}}\!\bigl(1-\tfrac12 \delta_{ij}\bigr)\,W_{\mathrm{eff},ij}\,x_j \,-\, x_i u_i \,+\, \int_0^{x_i}\!\!\varphi^{-1}(s)\,ds,
\label{eq:energy-i-inline}
\end{equation}
where $\delta_{ij}$, and $\varphi^{-1}: \mathrm{range}(\varphi) \to \RR$ denote the Kronecker delta and the inverse of the neural activation function $\varphi$ introduced in \eqref{eq:wc-cont-I}, respectively. Moreover, $u_i$ is the per-neuron component of the cross-attention $a^{(t)}$, and the factor $(1 - \tfrac{1}{2}\delta_{ij})$ halves the diagonal contribution $W_{\mathrm{eff},ii}\,x_i^2$ so that $\partial_{x_i}$ recovers the full self-coupling term. Differentiating \eqref{eq:energy-i-inline} with respect to $x_i$ and using $(d/dx_i)\!\int_{0}^{x_i}\!\varphi^{-1}(s)\,ds = \varphi^{-1}(x_i)$
yields $\varphi^{-1}(x_i) - \partial_{x_i} E_i = (W_{\mathrm{eff}}\, x)_i + u_i$, which is the right-hand side of the Wilson-Cowan dynamics in \eqref{eq:wc-cont-I}. This provides the following benefits: 
\textit{i}) A principled soft penalty on the gradient as residual $\sum_i \|\partial_{x_i} E_i\|^2$, which is instantiated as the game loss in Section~\ref{sec:arch:loss}; and 
\textit{ii}) A saturation-based path to guaranteed existence of a unique equilibrium under mild contractivity conditions on $W_{\mathrm{eff}}$ and $\varphi$, formalized in Theorem~\ref{thm:fp-exist}.

\begin{theorem}[\gls{lds} implies global convergence]
\label{thm:lds-main}
If $W_\mathrm{eff}-I$ is \gls{lds} with $D\!\succ\!0$ such that $D(W_\mathrm{eff}{-}I){+}(W_\mathrm{eff}{-}I)\T D \!\prec\!0$, the linearized dynamics $\dot x{=}W_\mathrm{eff}x-x+b$ admits a unique Nash equilibrium $x^\ast{=}(I{-}W_\mathrm{eff})^{-1}b$, to which every trajectory converges \cite{betteti2025competition}.
\end{theorem}

Therefore, \gls{lds} is monitored at runtime, but does not backpropagate through $\lambda_{\max}(W_\mathrm{eff}{+}W_\mathrm{eff}\T)/2$ given the eigen-decomposition is non-differentiable at crossings. Instead, gradient-based stability is achieved using the differentiable Perron surrogate provided below.

\subsubsection{Spectral Stability of Wilson-Cowan Dynamics}
\label{sec:theory:spec}

The Perron-based spectral regularizer is defined below, while its functionality and limitations are discussed in detail. The analysis is carried out under the assumption of fully isolated E and I populations. Furthermore, the coupled \gls{ei} stability is provided under the \gls{lds} Theorem~\ref{thm:lds-main}.

\textbf{Isolated-E bound:} Linearizing \eqref{eq:wc-euler-I} around
$r_E^\ast{=}0$ without inhibitory signals yields $r_E^{(t+1)}\!=\!M_E\,r_E^{(t)}$ with $M_E\!=\!(1-\alpha_E)I+\alpha_E\Wee$. Given that the eigenvalues of $M_E$ are $\mu_i\!=\!(1-\alpha_E)+\alpha_E\lambda_i(\Wee)$, for the real non-negative Perron eigenvalue $\perr(\Wee)\!\ge\!0$, Schur stability $|\mu_i|\!<\!1$ for every $i$ is reduced to:
$\;\perr(\Wee)\,<\,1\; \quad\text{(isolated-E Schur bound).}$
That is, the \emph{E-only} sub-dynamics is contracting iff $\Wee$'s Perron eigenvalue is strictly below $1$, for every Euler step $\alpha_E \in (0,1]$.

\textbf{Isolated-I bound:} With the explicit minus sign, the linearization on the inhibitory population results in $\mu_j\!=\!(1-\alpha_I)-\alpha_I\lambda_j(\Wii)$. Thus, Schur stability $|\mu_j|\!<\!1$ on the Perron eigenvalue reduces to $\perr(\Wii)\,<\,2/\alpha_I-1\,=\,9$ at $\alpha_I\!=\!0.20$.
The upper threshold is binding in the I population, unlike the E population, where $(1-\alpha_E)$ leakage guarantees the lower $|\mu|\!>\!-1$ bound trivially as the minus sign flips $\mu_j$ past $-1$ for large $\lambda_j$.

\textbf{The coupled system \& limiters:}
In practice, with an activated inhibitory pathway, \gls{tide} operates above the isolated-E bound, e.g., on ImageNet-1K~\cite{russakovsky2015imagenet} $\persurr(\Wee)\!\approx\!14.7$ at convergence. The \gls{ei} network is stable as the effective-matrix linearization, $M_\mathrm{eff}\!=\!I-\mathrm{diag}(\alpha)+\mathrm{diag}(\alpha)W_\mathrm{eff}$
with sign-indefinite $W_\mathrm{eff}$ (\eqref{eq:W-eff} in the Appendix), has inhibitory signals that suppress the unstable excitatory modes. The Perron values of the isolated populations $\Wee,\Wii$ are, therefore, \emph{not} Schur conditions for the coupled \gls{ei} network; they act as a limiting factor for each non-negative population. Empirically, we find that $\tau_{EE}\!=\!15,$ and $\tau_{II}\!=\!7$ keep gradient norms bounded and prevent instability during training on the ImageNet-1K dataset.

\textbf{Differentiable Perron surrogate:} Because $W_{\cdot\cdot}\!\ge\!0$, the Perron-Frobenius theorem guarantees a real non-negative dominant eigenvalue. Therefore, it can be estimated by the sum-ratio power iteration
\begin{equation}
v_0\!=\!\one/n,\qquad v_{k+1}\!=\!W_{\cdot\cdot}\,v_k/\|W_{\cdot\cdot}\,v_k\|_2,\qquad
\persurr(W_{\cdot\cdot})\!=\!\one\T W_{\cdot\cdot}\,v_K/\one\T v_K,\quad K\!=\!10,
\label{eq:perron-sumratio-main}
\end{equation}
which is differentiable in $W_{\cdot\cdot}$ (Proposition~\ref{prop:sumratio}). On trained $\Wee$, we observe $\|\Wee\|_F/\persurr(\Wee)\!\approx\!2.5$; thus the sum-ratio penalty is tighter than an operator-norm regularizer, which would collapse $\Wee$ to near-diagonal matrices, removing recurrent interactions. For the off-diagonal blocks $\Wei$ and $ \Wie$, we monitor the maximum singular value as a detached diagnostic, without backpropagation through it, since sustained growth above its rolling median indicates impending population collapse.

\subsection{TIDE Architecture}
\label{sec:arch}

In this section, we present \gls{tide} architecture's components, including its feature extractor, \gls{ei} population-based \gls{nlm}, population-based synchronization mechanism, lateral inhibition, cross-attention, and surprise-gated associative memory, see Figure~\ref{fig:tide-vs-ctm}. Each of these components is described in detail below, along with its inputs, outputs, and guiding equations. Figure~\ref{fig:tide-vs-ctm} not only illustrates each of these components and their role but also shows the divergence from the \gls{ctm} framework.  Moreover, the complete training process of \gls{tide} is summarized in Algorithm~\ref{alg:tide} in Appendix~\ref{app:obj:alg}.

\subsubsection{Hierarchical Receptive Field}
\label{sec:arch:hrf}

Given the varying complexity of datasets used in this paper, where different backbones are required to extract and encode features from the input data, we designed two backbone variants for \gls{tide}. A simple shallow \gls{hrf} backbone that maps $32{\times}32$ or $28{\times}28$ input images into an $8{\times}8$ grid of $d_\mathrm{attn}$-token positions, $P{=}64$, using multi-stage filters, while the second neuro-inspired deep \gls{hrf} uses a ResNet-style stem, followed by four hierarchical residual stages for extracting features from ImageNet-1K \cite{russakovsky2015imagenet}. The initial stage zero is a bank of learnable center-surround filters defined as:
\newline
\centerline{$
\phi^{(s)}(x) \;=\; \mathrm{ReLU}\,\!\Bigl(\mathrm{BN}\,\!\bigl(w_c^{(s)}\,C^{(s)}(x) - w_s^{(s)}\,S^{(s)}(x)\bigr)\Bigr),\quad s\!\in\!\{1,2,4,8\},
$}
with independent learnable center- and surround-convolutions $C^{(s)},S^{(s)}$, and scalar gains $(w_c^{(s)},w_s^{(s)})$ initialized at $(1,\,0.5)$, respectively, and $\mathrm{BN}$ denotes BatchNorm. The above expression subsumes the fixed-form \gls{dog} operator of~\cite{marr1980vision} as the special case where $C^{(s)}\!=\!G_s,\,S^{(s)}\!=\!G_{\kappa s}$ are frozen isotropic Gaussian distributions; see Remark~\ref{rem:dog-special} in Appendix~\ref{app:hrf}. Furthermore, stages one to four apply standard residual blocks with $2\times$ spatial down-sampling and channel expansion, while the final stage produces a $14\times 14$ grid of $1024$-dimensional tokens, $P{=}196$.

\textbf{Relation to CTM's backbone:} In contrast to \gls{ctm}, which uses the pre-trained ResNet-152 as its backbone, \gls{hrf} variants use the identity map in their stages one to four as a parameter-space point (cf. Appendix~\ref{app:algo}). This results in differences in stage widths, stem, and block shape. Therefore, we treat them as distinct architectures rather than members of a single parameterized family, and we mark the head-to-head comparisons in Section~\ref{sec:result} accordingly.

\subsubsection{Population-specific NLM}
\label{sec:arch:nlm}

Each neuron has a two-layer gated \gls{mlp} with hidden dimension $H{=}4$ for small images or $H{=}32$ for ImageNet-1K, implemented as a batched \texttt{SuperLinear} operator that maintains an independent $M\to\!2H$ linear map per neuron. The temporal weighting in the first layer uses an exponential kernel, $w_m=\mathrm{softmax}_m\,(-(M - m)/\tau)$ with $\tau{=}\tau_E{=}20$\,ms and $\tau{=}\tau_I{=}5$\,ms for excitatory and inhibitory NLM, matching AMPA and GABA$_A$ kinetics, respectively. Since using pre-activations leads to a positive feedback loop that destabilizes the dynamics, post-activations are stored in the FIFO buffer.

\subsubsection{Synchronization as Latent Representation}
\label{sec:arch:sync}

For each pair type $XY\!\in\!\{EE,EI,II\}$, a deterministic number of neuron pairs denoted by $p_{XY}$, is sampled from index tensors $\mathbf i \mathbf j$, to maintain recurrent sums of pair-wise product $\pi_k^{(t)}\!=\!r^{(t)}_{X,i_k}\,r^{(t)}_{Y,j_k}$:
\newline
\centerline{$
\nu_{XY}^{(t)} = e^{-\delta_{XY}}\,\nu_{XY}^{(t-1)} + \pi^{(t)},\qquad
\xi_{XY}^{(t)} = e^{-\delta_{XY}}\,\xi_{XY}^{(t-1)} + 1,
$}
with learnable decays $\delta_{XY} \in \RR^{p_{XY}}$ clipped to $[0,15]$. The step-$t$ synchronization vector is expressed as $z_{XY}^{(t)}\!=\!\texttt{Proj}_{XY}\!\bigl(\nu_{XY}^{(t)}/\sqrt{\xi_{XY}^{(t)}\!+\!\varepsilon}\bigr)$, where $\texttt{Proj}_{XY}$ is a linear projection from $\RR^{p_{XY}}$ to $\RR^{d_{XY}}$, which represents per-stream latent space. The full latent space is defined as their concatenation, $z^{(t)}\!=\!\texttt{LN}\,\!\bigl[z_{EE}^{(t)}; \,z_{EI}^{(t)}; \,z_{II}^{(t)}\bigr] \in \RR^{d_{\mathrm{sync}}}$. The individual streams are interpreted as \emph{binding} for E-E phase-lock, \emph{gating} for E-I coherence, and \emph{rhythm} for I-I coherence.

\subsubsection{Lateral Inhibition}
\label{sec:arch:wta}

Lateral inhibition is based on \gls{wta} and leverages a nested E-I network within the excitatory population prior to the sync readout (Algorithm~\ref{alg:tide}). A pool of auxiliary inhibitory neurons reads the current excitatory activity and sends back an inhibitory signal that is subtracted from the excitatory neurons as follows:
\newline
\centerline{$
r_I^{(k)} = \mathrm{ReLU}(W_{EI}^\mathrm{lat}\,r_E^{(k-1)}),\quad
r_E^{(k)} = \mathrm{ReLU}(r_E^{(0)} - \gamma\,W_{IE}^\mathrm{lat}\,r_I^{(k)}),\quad k\!=\!1,\dots,K_\mathrm{WTA},
$}
where $^{(k)}$ are \gls{wta} iterations, $\gamma{>}0$ is a learnable gain, $W_{\cdot\cdot}^\mathrm{lat}\!\ge\!0$ are lateral weights, and $K_\mathrm{WTA}{=}5$. Following~\cite{betteti2025competition}, the limit of the expression is given by a Nash equilibrium of the Zero-sum game in which each excitatory neuron maximizes its own activity subject to a shared inhibitory budget. This results in a sparse representation in which only the most strongly driven neurons remain active.

\subsubsection{Surprise-gated Associative Memory with an E-I Retention Gate}
\label{sec:arch:memory}

Inspired by the Titans~\cite{behrouz2024titans} and MIRAS~\cite{behrouz2025miras} architectures, \gls{tide} uses an internal-state module as its memory. Unlike \gls{ctm}, which has no memory module (Figure~\ref{fig:tide-vs-ctm}), a combination of a persistent buffer $m\!\in\!\RR^{d_\mathrm{memory}}$, a reconstruction head $f_\mathrm{rec}$, and a projection head $f_\mathrm{proj}$, enables the formation of surprise signal $s^{(t)}\!=\!\|f_\mathrm{rec}(m){-}z^{(t)}\|_2^2$. The \gls{ei} retention gate, $\iota^{(t)}\!=\!\sigma(-\kappa\,|\rho_{EI}^{(t)}{-}\rho_{EI}^\ast|)$ with learnable $\kappa\!>\!0$, is high when the population-normalized \gls{ei} activity ratio, $\rho^{(t)}_{EI}\!=\!(\|r_E^\mathrm{pre}\|_1/n_E)/(\|r_I\|_1/n_I)$, is near \gls{tide}'s target $\rho_{EI}^\ast\!=\!4$ and low otherwise.
The buffer is updated by a momentum-smoothed, surprise-gated rule
$ m^{(t)}\!=\! m^{(t-1)}\,\!+\,\! v^{(t)}$ with $ v^{(t)}\!=\!\mu v^{(t-1)}\,\!+\,\!g^{(t)}f_\mathrm{proj}(z^{(t)})$ and $g^{(t)}\!=\!\mathbbm 1[s^{(t)}{>}\theta_s](1-\iota^{(t)})$, for $\theta_s{=}0.5,\,\mu{=}0.9$. The memory state does not receive the task loss
gradient, while the heads and $\kappa$ are optimized via backpropagation. The \gls{ei} retention gate is the only component of the architecture that \emph{couples} the memory update with the network's homeostatic state: when the excitation dominates and the E-I ratio is unbalanced, i.e., the network is within an unreliable regime, the memory does not absorb new information.

\begin{figure}[t!]
  \centering
  \includegraphics[clip, width=1.0\textwidth]{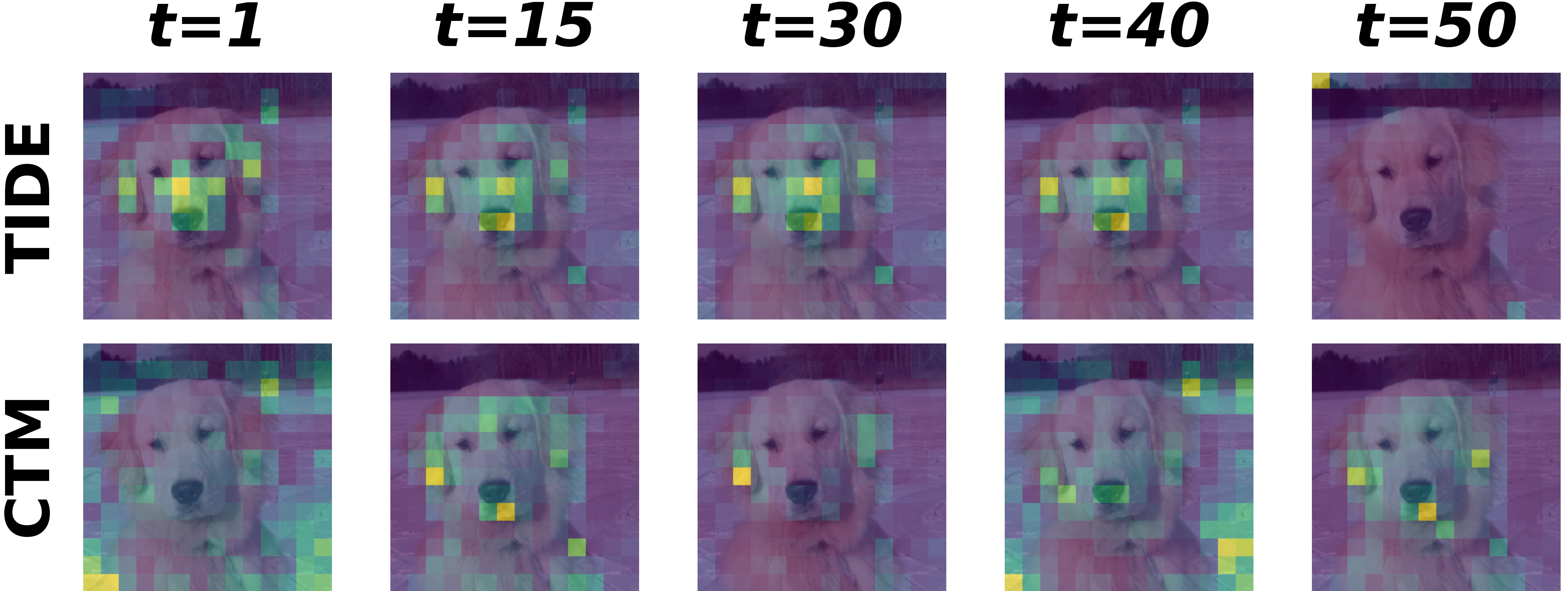}
  \caption{Temporal evolution of mean attention as saliency per computation step for \gls{tide} and CTM.}
  \label{fig:attention_compare}
\end{figure}

\subsubsection{Cross-attention \& Output Head}
\label{sec:arch:attn}

A multi-headed attention with query $\texttt{Proj}_Q(z^{(t)})$ and keys/values supplied by the backbone is used to compute $a^{(t)}$ that is fed back into \eqref{eq:h-I-main}. For ImageNet-1K~\cite{russakovsky2015imagenet}, we use an additional residual stream $a^{(t)}{\leftarrow}a_{\mathrm{attn}}^{(t)}{+}z^{(t)}$ to enable adaptive attention via joint attention and latent synchronization to improve the attention mechanism, as exemplified in Figure~\ref{fig:attention_compare}, where \gls{ctm}'s attention is comparatively less stable. The logits are computed as $o^{(t)}\!=\!W_\mathrm{out}\,\texttt{LN}\,\bigl[\texttt{GLU}(W_\mathrm{hidden} [z^{(t)};m^{(t)}])\bigr]$ with the final prediction denoted by $o^{(T)}$, and the intermediate $\{o^{(t)}\}_{t<T}$, which drive the task loss and the certainty curve used for adaptive compute at inference. To preserve the sign of the gradient, $\texttt{GLU}$ is used as its smooth gate keeps the gradient non-zero while operating on a normalized latent space that is sign-indefinite.

\subsubsection{Training Objective}
\label{sec:arch:loss}

The total loss is a weighted sum of five terms, each capturing a distinct design goal:
\begin{equation}
\mathcal L(\theta) \;=\; \mathcal L_\mathrm{task} \;+\; w(\texttt{step})\!\cdot\!\bigl(\lambda_\mathrm{EI}\mathcal L_\mathrm{EI} \!+\! \lambda_\mathrm{game}\mathcal L_\mathrm{game} \!+\! \lambda_\mathrm{sync}\mathcal L_\mathrm{sync}\bigr) \;+\; \lambda_\mathrm{spec}\mathcal L_\mathrm{spec},
\label{eq:total-loss-main}
\end{equation}
where $w(\texttt{step})\!\in\![0,1]$ is the curriculum warm-up
coefficient to allow warm-up training steps. Additionally, the weights for each loss are $(\lambda_{\mathrm{EI}},\lambda_{\mathrm{game}},\lambda_{\mathrm{sync}},\lambda_{\mathrm{spec}})\!=\!(10^{-2},10^{-3},10^{-4},10^{-1})$.

\textbf{Task loss:} \gls{tide} directly adopts \gls{ctm}'s task loss that is expressed as $\mathcal L_\mathrm{task} = \tfrac12\,\texttt{CE}\,(o^{(t_\mathrm{min})},y) + \tfrac12\,\texttt{CE}\,(o^{(t_\mathrm{cert})},y),$
where $t_\mathrm{min}\!=\!\arg\min_t\texttt{CE}\,(o^{(t)},y)$ and
$t_\mathrm{cert}\!=\!\arg\max_t\,c^{(t)}$ for the entropy-based certainty $c^{(t)}\!=\!1-H(p^{(t)})/\log C$. This loss encourages the model not only to provide a correct prediction at a given internal computation step but also to increase its confidence within the recurrent computation, thereby inducing adaptive behavior during internal computation steps.

\textbf{E-I balance:} To enforce \gls{ei} population balance, we penalize deviations of the activity ratio $\rho_{EI}$ from its designated target $\rho_{EI}^\ast = 4$ via $\mathcal L_\mathrm{EI} = \texttt{clip}~(\rho_{EI}-\rho_{EI}^\ast,-50,50)^2,$ which acts as a soft homeostatic regularizer on the population-normalized \gls{ei} activity ratio and is set at an $80{:}20$ population split. The target $\rho_{EI}^\ast$ is a design choice that does not directly reflect the biological constant.

\textbf{Game loss:}
We use a scalar, clipped, size-normalized surrogate of the per-neuron
gradient residual, $\sum_i\|\partial_{x_i}E_i\|^2$, under mean-field approximation and \eqref{eq:energy-i-inline}: 
$\mathcal L_\mathrm{game} = \tfrac{1}{d_\mathrm{model}}\,\min(\left[n^{-1}_E\!\sum_{i=1}^{n_E}\! \mathcal{E}_{E}^{(i)}\;+\; n^{-1}_I\!\sum_{j=1}^{n_I}\! \mathcal{E}_{I}^{(j)}\right],\,100)$, 
where $\mathcal{E}_E =\frac{[\,(\bar w_{EE} - d_E)\, r^{(t)}_E \;-\; \bar w_{EI}\, \bar r_I +\; u_E\,]^{2}} {2\,(d_E - \bar w_{EE})}$, $\mathcal{E}_I =\frac{[\,\bar w_{IE}\, \bar r_E \;-\; (\bar w_{II} + d_I)\, r^{(t)}_I \;+\; u_I\,]^{2}} {2\,(d_I + \bar w_{II})}$, $\bar r_{\bullet}=n_{\bullet}^{-1} \sum_i (r_{\bullet})_i$, $d_E > 0$, and $d_I > 0$ denote population dissipation constants, and $\bar w_{\bullet} = \texttt{mean}(\mathrm{diag}(W_{\bullet}))$ as population-effective scalar weights.
Dividing by $d_\mathrm{model}$ makes the term scale-invariant, while clipping stabilizes transient spikes early in training.

\textbf{Synchronization \& spectral regularizers:}
$\mathcal L_\mathrm{sync}\!=\!\|z^{(T)}\|_2^2/d_\mathrm{sync}$ is used to bound the accumulators and latent space. The spectral regularizer is formulated as $
\mathcal L_\mathrm{spec} = \mathrm{ReLU}\,(\persurr(\Wee)-\tau_{EE})^2 + \mathrm{ReLU}\,(\persurr(\Wii)-\tau_{II})^2,$
where $(\tau_{EE},\,\tau_{II})\!=\!(15,7)$, while using the sum-ratio Perron estimator of \eqref{eq:perron-sumratio-main}. As discussed in Section~\ref{sec:theory:spec}, we use these values to achieve stability during training. 

\begin{table}[!b]
\centering
\caption{Comparison between \gls{tide} and \gls{ctm} on five image classification datasets at equal training budget. The \emph{Best Seed} column reports the highest-performing run with its best/final \texttt{top-1} accuracy, $\mathrm{Acc@1}$ (\%), and \texttt{mean\,$\pm$\,std} across the multi seeded training for both architectures. \gls{ctm} is retrained based on~\cite{darlow2025continuous} for all the ablations with identical seeds. Note that ${}^\dagger$ denotes the mean \texttt{top-1} accuracy drop of more than $1$\%, and ${}^\ddagger$ indicates an effective batch size of 64. Moreover, the highest \texttt{top-1} accuracy for each dataset is in bold. Additional details on the experiments are provided in Appendix~\ref{app:add}.}
\label{tab:main-bench2}
\scriptsize
\setlength{\tabcolsep}{2pt}
\renewcommand{\arraystretch}{1.05}
\newcommand{\sdfont}{\fontsize{6}{7}\selectfont}
\newcommand{\ms}[2]{#1\textsubscript{\sdfont\textpm#2}}
\newcommand{\nodag}{\hphantom{${}^{\dagger}$}}
\begin{tabular*}{\textwidth}{@{\extracolsep{\fill}}l@{\hskip 5pt}l@{\hskip 5pt}l@{\hskip 14pt}c@{\hskip 6.5pt}c@{\hskip 11pt}c@{\hskip 6.5pt}c@{\hskip 14pt}c@{\hskip 6.5pt}c@{\hskip 11pt}c@{\hskip 6.5pt}c@{}}
\toprule
 & & & \multicolumn{4}{c@{\hskip 14pt}}{\textbf{TIDE}} & \multicolumn{4}{c@{}}{\textbf{CTM}} \\
\cmidrule(lr){4-7}\cmidrule(l){8-11}
 & & & \multicolumn{2}{c@{\hskip 11pt}}{\textbf{\texttt{mean\,$\pm$\,std}}} & \multicolumn{2}{c@{\hskip 14pt}}{\textbf{Best Seed}} & \multicolumn{2}{c@{\hskip 11pt}}{\textbf{\texttt{mean\,$\pm$\,std}}} & \multicolumn{2}{c@{}}{\textbf{Best Seed}} \\
\cmidrule(lr){4-5}\cmidrule(lr){6-7}\cmidrule(lr){8-9}\cmidrule(l){10-11}
\textbf{Task} & \textbf{Steps} & $d_{\mathrm{model}}$ & best & final & best & final & best & final & best & final \\
\midrule
MNIST & 50K & 256 & \textbf{\ms{99.64}{0.03}} & \textbf{\ms{99.61}{0.04}}\nodag & \textbf{99.67} & \textbf{99.63}\nodag & {\ms{99.54}{0.03}} & {\ms{99.52}{0.02}} & {99.56} & {99.54\nodag} \\
& 200K & 256 & \textbf{\ms{99.60}{0.05}} & \textbf{\ms{99.53}{0.07}}\nodag & \textbf{99.64} & \textbf{99.60}\nodag & \ms{99.59}{0.06} & \ms{99.49}{0.08} & \textbf{99.64} & 99.57\nodag \\
\midrule
F-MNIST & 50K & 256 & \textbf{\ms{94.21}{0.21}} & \textbf{\ms{94.15}{0.25}}\nodag & \textbf{94.42} & \textbf{94.42}\nodag & {\ms{92.51}{0.11}} & {\ms{92.45}{0.19}} & {92.63} & {92.63\nodag} \\
& 200K & 256 & \textbf{\ms{94.12}{0.19}} & \textbf{\ms{93.85}{0.10}}\nodag & \textbf{94.23} & \textbf{93.96}\nodag & \ms{92.53}{0.24} & \ms{92.43}{0.27} & 92.80 & 92.73\nodag \\
\midrule
CIFAR-10 & 600K & 512 & \textbf{\ms{90.62}{0.06}} & \textbf{\ms{90.54}{0.07}}\nodag & \textbf{90.68} & \textbf{90.62}\nodag & \ms{86.23}{0.34} & \ms{85.98}{0.35} & 86.59 & 86.31\nodag \\
\midrule
CIFAR-100 & 600K & 718 & \ms{62.23}{0.16} & \ms{61.39}{0.82}\nodag & 62.38 & 61.97\nodag & \multirow{2}{*}{\ms{65.19}{0.48}} & \multirow{2}{*}{\textbf{\ms{65.01}{0.39}}} & \multirow{2}{*}{65.70} & \multirow{2}{*}{\textbf{65.44}\nodag} \\
 & 600K${}^{\ddagger}$ & 718 & \textbf{\ms{65.68}{0.18}} & \ms{64.50}{0.60}${}^{\dagger}$ & \textbf{65.79} & 65.08\nodag &  &  &  &  \\
\midrule
ImageNet-1K & 100K & 4096 & \textbf{\ms{67.22}{1.52}} & \textbf{\ms{67.01}{1.58}}\nodag & \textbf{68.74} & \textbf{68.59}\nodag & \ms{52.67}{1.79} & \ms{51.47}{1.62} & 54.46 & 53.09\nodag \\
& 200K & 4096 & \textbf{\ms{68.40}{1.91}} & \textbf{\ms{68.18}{2.25}}\nodag & \textbf{69.55} & \textbf{69.50}\nodag & \ms{62.78}{1.31} & \ms{61.81}{1.50} & 64.06 & 62.88${}^{\dagger}$ \\
 & 500K & 4096 & --- & --- & 69.90 & 69.87\nodag & --- & --- & \textbf{71.78} & \textbf{69.88}${}^{\dagger}$ \\
\bottomrule
\end{tabular*}
\end{table}

\section{Experimental Evaluation}
\label{sec:result}

\begin{table}[!t]
\centering
\caption{Multi-seeded hyperparameter ablation study performed on MNIST and Fashion-MNIST (F-MNIST) datasets. The table provides the mean best and final \texttt{top-1} accuracy, $\mathrm{Acc@1}$ (\%), as the \texttt{mean\,$\pm$\,std} across the seeds for evaluation during training. Default hyperparameters for the reported experiments are in bold, while the headers indicate which hyperparameter is being studied. We denote training instability, where the mean \texttt{top-1} accuracy drops by more than $1$\%, with ${}^\dagger$. Appendix~\ref{app:add} provides additional details on the ablation studies.
} 
\label{tab:ablation-compact}
\fontsize{6.2}{6.2}\selectfont
\setlength{\tabcolsep}{2pt}
\renewcommand{\arraystretch}{1.05}
\newcommand{\sdfont}{\fontsize{4.5}{5}\selectfont}
\newcommand{\ms}[2]{#1\textsubscript{\sdfont \textpm#2}}
\newcommand{\pz}{\hphantom{0}}
\newcommand{\ps}{\hphantom{\sdfont 0}}
\newcommand{\nodag}{\hphantom{${}^{\dagger}$}}
\begin{tabular}{@{}l@{\hskip 3pt}c@{\hskip 3pt}c@{\hskip 5.5pt}l@{\hskip 3pt}c@{\hskip 3pt}c@{\hskip 5.5pt}l@{\hskip 3pt}c@{\hskip 3pt}c@{}}
\toprule
& \textbf{MNIST} & \textbf{F-MNIST}
& & \textbf{MNIST} & \textbf{F-MNIST}
& & \textbf{MNIST} & \textbf{F-MNIST} \\
\rlap{\textbf{Setting}}
& \multicolumn{2}{c@{\hskip 5.5pt}}{$\mathrm{Acc@1}$\,(\%) best / final}
& \rlap{\textbf{Setting}}
& \multicolumn{2}{c@{\hskip 5.5pt}}{$\mathrm{Acc@1}$\,(\%) best / final}
& \rlap{\textbf{Setting}}
& \multicolumn{2}{c}{$\mathrm{Acc@1}$\,(\%) best / final} \\
\midrule
\multicolumn{3}{@{}l@{\hskip 5.5pt}}{\textbf{\textit{(i) Excitatory fraction $n_E/d_{\mathrm{model}}$}}}
& \multicolumn{3}{l@{\hskip 5.5pt}}{\textbf{\textit{(ii) Game-loss weight $\lambda_{\mathrm{game}}$}}}
& \multicolumn{3}{l}{\textbf{\textit{(iii) $\tau_I$ (ms), $\tau_E{=}20$ fixed}}} \\[5pt]
0.6 & \ms{99.57}{0.07}/\ms{99.39}{0.30}\nodag & \ms{94.01}{0.36}/\ms{92.91}{2.17}${}^{\dagger}$
& $0$ & \ms{99.61}{0.04}/\ms{99.53}{0.09}\nodag\ps & \ms{92.62}{2.56}/\ms{92.19}{2.96}\nodag\ps
& \pz3 & \ms{99.61}{0.03}/\ms{99.58}{0.02}\nodag & \ms{93.62}{0.69}/\ms{92.01}{3.09}${}^{\dagger}$ \\
0.7 & \ms{99.53}{0.08}/\ms{99.41}{0.19}\nodag & \ms{93.83}{0.16}/\ms{92.92}{1.75}\nodag
& $\mathbf{10^{-3}}$ & \textbf{\ms{99.61}{0.04}/\ms{99.55}{0.05}}\nodag\ps & \textbf{\ms{93.80}{0.58}/\ms{93.70}{0.60}}\nodag\ps
& \pz\textbf{5} & \textbf{\ms{99.61}{0.05}/\ms{99.55}{0.05}}\nodag & \textbf{\ms{93.69}{0.67}/\ms{93.48}{0.84}}\nodag \\
\textbf{0.8} & \textbf{\ms{99.60}{0.06}/\ms{99.55}{0.05}}\nodag & \textbf{\ms{93.78}{0.59}/\ms{91.88}{3.84}}${}^{\dagger}$
& $10^{-2}$ & \ms{99.59}{0.09}/\ms{99.57}{0.09}\nodag\ps & \ms{93.90}{0.33}/\ms{93.75}{0.38}\nodag\ps
& \pz7 & \ms{99.56}{0.09}/\ms{98.54}{2.02}${}^{\dagger}$ & \ms{93.84}{0.68}/\ms{93.77}{0.68}\nodag \\
0.9 & \ms{99.59}{0.05}/\ms{99.51}{0.16}\nodag & \ms{94.01}{0.35}/\ms{93.92}{0.39}\nodag
& $10^{-1}$ & \ms{99.34}{0.36}/\ms{94.17}{6.23}${}^{\dagger}$\ps & \ms{93.39}{1.05}/\ms{92.95}{1.61}\nodag\ps
& 10 & \ms{99.59}{0.04}/\ms{99.57}{0.04}\nodag & \ms{93.95}{0.31}/\ms{93.91}{0.30}\nodag \\
\midrule
\multicolumn{3}{@{}l@{\hskip 5.5pt}}{\textbf{\textit{(iv) $\tau_E$ (ms), $\tau_I{=}5$ fixed}}}
& \multicolumn{3}{l@{\hskip 5.5pt}}{\textbf{\textit{(v) Internal computation depth $T$}}}
& \multicolumn{3}{l}{\textbf{\textit{(vi) Lateral inhibition iterations $K_{\mathrm{WTA}}$}}} \\[5pt]
10 & \ms{99.38}{0.46}/\ms{98.66}{1.79}\nodag & \ms{94.15}{0.20}/\ms{92.95}{2.34}${}^{\dagger}$
& \pz10 & \ms{99.63}{0.04}/\ms{99.61}{0.05}\nodag\ps & \ms{94.12}{0.18}/\ms{94.07}{0.18}\nodag\ps
& \pz1 & \ms{99.59}{0.08}/\ms{99.55}{0.10}\nodag & \ms{94.09}{0.06}/\ms{93.99}{0.09}\nodag \\
15 & \ms{99.61}{0.02}/\ms{99.56}{0.07}\nodag & \ms{94.13}{0.23}/\ms{94.04}{0.29}\nodag
& \pz25 & \ms{99.64}{0.04}/\ms{99.59}{0.06}\nodag\ps & \ms{93.73}{0.62}/\ms{93.34}{1.29}\nodag\ps
& \pz3 & \ms{99.62}{0.04}/\ms{99.54}{0.04}\nodag & \ms{94.02}{0.29}/\ms{94.00}{0.29}\nodag \\
\textbf{20} & \textbf{\ms{99.62}{0.04}/\ms{99.57}{0.05}}\nodag & \textbf{\ms{93.96}{0.59}/\ms{93.77}{0.58}}\nodag
& \pz\textbf{50} & \textbf{\ms{99.62}{0.04}/\ms{99.56}{0.04}}\nodag\ps & \textbf{\ms{93.86}{0.56}/\ms{93.78}{0.58}}\nodag\ps
& \pz\textbf{5} & \textbf{\ms{99.57}{0.11}/\ms{99.42}{0.29}}\nodag & \textbf{\ms{93.66}{0.69}/\ms{93.47}{0.85}}\nodag \\
25 & \ms{99.63}{0.03}/\ms{99.59}{0.06}\nodag & \ms{93.92}{0.66}/\ms{93.21}{1.80}\nodag
& \pz75 & \ms{99.54}{0.12}/\ms{84.20}{29.80}${}^{\dagger}$ & \ms{93.90}{0.15}/\ms{87.94}{11.68}${}^{\dagger}$
& 10 & \ms{99.61}{0.08}/\ms{99.58}{0.08}\nodag & \ms{94.16}{0.27}/\ms{94.07}{0.33}\nodag \\
30 & \ms{99.49}{0.17}/\ms{96.91}{5.21}${}^{\dagger}$ & \ms{93.89}{0.64}/\ms{93.79}{0.64}\nodag
& 100 & \ms{99.31}{0.36}/\ms{96.03}{5.49}${}^{\dagger}$\ps & \ms{93.33}{0.97}/\ms{75.67}{34.19}${}^{\dagger}$
& & & \\
\bottomrule
\end{tabular}
\end{table}

\begin{table}[!b]
\centering
\caption{Multi-seeded leave-one-out ablation performed on four image classification datasets to analyze TIDE and its components. The table provides the mean best and final \texttt{top-1} accuracy, $\mathrm{Acc@1}$ (\%), as the \texttt{mean\,$\pm$\,std} across the seeds for evaluation during training. Moreover, the highest \texttt{top-1} accuracy for each dataset is in bold, and the mean top-1 accuracy drop by more than $1$\% in the final stages of training is marked with ${}^\dagger$.}
\label{tab:leave}
\scriptsize
\setlength{\tabcolsep}{2pt}
\renewcommand{\arraystretch}{1.05}
\newcommand{\sdfont}{\fontsize{6}{6}\selectfont}
\newcommand{\ms}[2]{#1\textsubscript{\sdfont\textpm#2}}
\newcommand{\ps}{\hphantom{\sdfont 0}}
\newcommand{\nodag}{\hphantom{${}^{\dagger}$}}
\begin{tabular}{@{}l@{\hskip 6pt}c@{\hskip 4pt}c@{\hskip 6pt}c@{\hskip 4pt}c@{\hskip 6pt}c@{\hskip 4pt}c@{\hskip 6pt}c@{\hskip 4pt}c@{}}
\toprule
 & \multicolumn{2}{c@{\hskip 6pt}}{\textbf{MNIST}} & \multicolumn{2}{c@{\hskip 6pt}}{\textbf{F-MNIST}} & \multicolumn{2}{c@{\hskip 6pt}}{\textbf{CIFAR-10}} & \multicolumn{2}{c@{}}{\textbf{CIFAR-100}} \\
\cmidrule(lr){2-3}\cmidrule(lr){4-5}\cmidrule(lr){6-7}\cmidrule(lr){8-9}
\textbf{Configuration} & best & final & best & final & best & final & best & final \\
\midrule
\multicolumn{9}{@{}l@{}}{\textbf{\textit{Reference}}} \\[2pt]
CTM & \ms{99.54}{0.03} & \ms{99.52}{0.02} & \ms{92.51}{0.11} & \ms{92.45}{0.19} & \ms{86.23}{0.34} & \ms{85.98}{0.35}\nodag\ps & \ms{65.19}{0.48} & \textbf{\ms{65.01}{0.39}}\nodag \\
\textbf{TIDE (full, $\lambda_{\mathrm{spec}}\!=\!0.0$)} & \textbf{\ms{99.64}{0.02}} & \ms{99.57}{0.02} & \ms{94.21}{0.21} & \ms{94.15}{0.25} & \textbf{\ms{90.62}{0.06}} & \textbf{\ms{90.54}{0.07}}\nodag\ps & \textbf{\ms{65.68}{0.18}} & \ms{64.50}{0.60}${}^{\dagger}$ \\
\midrule
\multicolumn{9}{@{}l@{}}{\textbf{\textit{Component removed}}} \\[2pt]
$-$ Wilson-Cowan update & \ms{99.61}{0.02} & \ms{99.56}{0.01} & \ms{93.70}{0.12} & \ms{93.65}{0.15} & \ms{90.18}{0.26} & \ms{90.14}{0.27}\nodag\ps & \ms{64.52}{0.67} & \ms{63.65}{0.40}\nodag \\
$-$ Dale's law constraint & \ms{99.59}{0.03} & \ms{99.57}{0.04} & \ms{93.55}{0.18} & \ms{93.51}{0.16} & \ms{89.01}{1.74} & \ms{85.46}{6.75}${}^{\dagger}$\ps & \ms{65.39}{0.39} & \ms{64.53}{0.91}\nodag \\
$-$ Lateral inhibition (WTA) & \ms{99.61}{0.02} & \ms{99.58}{0.04} & \ms{93.80}{0.17} & \ms{93.77}{0.19} & \ms{89.13}{1.20} & \ms{82.90}{8.69}${}^{\dagger}$\ps & \ms{64.94}{0.30} & \ms{63.65}{0.24}${}^{\dagger}$ \\
$-$ Surprise-gated memory & \ms{99.59}{0.01} & \ms{99.51}{0.01} & \ms{92.95}{0.16} & \ms{92.76}{0.16} & \ms{88.97}{1.41} & \ms{85.68}{4.58}${}^{\dagger}$\ps & \ms{64.32}{0.30} & \ms{62.03}{0.97}${}^{\dagger}$ \\
$-$ ResNet-18 inst. of HRF backbone & \ms{99.60}{0.03} & \ms{99.51}{0.05} & \ms{93.03}{0.12} & \ms{92.83}{0.20} & \ms{83.02}{2.02} & \ms{48.82}{30.94}${}^{\dagger}$ & \ms{54.49}{0.76} & \ms{53.25}{1.02}${}^{\dagger}$ \\
\midrule
\multicolumn{9}{@{}l@{}}{\textbf{\textit{Component added}}} \\[2pt]
$+$ Spectral regularizer ($\lambda_{\mathrm{spec}}\!=\!0.1$) & \ms{99.64}{0.03} & \textbf{\ms{99.61}{0.05}} & \textbf{\ms{94.31}{0.06}} & \textbf{\ms{94.27}{0.08}} & \ms{89.02}{2.93} & \ms{88.14}{4.94} & \ms{65.38}{0.35} & \ms{64.60}{0.32}\nodag \\
\bottomrule
\end{tabular}
\end{table}

\textbf{Setup:}
We evaluate TIDE on five image classification datasets: MNIST~\cite{lecun1998gradient}, Fashion-MNIST~\cite{xiao2017fashion}, CIFAR-10 and CIFAR-100~\cite{krizhevsky2009cifar}, and ImageNet-1K~\cite{russakovsky2015imagenet}; Detailed hyperparameter configurations are specified in Appendix~\ref{app:exp:hparams}.

\textbf{Main benchmark:}
Table~\ref{tab:main-bench2} provides a detailed comparison between \gls{tide} and \gls{ctm} under multi-seeded training and evaluation under identical training budgets. Note that we retrain \gls{ctm} following~\cite{darlow2025continuous} using the same seeds as the \gls{tide} model to ensure reproducible results. In general, \gls{tide} is more sample efficient and requires fewer steps to reach higher \texttt{top-1} accuracy. As shown in Table~\ref{tab:main-bench2}, \gls{tide} without reliance on deep feature extractors as its backbone can outperform \gls{ctm} in the majority of image classification tasks, while requiring a smaller training budget.

\textbf{Hyperparameter ablation study:}
We use MNIST and Fashion-MNIST to analyze how various hyperparameter choices affect \texttt{top-1} accuracy, learning outcomes, and the stability of \gls{tide}. All models are trained for $50$\,K steps, with an identical simple \gls{hrf} backbone with $d_{model}=256$ across all ablation studies. Table~\ref{tab:ablation-compact} reports multi-seeded experiments on \gls{tide}, while Appendix~\ref{app:add} provides full breakdowns for each experiment. It has been observed that internal computation steps $T\!\in\![10,50]$ provides stable learning, while $\tau_E\!\in\![15,25]$ and $\tau_I\!\in\![5,7]$ provide the highest accuracy. Moreover, the choice of lateral inhibition iterations, $K_{\mathrm{WTA}}$, depends on both the internal computation steps and the dataset. As illustrated in Table~\ref{tab:ablation-compact}, $K_{\mathrm{WTA}}=\{1, 3\}$ outperforms the selected default value and results in higher mean \texttt{top-1} accuracy for both MNIST and Fashion-MNIST.


\begin{figure}[!t]
  \centering
  \includegraphics[clip, trim=2.9cm 1.0cm 0.75cm 0.74cm, width=\textwidth]{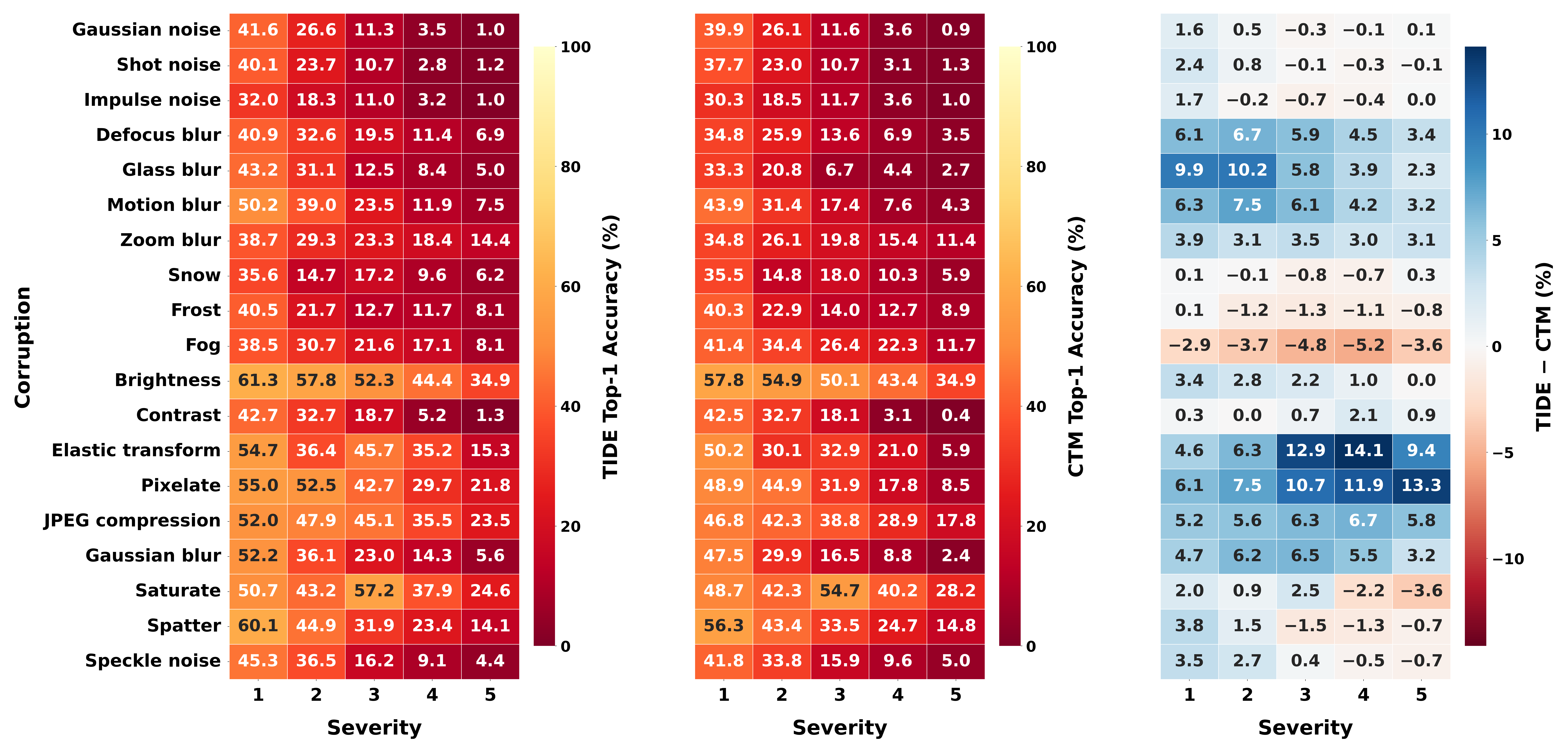}
  \caption{Robustness analysis using ImageNet-C~\cite{hendrycks2019robustness} for corrupted images with perturbations. \textit{Left} panel presents results for \gls{tide}, \textit{center} panel corresponds to \gls{ctm}, and \textit{right} panel reports differences.
  }
  \label{fig:robustness_c}
\end{figure}

\textbf{Architectural component study:}
To further analyze each component of the proposed architecture, we perform a leave-one-out ablation study across the four image classification tasks (cf. Table~\ref{tab:leave}) to compare \gls{tide} and its components against \gls{ctm} under a fixed training budget. Both \gls{tide} and \gls{ctm} are trained for $50$\,K steps across MNIST and Fashion-MNIST datasets with identical backbone and seeds similar to Table~\ref{tab:main-bench2}. Additionally, the training budgets for CIFAR-10 and CIFAR-100 follow a similar pattern and are set to $600$\,K steps. Table~\ref{tab:leave} demonstrates that \gls{tide} with all components achieves the highest performance across all the datasets compared to \gls{ctm}. In contrast, removing any individual component of \gls{tide} results in losing \texttt{top-1} accuracy. The largest drop in best \texttt{top-1} accuracy occurs in MNIST and Fashion-MNIST when removing the memory component, while the choice of backbone has a more detrimental effect on CIFAR-10 and CIFAR-100. Additional ablation studies on the backbone variants are provided in Appendix~\ref{app:add}.

\textbf{Robustness study:}
Tiny-ImageNet~\cite{le2015tiny}, ImageNet-C~\cite{hendrycks2019robustness}, and ImageNet-R~\cite{hendrycks2021many} were used to assess the \gls{ood} robustness of \gls{tide} and \gls{ctm} that are solely trained on the original ImageNet-1K (cf. Table~\ref{tab:main-bench2}). Figure~\ref{fig:robustness_c} summarizes the performance degradation of \gls{ctm} and \gls{tide} on ImageNet-C that are trained for $500$\,K steps. \gls{tide} outperforms \gls{ctm} in the presence of most perturbations. Note that under aerosol-particle perturbations such as fog, frost, and snow, \gls{ctm} is more robust. Nevertheless, under all considered perturbations, \gls{tide} improves \gls{ctm}'s \texttt{top-1} accuracy by an average of $+2.68$\% as shown in Figure~\ref{fig:robustness_c}. We also provide additional robustness evaluations for various severities, datasets, and training budgets in Appendix~\ref{app:add:robust}.

\section{Discussion \& Conclusions}
\label{sec:conclusions}

In this study, we presented \gls{tide}, a new architecture that extends \gls{ctm} with Wilson-Cowan dynamics and Dale's principle, enforcing non-negative weights. This design enables integration of an asymmetric neural circuit to stabilize temporal dynamics while running a two-population recurrence for internal computation steps, as in \gls{ctm}. Moreover, we used homeostatic \gls{ei} ratio activity to introduce a novel retention gate that serves as a persistent memory update.

\gls{tide} can achieve high \texttt{top-1} accuracy, even with the population collapse, where the full neuron population converts to excitatory, indicating that \gls{ctm} can be considered as a special case of \gls{tide}, without inhibitory neurons, $n_I=0$ (cf. Table~\ref{tab:ablation-compact}). Furthermore, Appendix~\ref{app:add:robust} reports that 
\gls{tide} trained with a smaller training budget between $100$\,K to $200$\,K steps can outperform \gls{ctm} under varied severity of perturbations. This supports the hypothesis about \gls{tide}'s stability and sample efficiency while justifying the design with neuro-inspired primitives, which require fewer computational resources to achieve generalization and better \gls{ood} performance.

Further performance gains can be achieved by replacing the scalar-weighted game loss with an augmented-Lagrangian formulation that couples the per-neuron Nash residual to $\rho_{EI} \approx \rho_{EI}^\star$, enabling manifold-based optimization instead of a task-specific gradient approach. Moreover, imposing stricter biophysical constraints and moving toward more bio-plausible modeling may improve stability under extensive inhibitory feedback, helping address the training drift and computational time instability observed in Table \ref{tab:ablation-compact}. Additional coupling between neural circuits and the input via anchoring~\cite{bai2019deep} can further stabilize the architecture and its temporal dynamics beyond the trained internal computation with recurrence (cf. Appendix~\ref{app:add}). Finally, further investigation and evaluation are needed to scale \gls{tide} to sequence-based modeling, enabling spike-based implementations for greater energy efficiency.

\newpage

\begin{ack}
We thank three anonymous reviewers for their comments that helped to improve the presentation style and experimental evaluation reported in the study. 
The authors thank RISE ICE Data Center for providing access to the computing infrastructure. 
DK acknowledges funding from the Swedish Strategic Research Foundation (Grant No. FFL24-0111) and the Swedish Research Council (Grant No. 2025-05421).
We thank the Kempe-foundation for providing us with the Lotty Bruzilius cluster used for compute.
The authors declare no competing interests.
\end{ack}

\small

\bibliography{main}

@article{lecun1998gradient,
  author  = {LeCun, Yann and Bottou, L\'{e}on and Bengio, Yoshua and Haffner, Patrick},
  title   = {Gradient-Based Learning Applied to Document Recognition},
  journal = {Proceedings of the IEEE},
  year    = {1998},
  volume  = {86},
  number  = {11},
  pages   = {2278--2324},
  doi     = {10.1109/5.726791},
  url     = {https://doi.org/10.1109/5.726791}
}

@article{krizhevsky2012imagenet,
    author = {Krizhevsky, Alex and Sutskever, Ilya and Hinton, Geoffrey E.},
    title = {{ImageNet} classification with deep convolutional neural networks},
    year = {2017},
    issue_date = {June 2017},
    address = {New York, NY, USA},
    volume = {60},
    number = {6},
    issn = {0001-0782},
    url = {https://doi.org/10.1145/3065386},
    doi = {10.1145/3065386},
    pages = {84–90},
    numpages = {7},
    journal = {Communications of the ACM},
}

@inproceedings{he2016deep,
  title={Deep residual learning for image recognition},
  author={He, Kaiming and Zhang, Xiangyu and Ren, Shaoqing and Sun, Jian},
  booktitle={IEEE Conference on Computer Vision and Pattern Recognition (CVPR)},
  pages={770--778},
  year={2016},
  doi={10.1109/CVPR.2016.90}
}

@inproceedings{vaswani2017attention,
  author    = {Vaswani, Ashish and Shazeer, Noam and Parmar, Niki and Uszkoreit, Jakob and Jones, Llion and Gomez, Aidan N. and Kaiser, {\L}ukasz and Polosukhin, Illia},
  title     = {Attention Is All You Need},
  booktitle = {Advances in Neural Information Processing Systems (NeurIPS)},
  volume    = {30},
  pages     = {},
  year      = {2017},
  url       = {https://papers.nips.cc/paper/7181-attention-is-all-you-need}
}

@inproceedings{dosovitskiy2021image,
  title     = {An Image is Worth 16x16 Words: Transformers for Image Recognition at Scale},
  author    = {Dosovitskiy, Alexey and Beyer, Lucas and Kolesnikov, Alexander and Weissenborn, Dirk and Zhai, Xiaohua and Unterthiner, Thomas and Dehghani, Mostafa and Minderer, Matthias and Heigold, Georg and Gelly, Sylvain and Uszkoreit, Jakob and Houlsby, Neil},
  booktitle = {International Conference on Learning Representations (ICLR)},
  year      = {2021},
  url       = {https://openreview.net/forum?id=YicbFdNTTy}
}

@article{isaacson2011inhibition,
  author  = {Isaacson, Jeffry S. and Scanziani, Massimo},
  title   = {How Inhibition Shapes Cortical Activity},
  journal = {Neuron},
  year    = {2011},
  volume  = {72},
  number  = {2},
  pages   = {231--243},
  doi     = {10.1016/j.neuron.2011.09.027},
  url     = {https://doi.org/10.1016/j.neuron.2011.09.027}
}

@article{zador2019critique,
  title={A critique of pure learning and what artificial neural networks can learn from animal brains},
  author={Zador, Anthony M.},
  journal={Nature Communications},
  volume={10},
  number={1},
  pages={3770},
  year={2019},
  doi={10.1038/s41467-019-11786-6}
}

@article{lillicrap2016random,
  author  = {Lillicrap, Timothy P. and Cownden, Daniel and Tweed, Douglas B. and Akerman, Colin J.},
  title   = {Random Synaptic Feedback Weights Support Error Backpropagation for Deep Learning},
  journal = {Nature Communications},
  year    = {2016},
  volume  = {7},
  pages   = {13276},
  doi     = {10.1038/ncomms13276},
  url     = {https://doi.org/10.1038/ncomms13276}
}

@article{scellier2017equilibrium,
  author  = {Scellier, Benjamin and Bengio, Yoshua},
  title   = {Equilibrium Propagation: Bridging the Gap Between Energy-Based Models and Backpropagation},
  journal = {Frontiers in Computational Neuroscience},
  year    = {2017},
  volume  = {11},
  pages   = {24},
  doi     = {10.3389/fncom.2017.00024},
  url     = {https://doi.org/10.3389/fncom.2017.00024}
}

@article{payeur2021burst,
  author  = {Payeur, Alexandre and Guerguiev, Jordan and Zenke, Friedemann and Richards, Blake A. and Naud, Richard},
  title   = {Burst-Dependent Synaptic Plasticity Can Coordinate Learning in Hierarchical Circuits},
  journal = {Nature Neuroscience},
  year    = {2021},
  volume  = {24},
  pages   = {1010--1019},
  doi     = {10.1038/s41593-021-00857-x},
  url     = {https://doi.org/10.1038/s41593-021-00857-x}
}

@inproceedings{haber2022daleian,
  title={The computational and learning benefits of {D}aleian neural networks},
  author={Haber, Adam and Schneidman, Elad},
  booktitle = {Advances in Neural Information Processing Systems (NeurIPS)},
  url = {https://proceedings.neurips.cc/paper_files/paper/2022/file/21cb5931c39d7bd21b34b3b8f14a125c-Paper-Conference.pdf},
  volume={35},
  year={2022}
}

@article{haider2006neocortical,
  title={Neocortical network activity in vivo is generated through a dynamic balance of excitation and inhibition},
  author={Haider, Bilal and Duque, Alvaro and Hasenstaub, Andrea R. and McCormick, David A.},
  journal={The Journal of Neuroscience},
  volume={26},
  number={17},
  pages={4535--4545},
  year={2006},
  doi={10.1523/JNEUROSCI.5297-05.2006}
}

@article{betteti2025competition,
  author        = {Betteti, Simone and Retnaraj, William and Davydov, Alexander and Cort\'{e}s, Jorge and Bullo, Francesco},
  title         = {Competition, Stability, and Functionality in Excitatory-Inhibitory Neural Circuits},
  journal       = {arXiv:2512.05252},
  year          = {2025},
  archivePrefix = {arXiv},
  eprint        = {2512.05252},
  doi           = {10.48550/arXiv.2512.05252},
  url           = {https://arxiv.org/abs/2512.05252}
}

@inproceedings{darlow2025continuous,
  author    = {Darlow, Luke and Regan, Ciaran and Risi, Sebastian and Seely, Jeffrey and Jones, Llion},
  title     = {Continuous Thought Machines},
  booktitle = {Advances in Neural Information Processing Systems (NeurIPS)},
  year      = {2025},
  url       = {https://openreview.net/forum?id=y0wDflmpLk}
}

@article{dale1935pharmacology,
  author  = {Dale, Henry},
  title   = {Pharmacology and Nerve-endings},
  journal = {Proceedings of the Royal Society of Medicine},
  year    = {1935},
  volume  = {28},
  number  = {3},
  pages   = {319--332},
  doi     = {10.1177/003591573502800330},
  url     = {https://doi.org/10.1177/003591573502800330}
}

@article{eccles1954cholinergic,
  author  = {Eccles, John C. and Fatt, Paul and Koketsu, Kyozo},
  title   = {Cholinergic and Inhibitory Synapses in a Pathway from Motor-Axon Collaterals to Motoneurones},
  journal = {The Journal of Physiology},
  year    = {1954},
  volume  = {126},
  number  = {3},
  pages   = {524--562},
  doi     = {10.1113/jphysiol.1954.sp005226},
  url     = {https://doi.org/10.1113/jphysiol.1954.sp005226}
}

@article{ChiYan2022CVNN,
  author  = {Lee, ChiYan and Hasegawa, Hideyuki and Gao, Shangce},
  journal = {IEEE/CAA Journal of Automatica Sinica}, 
  title   = {Complex-Valued Neural Networks: A Comprehensive Survey}, 
  year    = {2022},
  volume  = {9},
  number  = {8},
  pages   = {1406-1426},
  doi     = {10.1109/JAS.2022.105743}
}

@inproceedings{schwarzschild2021can,
  author        = {Schwarzschild, Avi and Borgnia, Eitan and Gupta, Arjun and Huang, Furong and Vishkin, Uzi and Goldblum, Micah and Goldstein, Tom},
  title         = {Can You Learn an Algorithm? {G}eneralizing from Easy to Hard Problems with Recurrent Networks},
  booktitle     = {Advances in Neural Information Processing Systems (NeurIPS)},
  pages         = {6695--6706},
  volume        = {34},
  year          = {2021},
  eprint        = {2106.04537},
  archivePrefix = {arXiv},
  url = {https://proceedings.neurips.cc/paper_files/paper/2021/file/3501672ebc68a5524629080e3ef60aef-Paper.pdf}
}

@inproceedings{bai2019deep,
  author        = {Bai, Shaojie and Kolter, J. Zico and Koltun, Vladlen},
  title         = {Deep Equilibrium Models},
  booktitle     = {Advances in Neural Information Processing Systems (NeurIPS)},
  pages         = {},
  volume        = {32},
  year          = {2019},
  eprint        = {1909.01377},
  archivePrefix = {arXiv},
  url           = {https://papers.nips.cc/paper/8358-deep-equilibrium-models}
}

@inproceedings{chen2018neural,
  author        = {Chen, Ricky T. Q. and Rubanova, Yulia and Bettencourt, Jesse and Duvenaud, David},
  title         = {Neural Ordinary Differential Equations},
  booktitle     = {Advances in Neural Information Processing Systems (NeurIPS)},
  year          = {2018},
  eprint        = {1806.07366},
  archivePrefix = {arXiv},
  url           = {https://papers.nips.cc/paper/7892-neural-ordinary-differential-equations}
}

@article{graves2016adaptive,
  author        = {Graves, Alex},
  title         = {Adaptive Computation Time for Recurrent Neural Networks},
  journal       = {arXiv preprint arXiv:1603.08983},
  year          = {2016},
  archivePrefix = {arXiv},
  eprint        = {1603.08983},
  url           = {https://arxiv.org/abs/1603.08983}
}

@inproceedings{banino2021pondernet,
  author        = {Banino, Andrea and Balaguer, Jan and Blundell, Charles},
  title         = {{PonderNet}: Learning to Ponder},
  booktitle       = {ICML Workshop on Automated Machine Learning},
  year          = {2021},
  archivePrefix = {arXiv},
  eprint        = {2107.05407},
  url           = {https://arxiv.org/abs/2107.05407}
}

@article{wilson1972excitatory,
  author  = {Wilson, Hugh R. and Cowan, Jack D.},
  title   = {Excitatory and Inhibitory Interactions in Localized Populations of Model Neurons},
  journal = {Biophysical Journal},
  year    = {1972},
  volume  = {12},
  number  = {1},
  pages   = {1--24},
  doi     = {10.1016/S0006-3495(72)86068-5},
  url     = {https://doi.org/10.1016/S0006-3495(72)86068-5}
}

@article{wilson1973mathematical,
  author  = {Wilson, Hugh R. and Cowan, Jack D.},
  title   = {A Mathematical Theory of the Functional Dynamics of Cortical and Thalamic Nervous Tissue},
  journal = {Kybernetik},
  year    = {1973},
  volume  = {13},
  number  = {2},
  pages   = {55--80},
  doi     = {10.1007/BF00288786},
  url     = {https://doi.org/10.1007/BF00288786}
}

@article{vanvreeswijk1996chaos,
  author  = {van Vreeswijk, Carl and Sompolinsky, Haim},
  title   = {Chaos in Neuronal Networks with Balanced Excitatory and Inhibitory Activity},
  journal = {Science},
  year    = {1996},
  volume  = {274},
  number  = {5293},
  pages   = {1724--1726},
  doi     = {10.1126/science.274.5293.1724},
  url     = {https://doi.org/10.1126/science.274.5293.1724}
}

@article{brunel2000dynamics,
  author  = {Brunel, Nicolas},
  title   = {Dynamics of Sparsely Connected Networks of Excitatory and Inhibitory Spiking Neurons},
  journal = {Journal of Computational Neuroscience},
  year    = {2000},
  volume  = {8},
  number  = {3},
  pages   = {183--208},
  doi     = {10.1023/A:1008925309027},
  url     = {https://doi.org/10.1023/A:1008925309027}
}

@article{vogels2011inhibitory,
  author  = {Vogels, Tim P. and Sprekeler, Henning and Zenke, Friedemann and Clopath, Claudia and Gerstner, Wulfram},
  title   = {Inhibitory Plasticity Balances Excitation and Inhibition in Sensory Pathways and Memory Networks},
  journal = {Science},
  year    = {2011},
  volume  = {334},
  number  = {6062},
  pages   = {1569--1573},
  doi     = {10.1126/science.1211095},
  url     = {https://doi.org/10.1126/science.1211095}
}

@article{turrigiano2008homeostatic,
  author  = {Turrigiano, Gina G.},
  title   = {The Self-Tuning Neuron: Synaptic Scaling of Excitatory Synapses},
  journal = {Cell},
  year    = {2008},
  volume  = {135},
  number  = {3},
  pages   = {422--435},
  doi     = {10.1016/j.cell.2008.10.008},
  url     = {https://doi.org/10.1016/j.cell.2008.10.008}
}

@article{harris2013cortical,
  author  = {Harris, Kenneth D. and Mrsic-Flogel, Thomas D.},
  title   = {Cortical Connectivity and Sensory Coding},
  journal = {Nature},
  year    = {2013},
  volume  = {503},
  number  = {7474},
  pages   = {51--58},
  doi     = {10.1038/nature12654},
  url     = {https://doi.org/10.1038/nature12654}
}

@article{hopfield1982neural,
  author  = {Hopfield, J. J.},
  title   = {Neural Networks and Physical Systems with Emergent Collective Computational Abilities},
  journal = {Proceedings of the National Academy of Sciences},
  year    = {1982},
  volume  = {79},
  number  = {8},
  pages   = {2554--2558},
  doi     = {10.1073/pnas.79.8.2554},
  url     = {https://doi.org/10.1073/pnas.79.8.2554}
}

@article{kazemian2024cortex,
  author  = {Kazemian, Atlas and Elmoznino, Eric and Bonner, Michael F.},
  title   = {Convolutional Architectures are Cortex-Aligned de Novo},
  journal = {Nature Machine Intelligence},
  year    = {2025},
  volume  = {7},
  pages   = {1834--1844},
  doi     = {10.1038/s42256-025-01142-3},
  url     = {https://doi.org/10.1038/s42256-025-01142-3}
}

@article{riesenhuber1999hierarchical,
  author  = {Riesenhuber, Maximilian and Poggio, Tomaso},
  title   = {Hierarchical Models of Object Recognition in Cortex},
  journal = {Nature Neuroscience},
  year    = {1999},
  volume  = {2},
  number  = {11},
  pages   = {1019--1025},
  doi     = {10.1038/14819},
  url     = {https://doi.org/10.1038/14819}
}

@article{yamins2014performance,
  author  = {Yamins, Daniel L. K. and Hong, Ha and Cadieu, Charles F. and Solomon, Ethan A. and Seibert, Darren and DiCarlo, James J.},
  title   = {Performance-Optimized Hierarchical Models Predict Neural Responses in Higher Visual Cortex},
  journal = {Proceedings of the National Academy of Sciences},
  year    = {2014},
  volume  = {111},
  number  = {23},
  pages   = {8619--8624},
  doi     = {10.1073/pnas.1403112111},
  url     = {https://doi.org/10.1073/pnas.1403112111}
}

@article{schrimpf2020integrative,
  author  = {Schrimpf, Martin and Kubilius, Jonas and Lee, Michael J. and Ratan Murty, N. Apurva and Ajemian, Robert and DiCarlo, James J.},
  title   = {Integrative Benchmarking to Advance Neurally Mechanistic Models of Human Intelligence},
  journal = {Neuron},
  year    = {2020},
  volume  = {108},
  number  = {3},
  pages   = {413--423},
  doi     = {10.1016/j.neuron.2020.07.040},
  url     = {https://doi.org/10.1016/j.neuron.2020.07.040}
}

@article{behrouz2025miras,
  author        = {Behrouz, Ali and Razaviyayn, Meisam and Zhong, Peilin and Mirrokni, Vahab},
  title         = {It's All Connected: A Journey Through Test-Time Memorization, Attentional Bias, Retention, and Online Optimization},
  journal       = {arXiv:2504.13173},
  year          = {2025},
  archivePrefix = {arXiv},
  eprint        = {2504.13173},
  url           = {https://arxiv.org/abs/2504.13173}
}

@article{gu2023mamba,
  author        = {Gu, Albert and Dao, Tri},
  title         = {{Mamba}: Linear-Time Sequence Modeling with Selective State Spaces},
  journal       = {arXiv:2312.00752},
  year          = {2023},
  archivePrefix = {arXiv},
  eprint        = {2312.00752},
  url           = {https://arxiv.org/abs/2312.00752}
}

@inproceedings{beck2024xlstm,
  author        = {Beck, Maximilian and P{\"o}ppel, Korbinian and Spanring, Markus and Auer, Andreas and Prudnikova, Oleksandra and Kopp, Michael and Klambauer, G{\"u}nter and Brandstetter, Johannes and Hochreiter, Sepp},
  title         = {{xLSTM}: Extended Long Short-Term Memory},
  booktitle     = {Advances in Neural Information Processing Systems (NeurIPS)},
  year          = {2024},
  eprint        = {2405.04517},
  archivePrefix = {arXiv},
  url           = {https://proceedings.neurips.cc/paper_files/paper/2024/hash/c2ce2f2701c10a2b2f2ea0bfa43cfaa3-Abstract-Conference.html}
}

@inproceedings{behrouz2024titans,
 author = {Behrouz, Ali and Zhong, Peilin and Mirrokni, Vahab},
 booktitle = {Advances in Neural Information Processing Systems},
 doi = {10.52202/085713-3786},
 editor = {D. Belgrave and C. Zhang and H. Lin and R. Pascanu and P. Koniusz and M. Ghassemi and N. Chen},
 pages = {113506--113543},
 publisher = {Curran Associates, Inc.},
 title = {Titans: Learning to Memorize at Test Time},
 url = {https://proceedings.neurips.cc/paper_files/paper/2025/file/a4ca07aa108036f80cbb5b82285fd4b1-Paper-Conference.pdf},
 volume = {38, Main Conference},
 year = {2025}
}

@inproceedings{zhang2019rmsnorm,
  author        = {Zhang, Biao and Sennrich, Rico},
  title         = {Root Mean Square Layer Normalization},
  booktitle     = {Advances in Neural Information Processing Systems (NeurIPS)},
  year          = {2019},
  eprint        = {1910.07467},
  archivePrefix = {arXiv},
  url           = {https://arxiv.org/abs/1910.07467}
}

@inproceedings{loshchilov2019decoupled,
  author        = {Loshchilov, Ilya and Hutter, Frank},
  title         = {Decoupled Weight Decay Regularization},
  booktitle     = {International Conference on Learning Representations (ICLR)},
  year          = {2019},
  eprint        = {1711.05101},
  archivePrefix = {arXiv},
  url           = {https://arxiv.org/abs/1711.05101}
}

@inproceedings{cornford2021dales,
	author = {Cornford, Jonathan and Kalajdzievski, Damjan and Leite, Marco and Lamarquette, Am{\'e}lie and Kullmann, Dimitri M. and Richards, Blake},
	title = {Learning to live with {Dale}{\textquoteright}s principle: {ANN}s with separate excitatory and inhibitory units},
    booktitle     = {International Conference on Learning Representations (ICLR)},
	year = {2021},
    pages   = {},
    url={https://openreview.net/forum?id=eU776ZYxEpz}
}

@inproceedings{loshchilov2017sgdr,
  author        = {Loshchilov, Ilya and Hutter, Frank},
  title         = {{SGDR}: Stochastic Gradient Descent with Warm Restarts},
  booktitle     = {International Conference on Learning Representations (ICLR)},
  year          = {2017},
  eprint        = {1608.03983},
  archivePrefix = {arXiv},
  url           = {https://arxiv.org/abs/1608.03983}
}

@inproceedings{micikevicius2018mixed,
  author        = {Micikevicius, Paulius and Narang, Sharan and Alben, Jonah and Diamos, Gregory F. and Elsen, Erich and Garc\'{i}a, David and Ginsburg, Boris and Houston, Michael and Kuchaiev, Oleksii and Venkatesh, Ganesh and Wu, Hao},
  title         = {Mixed Precision Training},
  booktitle     = {International Conference on Learning Representations (ICLR)},
  year          = {2018},
  eprint        = {1710.03740},
  archivePrefix = {arXiv},
  url           = {https://arxiv.org/abs/1710.03740}
}

@article{marr1980vision,
  author  = {Marr, David and Hildreth, Ellen},
  title   = {Theory of Edge Detection},
  journal = {Proceedings of the Royal Society of London. Series B, Biological Sciences},
  year    = {1980},
  volume  = {207},
  number  = {1167},
  pages   = {187--217},
  doi     = {10.1098/rspb.1980.0020},
  url     = {https://doi.org/10.1098/rspb.1980.0020}
}

@article{xiao2017fashion,
  title={{Fashion-MNIST}: A novel image dataset for benchmarking machine learning algorithms},
  author={Xiao, Han and Rasul, Kashif and Vollgraf, Roland},
  journal={arXiv:1708.07747},
  year={2017}
}

@article{russakovsky2015imagenet,
  author  = {Russakovsky, Olga and Deng, Jia and Su, Hao and Krause, Jonathan and Satheesh, Sanjeev and Ma, Sean and Huang, Zhiheng and Karpathy, Andrej and Khosla, Aditya and Bernstein, Michael and Berg, Alexander C. and Fei-Fei, Li},
  title   = {{ImageNet} Large Scale Visual Recognition Challenge},
  journal = {International Journal of Computer Vision},
  year    = {2015},
  volume  = {115},
  number  = {3},
  pages   = {211--252},
  doi     = {10.1007/s11263-015-0816-0},
  url     = {https://doi.org/10.1007/s11263-015-0816-0}
}

@techreport{krizhevsky2009cifar,
  author      = {Krizhevsky, Alex},
  title       = {Learning Multiple Layers of Features from Tiny Images},
  institution = {University of Toronto},
  year        = {2009},
  url         = {https://www.cs.toronto.edu/~kriz/learning-features-2009-TR.pdf}
}

@techreport{le2015tiny,
  title   = {Tiny {ImageNet} Visual Recognition Challenge},
  author  = {Le, Ya and Yang, Xuan},
  institution = {Stanford University},
  pages   = {1--7},
  year    = {2015},
  url     = {https://cs231n.stanford.edu/reports/2015/pdfs/yle_project.pdf}
}

@inproceedings{hendrycks2019robustness,
  title         = {Benchmarking Neural Network Robustness to Common Corruptions and Perturbations},
  author        = {Hendrycks, Dan and Dietterich, Thomas},
  booktitle     = {International Conference on Learning Representations (ICLR)},
  year          = {2019},
  url           = {https://openreview.net/forum?id=HJz6tiCqYm},
  eprint        = {1903.12261},
  archivePrefix = {arXiv},
  primaryClass  = {cs.LG}
}

@inproceedings{hendrycks2021many,
  title         = {The Many Faces of Robustness: A Critical Analysis of Out-of-     Distribution Generalization},
  author        = {Hendrycks, Dan and Basart, Steven and Mu, Norman and Kadavath, Saurav and Wang, Frank and Dorundo, Evan and Desai, Rahul and Zhu, Tyler and Parajuli, Samyak and Guo, Mike and Song, Dawn and Steinhardt, Jacob and Gilmer, Justin},
  booktitle     = {IEEE/CVF International Conference on Computer Vision (ICCV)},
  pages         = {8340--8349},
  year          = {2021},
  eprint        = {2006.16241},
  archivePrefix = {arXiv},
  primaryClass  = {cs.CV},
  url           = {https://arxiv.org/abs/2006.16241}, 
}

@book{granas2003fixed,
  author    = {Granas, Andrzej and Dugundji, James},
  title     = {Fixed Point Theory},
  publisher = {Springer},
  series    = {Springer Monographs in Mathematics},
  year      = {2003},
  isbn      = {978-0-387-00173-9},
  doi       = {10.1007/978-0-387-21593-8}
}

@article{ahmadian2021dynamical,
  author  = {Ahmadian, Yashar and Miller, Kenneth D.},
  title   = {What is the Dynamical Regime of Cerebral Cortex?},
  journal = {Neuron},
  volume  = {109},
  number  = {21},
  pages   = {3373--3391},
  year    = {2021},
  doi     = {10.1016/j.neuron.2021.07.031}
}

@article{vanvreeswijk1998chaotic,
  author  = {van Vreeswijk, Carl and Sompolinsky, Haim},
  title   = {Chaotic Balanced State in a Model of Cortical Circuits},
  journal = {Neural Computation},
  volume  = {10},
  number  = {6},
  pages   = {1321--1371},
  year    = {1998},
  doi     = {10.1162/089976698300017214}
}

@book{khalil2002nonlinear,
  author    = {Khalil, Hassan K.},
  title     = {Nonlinear Systems},
  publisher = {Prentice Hall},
  address   = {Upper Saddle River, NJ},
  edition   = {3rd},
  year      = {2002},
  isbn      = {9780130673893},
  url       = {https://www.pearson.com/us/higher-education/program/Khalil-Nonlinear-Systems-3rd-Edition/PGM253106.html}
}

@article{brunel_hakim1999,
  author  = {Brunel, Nicolas and Hakim, Vincent},
  title   = {Fast Global Oscillations in Networks of Integrate-and-Fire Neurons with Low Firing Rates},
  journal = {Neural Computation},
  year    = {1999},
  volume  = {11},
  number  = {7},
  pages   = {1621--1671},
  doi     = {10.1162/089976699300016179},
  url     = {https://doi.org/10.1162/089976699300016179}
}

@inproceedings{sahoo2024simple,
  author        = {Sahoo, Subham Sekhar and Arriola, Marianne and Schiff, Yair and Gokaslan, Aaron and Marroquin, Edgar and Chiu, Justin T. and Rush, Alexander and Kuleshov, Volodymyr},
  title   = {Simple and effective masked diffusion language models},
  booktitle     = {Advances in Neural Information Processing Systems (NeurIPS)},
  year          = {2024},
  archivePrefix = {arXiv},
  eprint        = {2406.07524},
  url           = {https://arxiv.org/abs/2406.07524}
}

@book{hornjohnson2013,
  author    = {Horn, Roger A. and Johnson, Charles R.},
  title     = {Matrix Analysis},
  publisher = {Cambridge University Press},
  address   = {New York, NY, USA},
  edition   = {2nd},
  year      = {2013},
  doi       = {10.1017/CBO9781139020411},
  url       = {https://doi.org/10.1017/CBO9781139020411}
}

@article{kuffler1953rgc,
  author  = {Kuffler, Stephen W.},
  title   = {Discharge Patterns and Functional Organization of Mammalian Retina},
  journal = {Journal of Neurophysiology},
  year    = {1953},
  volume  = {16},
  number  = {1},
  pages   = {37--68},
  doi     = {10.1152/jn.1953.16.1.37},
  url     = {https://doi.org/10.1152/jn.1953.16.1.37}
}

@article{dacey2003rgc,
  author  = {Dacey, Dennis M. and Peterson, Beth B. and Robinson, Farrel R. and Gamlin, Paul D.},
  title   = {Fireworks in the Primate Retina: In Vitro Photodynamics Reveals Diverse {LGN}-Projecting Ganglion Cell Types},
  journal = {Neuron},
  year    = {2003},
  volume  = {37},
  number  = {1},
  pages   = {15--27},
  doi     = {10.1016/S0896-6273(02)01143-1},
  url     = {https://doi.org/10.1016/S0896-6273(02)01143-1}
}

@article{hubel1962receptive,
  author  = {Hubel, David H. and Wiesel, Torsten N.},
  title   = {Receptive Fields, Binocular Interaction and Functional Architecture in the Cat's Visual Cortex},
  journal = {The Journal of Physiology},
  year    = {1962},
  volume  = {160},
  number  = {1},
  pages   = {106--154},
  doi     = {10.1113/jphysiol.1962.sp006837},
  url     = {https://doi.org/10.1113/jphysiol.1962.sp006837}
}

@book{lindeberg1994scale,
  author    = {Lindeberg, Tony},
  title     = {Scale-Space Theory in Computer Vision},
  publisher = {Kluwer Academic Publishers},
  address   = {Boston, MA},
  series    = {The Kluwer International Series in Engineering and Computer Science},
  year      = {1994},
  doi       = {10.1007/978-1-4757-6465-9},
  url       = {https://doi.org/10.1007/978-1-4757-6465-9}
}
\bibliographystyle{unsrt}

\newpage
\section*{NeurIPS Paper Checklist}

\begin{enumerate}

\item {\bf Claims}
    \item[] Question: Do the main claims made in the abstract and introduction accurately reflect the paper's contributions and scope?
    \item[] Answer: \answerYes{} 
    \item[] Justification: The main theoretical contribution of the paper is described in Sections~3-4, while its empirical validation is provided in Section~5 with additional details following in Appendix~F.

\item {\bf Limitations}
    \item[] Question: Does the paper discuss the limitations of the work performed by the authors?
    \item[] Answer: \answerYes{} 
    \item[] Justification: In Section~6, we discuss several limitations, including the addition of an augmented-Lagrangian formulation to couple the per-neuron Nash residual to $\rho_{EI} \approx \rho_{EI}^*$, enabling manifold-based optimization instead of a task-specific gradient approach. Moreover, imposing strict biophysical constraints to further improve the current architecture should be considered. Additional limitations and open questions for future work are provided in Appendix~H.

\item {\bf Theory assumptions and proofs}
    \item[] Question: For each theoretical result, does the paper provide the full set of assumptions and a complete (and correct) proof?
    \item[] Answer: \answerYes{}. 
    \item[] Justification: All the theorems, formulas, and proofs in the paper are numbered and cross-referenced. 
    We state the main theoretical results in the main text and provide proofs and additional details in the Appendix, e.g., in Appendices~A-C.

    \item {\bf Experimental result reproducibility}
    \item[] Question: Does the paper fully disclose all the information needed to reproduce the main experimental results of the paper to the extent that it affects the main claims and/or conclusions of the paper (regardless of whether the code and data are provided or not)?
    \item[] Answer: \answerYes{} 
    \item[] Justification: The information needed to reproduce the main experimental results of the paper is provided in the Appendix. Appendix~E provides the pseudocode of the proposed architecture. Appendix~F provides all hyperparameters and experimental setup necessary to reproduce the experiments. To streamline reproducibility, we are attaching the source code implementing the architecture.

\item {\bf Open access to data and code}
    \item[] Question: Does the paper provide open access to the data and code, with sufficient instructions to faithfully reproduce the main experimental results, as described in supplemental material?
    \item[] Answer: \answerYes{} 
    \item[] Justification:  The datasets used for the experiments reported in this study are open-source and are available from Hugging Face. Along with this submission, we provide the source code that implements the architecture. The experimental protocol, preprocessing steps, and hyperparameters used in simulations are provided in Appendix~F. The scripts necessary to rerun the simulations will be released prior to publication of the paper.

\item {\bf Experimental setting/details}
    \item[] Question: Does the paper specify all the training and test details (e.g., data splits, hyperparameters, how they were chosen, type of optimizer) necessary to understand the results?
    \item[] Answer: \answerYes{} 
    \item[] Justification: All the details are provided in Appendix~F.

\item {\bf Experiment statistical significance}
    \item[] Question: Does the paper report error bars suitably and correctly defined or other appropriate information about the statistical significance of the experiments?
    \item[] Answer: \answerYes{} 
    \item[] Justification: Whenever possible, we reported the results of several simulations with different seeds for PRG. The corresponding figures/tables report the mean and standard deviation. However, the overall suite of experiments was computationally demanding (e.g., some runs took up to two weeks, see Appendix~F.3 for more details), and taking into account the constraints of our compute infrastructure, some of the experiments and ablation studies we ran used only one seed. This was the choice we made to increase the breadth of the experiments being conducted.

\item {\bf Experiments compute resources}
    \item[] Question: For each experiment, does the paper provide sufficient information on the computer resources (type of compute workers, memory, time of execution) needed to reproduce the experiments?
    \item[] Answer: \answerYes{} 
    \item[] Justification: This information is provided in Appendix~F.3.
    
\item {\bf Code of ethics}
    \item[] Question: Does the research conducted in the paper conform, in every respect, with the NeurIPS Code of Ethics \url{https://neurips.cc/public/EthicsGuidelines}?
    \item[] Answer: \answerYes{} 
    \item[] Justification: We followed the NeurIPS Code of Ethics while working on this study.

\item {\bf Broader impacts}
    \item[] Question: Does the paper discuss both potential positive societal impacts and negative societal impacts of the work performed?
    \item[] Answer: \answerNA{} 
    \item[] Justification: The presented study is foundational research. We do not see any immediate path to negative applications and, therefore, the paper does not discuss potential societal impact.
    
\item {\bf Safeguards}
    \item[] Question: Does the paper describe safeguards that have been put in place for the responsible release of data or models that have a high risk for misuse (e.g., pre-trained language models, image generators, or scraped datasets)?
    \item[] Answer: \answerNA{} 
    \item[] Justification: We do not foresee that at the current stage of research, the proposed architecture requires any safeguards.

\item {\bf Licenses for existing assets}
    \item[] Question: Are the creators or original owners of assets (e.g., code, data, models), used in the paper, properly credited and are the license and terms of use explicitly mentioned and properly respected?
    \item[] Answer: \answerYes{} 
    \item[] Justification: Everything that was used in this study is open source, and where applicable, the credit was given via academic references.

\item {\bf New assets}
    \item[] Question: Are new assets introduced in the paper well documented and is the documentation provided alongside the assets?
    \item[] Answer: \answerYes{} 
    \item[] Justification: Appendix~E provides the pseudocode of the proposed architecture. We also attach the source code implementing the architecture.

\item {\bf Crowdsourcing and research with human subjects}
    \item[] Question: For crowdsourcing experiments and research with human subjects, does the paper include the full text of instructions given to participants and screenshots, if applicable, as well as details about compensation (if any)? 
    \item[] Answer: \answerNA{} 
    \item[] Justification: Crowdsourcing was not used in this study. 

\item {\bf Institutional review board (IRB) approvals or equivalent for research with human subjects}
    \item[] Question: Does the paper describe potential risks incurred by study participants, whether such risks were disclosed to the subjects, and whether Institutional Review Board (IRB) approvals (or an equivalent approval/review based on the requirements of your country or institution) were obtained?
    \item[] Answer: \answerNA{} 
    \item[] Justification: Human subjects were not involved in this study. 

\item {\bf Declaration of LLM usage}
    \item[] Question: Does the paper describe the usage of LLMs if it is an important, original, or non-standard component of the core methods in this research? Note that if the LLM is used only for writing, editing, or formatting purposes and does \emph{not} impact the core methodology, scientific rigor, or originality of the research, declaration is not required.
    \item[] Answer: \answerNA{} 
    \item[] Justification: LLMs were not used in the ways that might impact the core methodology, scientific rigor, or originality of the research.

\end{enumerate}

\newpage
\appendix
\setcounter{equation}{0}
\setcounter{figure}{0}
\setcounter{table}{0}
\setcounter{algorithm}{0}
\makeatletter
\renewcommand{\theequation}{A.\arabic{equation}}
\renewcommand{\thefigure}{A.\arabic{figure}}
\renewcommand{\thetable}{A.\arabic{table}}
\renewcommand{\thealgorithm}{A.\arabic{algorithm}}
\makeatother

\section*{Appendices}
\label{sec:appendix}
\addcontentsline{toc}{section}{Appendices}
 
\tableofcontents
\vspace{0.5em}
\newpage

\providecommand{\rE}{r_{E}}
\providecommand{\rI}{r_{I}}
\providecommand{\hE}{h_{E}}
\providecommand{\hI}{h_{I}}
\providecommand{\nEpop}{n_{E}}
\providecommand{\nIpop}{n_{I}}
\providecommand{\uE}{\mathbf{u}_{E}}
\providecommand{\uI}{\mathbf{u}_{I}}
\providecommand{\phiE}{\varphi_{E}}
\providecommand{\phiI}{\varphi_{I}}
\providecommand{\tauE}{\tau_{E}}
\providecommand{\tauI}{\tau_{I}}
\providecommand{\alphaE}{\alpha_{E}}
\providecommand{\alphaI}{\alpha_{I}}
\providecommand{\dmodel}{d_{\mathrm{model}}}
\providecommand{\dattn}{d_{\mathrm{attn}}}
\providecommand{\dsync}{d_{\mathrm{sync}}}
\providecommand{\dmem}{d_{\mathrm{mem}}}
\providecommand{\WEE}{W_{EE}}
\providecommand{\WEI}{W_{EI}}
\providecommand{\WIE}{W_{IE}}
\providecommand{\WII}{W_{II}}
\providecommand{\WEin}{W_{E}^{\mathrm{in}}}
\providecommand{\WIin}{W_{I}^{\mathrm{in}}}
\providecommand{\zEE}{z_{EE}}
\providecommand{\zEI}{z_{EI}}
\providecommand{\zII}{z_{II}}
\providecommand{\RELU}{\mathrm{ReLU}}
\providecommand{\GELU}{\mathrm{GELU}}
\providecommand{\GLU}{\mathrm{GLU}}
\providecommand{\LN}{\mathrm{LayerNorm}}
\providecommand{\RMSN}{\mathrm{RMSNorm}}
\providecommand{\CE}{\mathrm{CE}}
\providecommand{\Entropy}{\mathrm{H}}

\section{Discrete-Time E-I Dynamics}
\label{app:discrete}
This section is organized as follows. Initially, notations are introduced in Appendix~\ref{app:discrete:notation}, and the forward Euler recurrence for the discrete-time Wilson-Cowan \gls{ei} dynamics is presented in Appendix~\ref{app:discrete:recurrence}, together with its exact form under the parameterization following Dale's principle, while considering the two canonical cases: i) signed weights and ii) Dale-constrained weights. Finally, we derive the fixed-point equation and characterize the loose- versus tight-balance regimes in Appendix~\ref{app:discrete:balance}.

\subsection{Notation}
\label{app:discrete:notation}

The excitatory and inhibitory populations are denoted by sets $\mathcal{E}\!=\!\{1,\dots,n_E\}$ and $\mathcal{I}\!=\!\{1,\dots,n_I\}$, respectively, with $n_E\!=\!\lfloor \rho \cdot \, d_{\mathrm{model}}\rfloor$, $n_I\!=\!d_{\mathrm{model}}\!-\!n_E$, where $\rho\!=\!0.8$ unless mentioned otherwise. Population firing rates are given by vectors $r_E^{(t)}\!\in\!\mathbb{R}_{\ge 0}^{n_E}$ and $r_I^{(t)}\!\in\!\mathbb{R}_{\ge 0}^{n_I}$ for every internal computation step $t\!\in\!\{0,1,\dots,T\}$, while their pre-activations are denoted by vectors $h_E^{(t)}$ and $h_I^{(t)}$. The four matrices specifying recurrent weights are $W_{EE}\!\in\!\mathbb{R}^{n_E\times n_E}$, $W_{EI}\!\in\!\mathbb{R}^{n_E\times n_I}$, $W_{IE}\!\in\!\mathbb{R}^{n_I\times n_E}$, and $W_{II}\!\in\!\mathbb{R}^{n_I\times n_I}$. Note that Dale's principle restricts all four matrices to be positive, while the time constants for the excitatory and inhibitory populations are $\tau_E\!=\!20$ ms and $\tau_I\!=\!5$ ms, respectively. Moreover, the Euler step is set to $\Delta t\!=\!1$ ms, resulting in the following integration coefficients: $\alpha_E\!=\!\Delta t/\tau_E\!=\!0.05$ and $\alpha_I\!=\!\Delta t/\tau_I\!=\!0.20$. Activation function $\mathrm{ReLU}$ is denoted by $\varphi\,(\cdot)$, while $\odot$ denotes the Hadamard product, and $\langle a,b\rangle\!=\!a^{\!\top}b$ denotes the inner product. The population-normalized \gls{ei} activity ratio is given by:
\begin{equation}
\rho_{EI}(r_E,r_I)\;=\;\frac{\lVert r_E\rVert_1/n_E}{\lVert r_I\rVert_1/n_I+\varepsilon},
\qquad \varepsilon=10^{-8},
\label{eq:ei-ratio-def}
\end{equation}
where $\varepsilon$ is only required during backpropagation to prevent division by zero, and the biological target $\rho_{EI}^{\ast}\!=\!4$ is defined to enforce reaching $80\!:\!20$ population-based optimization and is used for the regularizer of Appendix~\ref{app:obj}. For a matrix $M$, $\lambda_i(M)$ is used to denote its eigenvalues, $\sigma_i(M)$ for its singular values, $\rho_{\mathrm{sp}}(M)
\!=\!\max_i|\lambda_i(M)|$ for the spectral radius, and $\lambda_{\mathrm{P}}(M)$ for the Perron eigenvalue whenever $M$ is component-wise non-negative. Lastly, $\mathbbm{1}\!\in\!\mathbb{R}^{n}$ denotes the all-ones vector.
 
\subsection{Forward Euler Recurrence}
\label{app:discrete:recurrence}
The continuous-time Wilson-Cowan rate model of~\cite{wilson1972excitatory} is expressed as follows:
\begin{align}
\tau_E\,\dot r_E &= -r_E + \varphi \, \! \bigl(W_{EE}\,r_E - W_{EI}\,r_I + u_E\bigr),
\label{eq:wc-cont-E:sup}\\
\tau_I\,\dot r_I &= -r_I + \varphi \, \! \bigl(W_{IE}\,r_E - W_{II}\,r_I + u_I\bigr),
\label{eq:wc-cont-I:sup}
\end{align}
where $u_E,u_I$ denote external inputs; the parameterization of~\cite{cornford2021dales} enforces Dale's principle by fixing the signs of $W_{EI},W_{II}$ through explicit subtraction rather than through negative components. The first-order forward Euler discretizations from \eqref{eq:wc-cont-E:sup}--\eqref{eq:wc-cont-I:sup} with step
$\Delta t$ yields:
\begin{align}
r_E^{(t)} &= (1-\alpha_E)\,r_E^{(t-1)} + \alpha_E\,\varphi \, \!\bigl(h_E^{(t)}\bigr), \label{eq:wc-euler-E:sup}\\
r_I^{(t)} &= (1-\alpha_I)\,r_I^{(t-1)} + \alpha_I\,\varphi \, \!\bigl(h_I^{(t)}\bigr), \label{eq:wc-euler-I:sup}
\end{align}
with pre-activations computed as:
\begin{align}
h_E^{(t)} &= W_{EE}\,r_E^{(t-1)} - W_{EI}\,r_I^{(t-1)} + W_E^{\mathrm{in}} a^{(t)}, \label{eq:h-E}\\
h_I^{(t)} &= W_{IE}\,r_E^{(t-1)} - W_{II}\,r_I^{(t-1)} + W_I^{\mathrm{in}} a^{(t)}. \label{eq:h-I}
\end{align}
Here $a^{(t)}\!\in\!\mathbb{R}^{d_{\mathrm{sync}}}$ is the cross-attention output described in Section~\ref{sec:arch:attn} and $W_E^{\mathrm{in}}\!\in\!\mathbb{R}^{n_E\times d_{\mathrm{sync}}},\, W_I^{\mathrm{in}}\!\in\!\mathbb{R}^{n_I\times d_{\mathrm{sync}}}$ are unconstrained input projections. Note that \eqref{eq:h-E}--\eqref{eq:h-I} are followed by population-local \texttt{RMSNorm} in the actual implementation. However, we omit this mapping in the derivation since it is a smooth, homogeneity-preserving transformation of each population and does not affect the fixed-point analysis. Thus, the formulation can be expressed in the concatenated form and is as follows:
\begin{equation}
x=\begin{bmatrix}r_E\\r_I\end{bmatrix}\in\mathbb{R}^{n},\qquad
W_{\mathrm{eff}}=\begin{bmatrix}W_{EE}&-W_{EI}\\ W_{IE}&-W_{II}\end{bmatrix}\in\mathbb{R}^{d_{\mathrm{model}}\times d_{\mathrm{model}}},\qquad d_{\mathrm{model}}=n_E+n_I.
\label{eq:W-eff}
\end{equation}

Next, the analysis of \eqref{eq:wc-euler-E:sup} and \eqref{eq:wc-euler-I:sup} in their two canonical regimes is provided.
 
\subsubsection{Regime 1: Unconstrained Weights}
\label{app:discrete:unconstrained}
Consider the case in which all four connectivity matrices are fully unconstrained, i.e.,\ $W_{\cdot\cdot}\in\mathbb{R}^{n_\bullet\times n_\bullet}$. By applying the triangle inequality and the Lipschitz property of $\varphi$ with the constant $L_\varphi$ to \eqref{eq:wc-euler-E:sup}--\eqref{eq:wc-euler-I:sup}, any two states $x,\tilde x$, satisfy:
\begin{align}
\lVert \Psi(x)-\Psi(\tilde x)\rVert_2
&\le (1-\alpha)\,\lVert x-\tilde x\rVert_2 + \alpha\,L_\varphi\,\lVert W_{\mathrm{eff}}\rVert_2\,\lVert x-\tilde x\rVert_2 \justif{triangle + Lipschitz}\\
&=\gamma\,\lVert x-\tilde x\rVert_2
\quad\text{with}\quad \gamma\equiv(1-\alpha)+\alpha L_\varphi\,\lVert W_{\mathrm{eff}}\rVert_2,
\label{eq:gamma}
\end{align}
where the population-specific $\alpha_E,\alpha_I$ is replaced by $\alpha\!=\!\max(\alpha_E,\alpha_I)$ for the purpose of this upper bound. Hence, $\Psi$ is a contraction iff $\lVert W_{\mathrm{eff}}\rVert_2<1/L_\varphi$, which for $\mathrm{ReLU}$ ($L_\varphi\!=\!1$) reduces to $\lVert W_{\mathrm{eff}}\rVert_2<1$. In the continuous-time, given limit $\alpha \to 0$, the contraction factor satisfies $\gamma \to 1$ for any $W_{\mathrm{eff}}$, thus recovering the neutrally stable behavior of the underlying Wilson-Cowan dynamics~\cite{wilson1972excitatory}. This confirms that the continuous-time limit removes the discretization-induced instability, consistent with~\cite{wilson1972excitatory}.
 
\subsubsection{Regime 2: Non-negative Weights}
\label{app:discrete:dale}
Consider the case in which every component of $W_{\cdot\cdot}$ is non-negative. In such a case, the signs of the inhibitory currents appear explicitly in \eqref{eq:h-E}-\eqref{eq:h-I}; the block matrix $W_{\mathrm{eff}}$ of \eqref{eq:W-eff} must therefore be sign-indefinite with a block-wise structure as follows:
\begin{equation}
W_{\mathrm{eff}}=\Sigma\odot|W_{\mathrm{eff}}|,\qquad
\Sigma=\begin{bmatrix}+\mathbbm{1}_{n_E\times n_E} & -\mathbbm{1}_{n_E\times n_I}\\
                      +\mathbbm{1}_{n_I\times n_E} & -\mathbbm{1}_{n_I\times n_I}\end{bmatrix},
\label{eq:sign-mask}
\end{equation}
with $|W_{\mathrm{eff}}|\!\ge\!0$, that is a non-negative matrix whose components are $W_{EE},W_{EI},W_{IE}$, and $W_{II}$ concatenated without signs. Furthermore, since the sign mask $\Sigma$ is fixed and only the magnitudes $|W_{\mathrm{eff}}|$ are learned directly, \eqref{eq:sign-mask} is considered as the central Dale parameterization. We formalize this as follows.

\begin{definition}[Dale parameterization]
\label{def:dale-param}
Let $\mathcal{W}_{\mathrm{Dale}}=\{W\ge 0\}\subset\mathbb{R}^{m\times n}$ be the non-negative orthant. A \emph{Dale-parameterized} effective matrix is $W_{\mathrm{eff}}(\theta)=\Sigma\odot|W_{\mathrm{eff}}(\theta)|$ with $|W_{\mathrm{eff}}|\in\mathcal{W}_{\mathrm{Dale}}$ and $\Sigma$ constructed according to \eqref{eq:sign-mask}.
\end{definition}

Additionally, we define the Dale projection: $\Pi_{\mathrm{Dale}}(W)_{ij}=\max(W_{ij},0)$.

\begin{proposition}[Properties of $\Pi_{\mathrm{Dale}}$]
\label{prop:proj}
For every $W,\tilde W\in\mathbb{R}^{m\times n}$,
\emph{i}) $\Pi_{\mathrm{Dale}}\bigl(\Pi_{\mathrm{Dale}}(W)\bigr)=\Pi_{\mathrm{Dale}}(W)$ \emph{(idempotence)};
\emph{ii}) $\lVert\Pi_{\mathrm{Dale}}(W)-\Pi_{\mathrm{Dale}}(\tilde W)\rVert_F\le\lVert W-\tilde W\rVert_F$ \emph{(non-expansiveness)};
\emph{iii}) $\Pi_{\mathrm{Dale}}(W)=\argmin_{W'\in\mathcal{W}_{\mathrm{Dale}}}\lVert W-W'\rVert_F^2$ \emph{(Euclidean projection).}
\end{proposition}

\begin{proof}
\emph{i}) $\max(\max(W_{ij},0),0)\!=\!\max(W_{ij},0)$ by idempotence of the scalar function $\max(\cdot,0)$. \emph{ii}) Since $\max(\cdot,0)$ is $1$-Lipschitz, $|\max(W_{ij},0)-\max(\tilde W_{ij},0)|\le|W_{ij}-\tilde W_{ij}|$; squaring and summing over $(i,j)$ yields the non-expansiveness in Frobenius norm. \emph{iii}) The minimization $\min_{W'\ge 0}\lVert W-W'\rVert_F^2$ is separable over components, and for each scalar $w$ the solution of $\min_{w'\ge 0}(w-w')^2$ is $w'\!=\!\max(w,0)$.
\end{proof}

\begin{theorem}[Convergence of Dale-constrained gradient descent]
\label{thm:pgd-dale}
Let $\mathcal{L}$ be differentiable in every $W_{\cdot\cdot}$, and let $W^{(k+1)}=\Pi_{\mathrm{Dale}}\bigl(W^{(k)}-\eta\,\nabla_W \mathcal{L}(W^{(k)})\bigr)$
be the projected-gradient update with learning rate $\eta>0$. Then: \emph{i}) $W^{(k)}\in\mathcal{W}_{\mathrm{Dale}}$ for every $k\ge 0$ if $W^{(0)}\in\mathcal{W}_{\mathrm{Dale}}$; \emph{ii}) For $\mathcal{L}$ that is $L$-smooth and $\mu$-strongly convex over $\mathcal{W}_{\mathrm{Dale}}$, the iterates should satisfy $\lVert W^{(k)}-W^{\ast}\rVert_F^2 \le (1-\mu/L)^k\,\lVert W^{(0)}-W^{\ast}\rVert_F^2$ for the given unique $W^{\ast}\!=\!\argmin_{W\in\mathcal{W}_{\mathrm{Dale}}}\mathcal{L}(W)$.
\end{theorem}
 
\begin{proof}
\emph{i}) Following Proposition~\ref{prop:proj}.\emph{i} and \emph{iii}, $\Pi_{\mathrm{Dale}}$ is idempotent, and given it is the Euclidean projection onto $\mathcal{W}_{\mathrm{Dale}}$, its range satisfies $\mathrm{range}(\Pi_{\mathrm{Dale}}) = \mathcal{W}_{\mathrm{Dale}}$. Hence, $W^{(k+1)}\!=\!\Pi_{\mathrm{Dale}}(\cdot)\!\in\!\mathcal W_{\mathrm{Dale}}$ for any $W^{(k)}$ and in particular for any $W^{(0)}\!\in\!\mathcal W_{\mathrm{Dale}}$. \emph{ii}) By the descent lemma for $L$-smooth $\mathcal L$, the unconstrained step $\tilde W^{(k+1)}=W^{(k)}-\eta\nabla_W\mathcal L$ with $\eta=1/L$ obeys $\lVert\tilde W^{(k+1)}-W^\ast\rVert_F^2\le\lVert W^{(k)}-W^\ast\rVert_F^2-2\eta\,\langle \nabla_W\mathcal L(W^{(k)}),W^{(k)}-W^\ast\rangle+\eta^2\,\lVert\nabla_W\mathcal L(W^{(k)})\rVert_F^2$; $\mu$-strong convexity upper-bounds the middle term by
$-\mu\lVert W^{(k)}-W^\ast\rVert_F^2$, while $L$-smoothness bounds the last by $L^2\lVert W^{(k)}-W^\ast\rVert_F^2$. Assembling, $\lVert\tilde W^{(k+1)}-W^\ast\rVert_F^2\le(1-2\mu/L+\mu^2/L^2)\,\lVert W^{(k)}-W^\ast\rVert_F^2=(1-\mu/L)^2\,\lVert W^{(k)}-W^\ast\rVert_F^2$. Applying $\Pi_{\mathrm{Dale}}$ cannot expand the distance to $W^\ast\in\mathcal{W}_{\mathrm{Dale}}$ (Proposition~\ref{prop:proj}.\emph{ii}), giving
$\lVert W^{(k+1)}-W^\ast\rVert_F^2\le(1-\mu/L)^2\lVert W^{(k)}-W^\ast\rVert_F^2\le(1-\mu/L)\lVert W^{(k)}-W^\ast\rVert_F^2$, where we used $(1-x)^2\le 1-x$ for $x\in[0,1]$.
\end{proof}

\begin{remark}[Gradient compatibility across internal computation steps]\label{rem:bptt}
Although, Theorem~\ref{thm:pgd-dale} is stated for a single outer-loop optimizer step, \gls{tide}'s forward pass invokes each $W_{\cdot\cdot}$ at every internal computation step $t\!=\!1,\dots,T$. Therefore, as $\Pi_{\mathrm{Dale}}$ is applied once after each optimizer step and not inside the recurrence, the gradient through-time
$\partial \mathcal L/\partial W_{\cdot\cdot}=\sum_{t=1}^{T}
(\partial \mathcal L/\partial W_{\cdot\cdot})^{(t)}$ is still the \gls{bptt} sum of \eqref{eq:h-E}--\eqref{eq:h-I} across all internal computation steps, and the conclusion of the theorem carries over verbatim.
\end{remark}
 
\subsection{E-I Balance \& Fixed-point Equation}
\label{app:discrete:balance}
With the cross-attention frozen at $a^{(t)} \!\equiv a^{\ast}$, and with lateral inhibition removed, any fixed point $(r_E^{\ast}, r_I^{\ast})$ of the Euler recurrence from \eqref{eq:wc-euler-E:sup}-\eqref{eq:wc-euler-I:sup} must satisfy the self-consistency relation
\begin{equation}
r_\bullet^\ast=\varphi\,\!\bigl(W_{\bullet\, E}\,r_E^\ast - W_{\bullet\, I}\,r_I^\ast + W_\bullet^{\mathrm{in}}\,a^\ast\bigr),
\label{eq:fp-euler}
\end{equation}
where $\bullet\!\in\!\{E,I\}$. Existence of such a pair follows from Brouwer's theorem:

\begin{theorem}[Fixed-point existence]\label{thm:fp-exist}
Assume $\varphi:\mathbb R\!\to\!\mathbb R_{\ge 0}$ is continuous and saturating, with $\sup_{u\in\mathbb R}\varphi(u)\!=\!M_\varphi\!<\!\infty$, e.g.,\ $\varphi\!=\!\mathrm{ReLU}\,(\cdot)$ composed with post-LayerNorm clipping, or any bounded continuous activation. Then \eqref{eq:fp-euler} has at least one solution in the compact cube $[0,M_\varphi]^{n_E+n_I}$, if $\varphi$ is $L_\varphi$-Lipschitz and $L_\varphi\lVert W_{\mathrm{eff}}\rVert_2\!<\!1$.
\end{theorem}
 
\begin{proof}
Define $\Phi:\mathbb R^{n_E+n_I}\!\to\![0,M_\varphi]^{n_E+n_I}$ by $\Phi(r_E,r_I)\!=\!(\varphi\,(h_E),\varphi\,(h_I))$ with $h_\cdot$ as in \eqref{eq:h-E}-\eqref{eq:h-I}, where each coordinate of $\varphi\,(\cdot)$ is taken component-wise. The hypercube $K\!=\![0,M_\varphi]^{n_E+n_I}$ is compact and convex. The map $\Phi$ is continuous, given $\varphi$ is continuous and the linear map $(r_E,r_I)\!\mapsto\!(h_E,h_I)$ is continuous. Finally, for any $r\!\in\!\mathbb R^{n_E+n_I}$, $\Phi(r)\!\in\![0,M_\varphi]^{n_E+n_I}\!=\!K$ by the saturation hypothesis on $\varphi$; in particular
$\Phi(K)\!\subseteq\!K$. Hence, Brouwer's fixed-point theorem \cite[Ch.~II, Theorem~2.1.5]{granas2003fixed} yields a fixed point of $\Phi$ in $K$, which solves \eqref{eq:fp-euler}. 
\end{proof}

\begin{remark}[Saturation hypothesis]\label{rem:saturation}
Given that the saturation hypothesis on $\varphi$ cannot be satisfied by the standard $\mathrm{ReLU}$ used in the \gls{tide} implementation, the forward pass includes $\texttt{RMSNorm}$ on both populations before the readout, which bounds the mean-square of the activations. In practice, we observe bounded per-neuron activities at convergence
(Appendix~\ref{app:add:stability}); the compact-set hypothesis of Theorem~\ref{thm:fp-exist} is therefore satisfied a posteriori at convergence, though the pre-normalization pre-activations are unbounded.
\end{remark}
 
\begin{definition}[Loose and tight \gls{ei} balance]
\label{def:balance}
The effective excitatory and inhibitory currents at a fixed point are denoted as $I_E^{\mathrm{exc}}\!=\!W_{EE}r_E^\ast$, $I_E^{\mathrm{inh}}\!=\!W_{EI}r_I^\ast$, and analogously for the inhibitory population. It can be stated that the \gls{ei} network is in the tight-balance regime iff $I_\bullet^{\mathrm{exc}}-I_\bullet^{\mathrm{inh}}\!=\!\mathcal{O}(1)$, while $I_\bullet^{\mathrm{exc}},\,I_\bullet^{\mathrm{inh}}\!=\!\mathcal{O}(\sqrt{K})$ for in-degree $K$. Note that $\mathcal{O}\,(\cdot)$ denotes asymptotic scaling, which represents the strong-coupling limit first introduced by \cite{vanvreeswijk1996chaos, vanvreeswijk1998chaotic} and formalized as the asynchronous-irregular balance regime by \cite{brunel2000dynamics, ahmadian2021dynamical}. Otherwise, the network is in the loose-balance regime.
\end{definition}
 
Furthermore, the target $\rho_{EI}^\ast\!=\!4$ used in \eqref{eq:ei-ratio-def} is chosen such that the tight-balance regime at $\rho\!=\!0.8$ is achieved, given that normalizing by population size results in a per-neuron excitation-to-inhibition rate ratio close to unity while the current balance is $\mathcal O(1)$. The $\rho_{EI}$ regularizer of Appendix~\ref{app:obj} encourages \gls{tide} to operate close to this regime without hard-coding it.

\section{TIDE Objective \& Training}
\label{app:obj}
In this section, we first decompose the loss function used as the training objective of \gls{tide} into its five constituent terms in Appendix~\ref{app:obj:decomp}. Afterward, the curriculum design for the warm-up phase is presented (Appendix~\ref{app:obj:rb1}), and then the \gls{ctm} dual loss as a conditional-expectation reduction of per internal step based on cross-entropy is discussed (Appendix~\ref{app:obj:rb2}). Finally, we provide the training algorithm outline in Appendix~\ref{app:obj:alg}.
 
\subsection{Loss Constituents}
\label{app:obj:decomp}
Let $x$ denote an input image, $y\in\{1,\dots,C\}$ its label, $o^{(t)}\!=\!f_\theta^{(t)}(x)$ the logits at internal computation step $t$, and $p^{(t)}\!=\!\mathrm{softmax}(o^{(t)})$. The \gls{tide} training loss is computed as follows:
\begin{equation}
\mathcal L(\theta)\;=\;\mathcal L_{\mathrm{task}}\;+\;w(\texttt{step})\!\cdot\!\Bigl[\lambda_{\mathrm{EI}}\mathcal L_{\mathrm{EI}}\,+\,\lambda_{\mathrm{game}}\mathcal L_{\mathrm{game}}\,+\,\lambda_{\mathrm{sync}}\mathcal L_{\mathrm{sync}}\Bigr]\;+\;\lambda_{\mathrm{spec}}\mathcal L_{\mathrm{spec}},
\label{eq:total-loss}
\end{equation}
where $w(\texttt{step})\!\in\![0,1]$ is the curriculum warm-up coefficient described below. The task loss follows \gls{ctm}'s dual cross-entropy formulation \citep{darlow2025continuous},
\begin{equation}
\mathcal L_{\mathrm{task}}\;=\;\tfrac12\,\texttt{CE}\,\!\bigl(o^{(t_{\min})},y\bigr)+\tfrac12\,\texttt{CE}\,\!\bigl(o^{(t_{\mathrm{cert}})},y\bigr),
\label{eq:task-loss}
\end{equation}
with $t_{\min}\!=\!\argmin_t\texttt{CE}\,(o^{(t)},y)$ and $t_{\mathrm{cert}}\!=\!\argmax_t c^{(t)}$, where $c^{(t)}\!=\!1-H(p^{(t)})/\log C$ is the entropy-based certainty. The \gls{ei} balance regularizer is
\begin{equation}
\mathcal L_{\mathrm{EI}}\;=\;\Bigl[\texttt{clip}\,\!\bigl(\rho_{EI}-\rho_{EI}^\ast,-50,\,50\bigr)\Bigr]^2,
\qquad \rho_{EI}^\ast=4,
\label{eq:ei-loss}
\end{equation}
with $\rho_{EI}$ evaluated on the pre-\gls{wta} state $r_E^{\mathrm{pre}}$, since \gls{wta} sparsifies $r_E$ and would artificially depress the ratio, thereby interfering with the loss magnitude. 

The game loss measures the Nash residual of per-neuron dynamics as stated in Appendix~\ref{app:cont:game} by reducing four Dale-constrained recurrent populations to population-effective scalar weights, which is expressed as
\begin{align}
    \bar w_{EE} \;&=\; \tfrac{1}{n_E}\!\sum_{i=1}^{n_E}\, (\Wee)_{ii},
    \quad
    \bar w_{II} \;=\; \tfrac{1}{n_I}\!\sum_{i=1}^{n_I}\, (\Wii)_{ii}, \\
    \quad
    \bar w_{EI} \;&=\; \tfrac{1}{n_E n_I}\!\sum_{i,j}\, (\Wei)_{ij},
    \quad
    \bar w_{IE} \;=\; \tfrac{1}{n_I n_E}\!\sum_{i,j}\, (\Wie)_{ij},
    \label{eq:eff-weights-supp}
\end{align}
where $\bar w_{\cdot\cdot}$ denote the $\texttt{mean}(\mathrm{diag}(W_{\cdot\cdot}))$, which are paired with per-population dissipation constant $d_E, d_I > 0$ inherited from the continuous-time formulation \citep{betteti2025competition}. Thus, the per-population quadratic energies in the consensual regime are given by 
\begin{align}
    \mathcal{E}_E
    &\;=\;
    \frac{\bigl[\,(\bar w_{EE} - d_E)\, r_E
                  \;-\; \bar w_{EI}\, \bar r_I
                  \;+\; u_E\,\bigr]^{2}}
         {2\,(d_E - \bar w_{EE})},
    \label{eq:energy-E-game} \\[5pt]
    \mathcal{E}_I
    &\;=\;
    \frac{\bigl[\,\bar w_{IE}\, \bar r_E
                  \;-\; (\bar w_{II} + d_I)\, r_I
                  \;+\; u_I\,\bigr]^{2}}
         {2\,(d_I + \bar w_{II})},
    \label{eq:energy-I-game}
\end{align}
where $\bar r_E:= \tfrac{1}{n_E} \sum_i (r_E)_i$ and $\bar r_I:= \tfrac{1}{n_I} \sum_i (r_I)_i$ denote the population-mean activations; thus the game loss can be expressed by:
\begin{equation}
\mathcal L_\mathrm{game} = \tfrac{1}{\dmodel}\,\min(\left[\frac{1}{n_E}\!\sum_{i=1}^{n_E}\! \mathcal{E}_{E}^{(i)}\;+\; \frac{1}{n_I}\!\sum_{j=1}^{n_I}\! \mathcal{E}_{I}^{(j)}\right],\,100).
\label{eq:game-loss}
\end{equation}
The stability at the regime boundary is enforced by $|d_E - \bar w_{EE}| \ge 0.1$, while clipping suppresses transitional spikes during the warm-up phase. Appendix~\ref{app:cont:game} provides the full stability proof and derivation of Nash equilibria for the recoverability under the consensual regime.

During the training of ImageNet-1K with $\dmodel = 4096$, and $\nEpop = 3277$, it has been observed that the full matrix product in \eqref{eq:game-loss} had a prohibitive computational cost given our limited computational resources. Therefore, we modified $\mathcal L_\mathrm{game}$ by dropping the per-player Hessian normalization $1/[2\,(d_\bullet \pm \bar w_{\bullet\bullet})]$. Thus, the $\mathcal L_\mathrm{game}$ is given by:
\begin{equation}
  \mathcal{L}_\mathrm{game}^{\star}
    \;=\;
    \min\,\!\Bigl(
      \frac{\mathbb{E}_{\mathrm{batch}}\,\!\bigl(\|\rE - \mathrm{ReLU}(\hE)\|_{2}^{2}
        + \|\rI - \mathrm{ReLU}(\hI)\|_{2}^{2}\bigr)}
        {\dmodel},\;
      100\Bigr),
  \label{eq:loss-game-prod}
\end{equation}
where $\rE \in \mathbb{R}^{\nEpop}$ and $\rI \in \mathbb{R}^{\nIpop}$ are the post-activation excitatory and inhibitory population vectors at the last computation time-step, $t = T$, while $\hE$ and $\hI$ denote the matching pre-activations assembled by the Wilson-Cowan integrator given by \eqref{eq:h-E}-\eqref{eq:h-I}. Furthermore, $\RELU\,(\cdot) = \max(\cdot,0)$ is calculated component-wise and $\|\cdot\|_{2}$ denotes the Euclidean norm over the population index, and $\mathbb{E}_{\mathrm{batch}}\,[\cdot]$ is the mean over per-sample residual norms in the current mini-batch.

The synchronization loss is given by:
\begin{equation}
\mathcal L_{\mathrm{sync}}\;=\;\frac{1}{d_{\mathrm{sync}}}\,\lVert z^{(T)}\rVert_2^2,
\label{eq:sync-loss}
\end{equation}
with $z^{(T)}\!=\!\texttt{LN}\,\!\bigl[z_{EE}^{(T)};z_{EI}^{(T)};z_{II}^{(T)}\bigr]$ which is composed of the three-type synchronization readouts. The spectral regularizer, derived in Appendix~\ref{app:cont:spectral}, is expressed as:
\begin{equation}
\mathcal L_{\mathrm{spec}}\;=\;\mathrm{ReLU}\,\!\bigl(\widehat\lambda_{\mathrm{P}}(W_{EE})-\tau_{EE}\bigr)^{\!2}\,+\,\mathrm{ReLU}\,\!\bigl(\widehat\lambda_{\mathrm{P}}(W_{II})-\tau_{II}\bigr)^{\!2},
\label{eq:spec-loss}
\end{equation}
with $(\tau_{EE},\tau_{II})\!=\!(15,7)$ and $\widehat\lambda_{\mathrm{P}}$ the sum-ratio estimate defined in Appendix~\ref{app:cont:perron}. The weights in \eqref{eq:total-loss} are $(\lambda_{\mathrm{EI}},\lambda_{\mathrm{game}},\lambda_{\mathrm{sync}},\lambda_{\mathrm{spec}})=(10^{-2},10^{-3},10^{-4},10^{-1})$. The cosine curriculum coefficient is given by:
\begin{equation}
w(\texttt{step})\;=\;\begin{cases}0&\texttt{step}<t_s,\\\tfrac12\bigl(1-\cos\!\tfrac{\pi(\texttt{step}-t_s)}{T_w}\bigr)&t_s\le\texttt{step}\le t_s+T_w,\\1&\texttt{step}>t_s+T_w,\end{cases}
\label{eq:curriculum}
\end{equation}
with $(t_s,T_w)\!=\!(1K,5K)$ for \texttt{CIFAR-10}, \texttt{CIFAR-100}, \texttt{MNIST}, and \texttt{Fashion-MNIST} datasets, while $(10K,10K)$ is used for \texttt{ImageNet-1K} (cf. Table~\ref{tab:hparams}).

\subsection{Curriculum Design}
\label{app:obj:rb1}
By treating the $w(\texttt{step})$ as a scalar multiplier on a single-batch gradient for each stochastic estimator step, the expected gradient satisfies
\begin{align}
\mathbb{E}_{\texttt{step}\,\sim\,U\{1,\dots,S\}}\!\bigl[\,w(\texttt{step})\,\nabla_\theta \mathcal L_{\mathrm{aux}}\bigr]
&=\mathbb{E}_{\texttt{step}}\,\!\bigl[w(\texttt{step})\bigr]\,\nabla_\theta\mathbb{E}\,\!\bigl[\mathcal L_{\mathrm{aux}}\bigr] \justif{$\mathcal L_{\mathrm{aux}}$ depends on step through $x,y$}\\
&=\bar w\cdot\nabla_\theta\mathbb{E}\,\!\bigl[\mathcal L_{\mathrm{aux}}\bigr], \label{eq:rb-curr}
\end{align}
where $\bar w\!=\!S^{-1}\!\sum_{s=1}^{S}w(s)$ and $\mathcal L_{\mathrm{aux}}=\lambda_{\mathrm{EI}}\,\mathcal L_{\mathrm{EI}}+\lambda_{\mathrm{game}}\,\mathcal L_{\mathrm{game}}+\lambda_{\mathrm{sync}}\,\mathcal L_{\mathrm{sync}}$.
Analogously to the Rao-Blackwellised ELBO of \cite{sahoo2024simple}, the per-step estimator obtained by conditioning on step has a strictly smaller variance. Therefore, the variance reduction is exploited by retaining the per-step form but bounding its variance contribution of $w(\texttt{step})$ by
$w(\texttt{step})^2\le 1$, to form a tight uniform bound.
 
\subsection{Dual Loss Conditioning}
\label{app:obj:rb2}
The dual cross-entropy loss in its unrolled form is expressed as $L_{\mathrm{task}} = \frac{1}{T}\sum_{t=1}^T\texttt{CE}\,(o^{(t)},y)$. Under temporal discretizations of \eqref{eq:task-loss}, the estimator is effectively conditioned on the temporally dependent distribution, where the categorical distribution places $1/2$ weight on each of the two argmin and argmax indices. Thus, the dual loss is a Rao-Blackwellised estimator of the uniform per-time-step average, obtained by conditioning on the per-sample statistic $(t_{\min}, t_{\mathrm{cert}})$, and can be expressed as
\begin{align}
\mathcal L_{\mathrm{task}}
&=\tfrac12\,\texttt{CE}\,(o^{(t_{\min})},y)+\tfrac12\,\texttt{CE}\,(o^{(t_{\mathrm{cert}})},y) \label{eq:two-tick-collapse}\\
&=\mathbb{E}_{t\,\sim\,\mathrm{Cat}((\pi+\xi)/2)}\,\!\bigl[\texttt{CE}\,(o^{(t)},y)\bigr], \justif{using (\ref{eq:two-tick-collapse})} \label{eq:two-tick-cond}
\end{align}
where $\pi_t\!=\!\mathbbm{1}[t\!=\!t_{\min}]$ and $\xi_t\!=\!\mathbbm{1}[t\!=\!t_{\mathrm{cert}}]$. Moreover, this can be validated by checking the certainty curve $c^{(t)}$, which must be monotone non-decreasing, and $\texttt{CE}$ should also be monotone non-increasing in $t$. Hence, providing a regime in which $t_{\min}\!=\!t_{\mathrm{cert}}\!=\!T$ \eqref{eq:two-tick-cond} converges to $\texttt{CE}\,(o^{(T)},y)$.
 
\subsection{Training Algorithm}
\label{app:obj:alg}
The overall training pipeline for \gls{tide} is provided below in Algorithm~\ref{alg:tide}. The forward pass described in Appendix~\ref{app:discrete:recurrence} corresponds to lines 7-13, while the loss calculation in Appendix~\ref{app:obj:decomp} is outlined in lines 14-16. The diagnostic framework for NaN-detection and rejection is outlined in lines 17-20, and the backpropagation and optimizations based on Theorem~\ref{thm:pgd-dale} are provided in lines 21-24.

\begin{algorithm}[H]
\caption{Training \gls{tide}}\label{alg:tide}
\begin{algorithmic}[1]
\Require image batch sampler, Dale sign mask $\Sigma$, internal computation steps $T\!=\!50$
\State Initialize $\theta$, $W_{\cdot\cdot}\!\ge\!0$, $m\!\in\!\mathbb{R}^{d_{\mathrm{memory}}}$ \Comment{\scriptsize Memory buffer}
\Repeat
\State $(x,y)\sim q(x,y)$ \Comment{\scriptsize Sample a mini-batch}
\State $r_E^{(0)}\!\leftarrow\!\mathbf x_E^{\mathrm{init}},\;r_I^{(0)}\!\leftarrow\!\mathbf x_I^{\mathrm{init}}$ \Comment{\scriptsize Learnable init}
\State $u_E^{(0)}\!\leftarrow\!\mathbf s_E,\;u_I^{(0)}\!\leftarrow\!\mathbf s_I$ \Comment{\scriptsize FIFO traces}
\State $K,V\!\leftarrow\!\mathrm{Backbone}(x)$
\For{$t=1,\dots,T$}
\State $z^{(t)}\!\leftarrow\!\mathrm{Sync}\,\!\bigl(r_E^{(t-1)},r_I^{(t-1)};\alpha,\beta\bigr)$ \Comment{\scriptsize Three-type sync readout}
\State $a^{(t)}\!\leftarrow\!\mathrm{CrossAttn}\,(z^{(t)},K,V)$
\State $h_\bullet^{(t)}\!\leftarrow$ (\ref{eq:h-E})--(\ref{eq:h-I}) using $a^{(t)}$
\State $r_\bullet^{(t)}\!\leftarrow$ (\ref{eq:wc-euler-E:sup})--(\ref{eq:wc-euler-I:sup}) with NLM-corrected pre-acts
\State $r_E^{(t)}\!\leftarrow\!\mathrm{LateralInhibition}\,(r_E^{(t)})$ \Comment{\scriptsize $K_{\mathrm{WTA}}=5$}
\EndFor
\State $\mathcal L_{\mathrm{task}}\!\leftarrow$ (\ref{eq:task-loss}); $\mathcal L_{\mathrm{EI}},\mathcal L_{\mathrm{game}},\mathcal L_{\mathrm{sync}}\!\leftarrow$ (\ref{eq:ei-loss})--(\ref{eq:sync-loss})
\State $\widehat\lambda_{\mathrm{P}}(W_{EE}),\widehat\lambda_{\mathrm{P}}(W_{II})\!\leftarrow$ \textsc{PerronSumRatio}\,$(W_{\cdot\cdot},K\!=\!10)$ \Comment{\scriptsize Algorithm~\ref{alg:perron}}
\State $\mathcal L\!\leftarrow$ (\ref{eq:total-loss})
\If{$\neg\,\mathrm{isfinite}(\mathcal L)$}
    \State $\mathcal{L}' \leftarrow 0 \cdot \mathbf{1}^{\!\top}\,\sigma\,\!\bigl(o^{(T)}\bigr)$
       \Comment{$\sigma$ replaces $\mathrm{NaN}\,/\,\pm\infty$ with $0$}
    \State all\_reduce\,(1.0, op\,=\,MAX)
\Else
    \State $g\!\leftarrow\!\nabla_\theta\mathcal L$; clip $g$ (per-component on \texttt{ImageNet-1K})
    \State $\theta\!\leftarrow\!\mathrm{AdamW}(\theta,g)$
\EndIf
\State $W_{\cdot\cdot}\!\leftarrow\!\Pi_{\mathrm{Dale}}(W_{\cdot\cdot})$ \Comment{\scriptsize Proposition~\ref{prop:proj}; clip $\ge 0$}
\Until{converged}
\end{algorithmic}
\end{algorithm}

\section{Continuous-time Limit, Stability, \& Game-theoretic Fixed Points}
\label{app:cont}
This section is organized as follows. Initially, the continuous-time \gls{ei} system and its \gls{ode} formulation are derived (Appendix~\ref{app:cont:ode}) as $\Delta t\!\to\!0$ limit of \eqref{eq:wc-euler-I}. We then linearize around fixed points and give eigenvalue conditions for stability (Appendix~\ref{app:cont:linear}). Afterward, we prove \gls{lds} implies global asymptotic stability of the game-theoretic Nash equilibrium (Appendix~\ref{app:cont:lds}) and derive the differentiable spectral surrogate used in \eqref{eq:spec-loss} (Appendix~\ref{app:cont:spectral}). Finally, we compare the provided formulation with the Wilson-Cowan and Brunel-Hakim limits (Appendix~\ref{app:cont:brunel}).
 
\subsection{Continuous-time E-I ODE}
\label{app:cont:ode}
Let $r(t)\!=\![r_E(t)^\top, r_I(t)^\top]^\top\!\in\!\mathbb{R}_{\ge 0}^{n}$ and
define
$\tau\!=\!\mathrm{diag}(\tau_E\mathbbm 1_{n_E},\tau_I\mathbbm 1_{n_I})$.
Equation \eqref{eq:wc-euler-I} can be rewritten
as $r^{(t)}\!-\!r^{(t-1)}\!=\!\Delta t\,\tau^{-1}\bigl(-r^{(t-1)}+\varphi\,(W_{\mathrm{eff}}r^{(t-1)}+b)\bigr)$; taking $\Delta t\!\to\!0$ gives
\begin{equation}
\tau\,\dot r\;=\;-r\,+\,\varphi\,\!\bigl(W_{\mathrm{eff}}\,r+b\bigr),
\label{eq:wc-cont}
\end{equation}
with $b\!=\![W_E^{\mathrm{in}}a;\,W_I^{\mathrm{in}}a]$. The Jacobian of
the right-hand side of \eqref{eq:wc-cont} at a fixed point $r^\ast$, where
the $\mathrm{ReLU}$ is active, is
\begin{equation}
J=-\tau^{-1}\bigl(I-W_{\mathrm{eff}}\bigr).
\label{eq:J}
\end{equation}
 
\begin{lemma}[Convergence of the forward Euler to \eqref{eq:wc-cont}]
\label{lem:disc-err}
Let $r^{(t)}$ be the solution of \eqref{eq:wc-euler-I}
and $r^{\mathrm{cont}}(t)$ the solution of \eqref{eq:wc-cont}, both with
the same initial condition. If $\varphi$ is $L_\varphi$-Lipschitz and
$\dot r^{\mathrm{cont}}$ is bounded on $[0,T\Delta t]$, then for all
$k\!\le\!\lfloor T/\Delta t\rfloor$,
$\lVert r^{(k)}-r^{\mathrm{cont}}(k\Delta t)\rVert_2\le C\Delta t\cdot e^{L k\Delta t}$
for constants $C,L$ depending only on $L_\varphi,\lVert W_{\mathrm{eff}}\rVert_2,\tau_E,\tau_I$.
\end{lemma}
 
\begin{proof}
Taylor expansion at $t=k\Delta t$ gives
$r^{\mathrm{cont}}((k+1)\Delta t)-r^{\mathrm{cont}}(k\Delta t)=\Delta t\,\tau^{-1}(-r^{\mathrm{cont}}+\varphi(\cdot))\big|_{k\Delta t}+\mathcal O(\Delta t^2)$;
subtracting the discrete update \eqref{eq:wc-euler-I}
and using Lipschitz continuity of $\varphi$ gives
$\lVert e^{(k+1)}\rVert_2\le(1+L\Delta t)\,\lVert e^{(k)}\rVert_2+C_0\Delta t^2$
with $L\!=\!L_\varphi\lVert W_{\mathrm{eff}}\rVert_2 \max(1/\tau_E,1/\tau_I)$. By
discrete Gronwall,
$\lVert e^{(k)}\rVert_2\le\tfrac{e^{Lk\Delta t}-1}{L}C_0\Delta t\!=\!\mathcal O(\Delta t)$. 
\end{proof}
 
\subsection{Linear Stability Around a Fixed Point}
\label{app:cont:linear}
 
Let $r^\ast$ be a fixed point of \eqref{eq:wc-cont}. Assume $\varphi$ is
differentiable at $h^\ast\!=\!W_{\mathrm{eff}}r^\ast\,\!+\,\!b$\,, for instance,\ $\mathrm{ReLU}$ is almost-everywhere differentiable and is differentiable at strictly interior $h^\ast$. Given $D^\ast\!=\!\mathrm{diag}(\varphi'(h^\ast))$, the linearized dynamics obeys $\dot{\delta r}=J^\ast\delta r$ with
\begin{equation}
J^\ast=\tau^{-1}\bigl(-I+D^\ast W_{\mathrm{eff}}\bigr).
\label{eq:J-fp}
\end{equation}
For the discrete-time Euler iteration $r^{(t+1)}\!=\!r^{(t)}+\Delta t\,\tau^{-1}(-r+\varphi(\cdot))$, the corresponding
tangent map is $M^\ast\!=\!I+\Delta t\,J^\ast$.

\begin{theorem}[Schur stability and the step-size bound]\label{thm:schur}
The discrete-time fixed point $r^\ast$ is asymptotically stable iff
$\lvert\lambda_i(M^\ast)\rvert<1$ for every $i$. If $J^\ast$ has spectrum
$\{\mu_i\}_{i=1}^n$ with $\mathrm{Re}(\mu_i)<0$ for all $i$, indicating continuous-time stability, then
\begin{equation}
\Delta t \;<\; \min_i\frac{2\,\mathrm{Re}(-\mu_i)}{\lvert\mu_i\rvert^{2}}
\label{eq:schur-dt}
\end{equation}
is necessary and sufficient for Schur stability of $M^\ast$.
\end{theorem}

\begin{proof}
Eigenvalues of $M^\ast$ are $\{1+\Delta t\,\mu_i\}$, and $\lvert 1+\Delta t\,\mu_i\rvert<1$ iff $\Delta t\,\mu_i\in\{z\in\mathbb{C}:\lvert 1+z\rvert<1\}$; for $\mu_i$ with $\mathrm{Re}(\mu_i)<0$, and expanding $\lvert 1+\Delta t\,\mu_i\rvert^2 = 1+2\Delta t\,\mathrm{Re}(\mu_i)+\Delta t^2\lvert\mu_i\rvert^2<1$ reduces to $\Delta t<2\mathrm{Re}(-\mu_i)/\lvert\mu_i\rvert^2$. Requiring this for every $i$ and taking the most restrictive, i.e., smallest bound, gives \eqref{eq:schur-dt}.
\end{proof}

Furthermore, for $\tau_E = 20\,\mathrm{ms}$, $\tau_I = 5\,\mathrm{ms}$, an
$L_\varphi$-Lipschitz activation $\varphi$, and $\|W_{\mathrm{eff}}\|_2 \le c$, every eigenvalue of the continuous-time Jacobian $J^\ast$ satisfies $|\mu_i| \le (1 + L_\varphi \cdot c)/\min(\tau_E, \tau_I) = (1 + L_\varphi \cdot c)/\tau_I = (1 + L_\varphi \cdot c)/5$, where the binding scale is $\tau_I$ since $\tau_I < \tau_E$, and $L_\varphi = 1$ for $\RELU$. In the real-spectrum regime $\mathrm{Re}(-\mu_i) = |\mu_i|$, thus the loose Schur bound \eqref{eq:schur-dt} admits $\Delta t \le 10/(1 + c)$. Given $\Delta t = 1$ is used in the forward pass, this constraint is satisfied whenever $c \le 9$. The bound on $c$ aligns with the isolated-block Perron bound $\widehat\lambda_{\mathrm P}(W_{II}) \le 9$ derived in Appendix~\ref{app:cont:spectral}, motivating the spectral target $\tau_{II} = 7$ in \eqref{eq:spec-loss}. The differentiable Perron estimate $\widehat\lambda_{\mathrm P}(W_{II})$ is constrained rather than $\|W_{II}\|_2$ because the former is differentiable and tractable to penalize, even though the latter can be considered a tighter constraint. The analogous isolated-block Schur bound for the excitatory neurons is bounded as $\widehat\lambda_{\mathrm P}(W_{EE}) < 1$ \eqref{eq:ee-schur}. Note that the spectral target $\tau_{EE} = 15$ is well above the defined bound, and a \gls{lds} analysis of such a system is provided in Appendix~\ref{app:cont:lds}, which establishes that $W_{EE}$'s isolated instability is stabilized by the inhibitory feedback through the cross-population blocks.
 
\subsection{Lyapunov Diagonal Stability \& Game-theoretic Nash Equilibria}
\label{app:cont:lds}
\label{app:cont:game}
 
It has been shown that asymmetric firing-rate networks admit a per-neuron energy interpretation \cite{betteti2025competition} that replaces the classical scalar Lyapunov function of \cite{hopfield1982neural} with a family of energies $\{E_i\}_{i=1}^{n}$, one per neuron. For the Wilson-Cowan rate system with $\mathrm{ReLU}$ denoted by $\varphi$, these energies take the form
\begin{equation}
E_i(x,u_i)\;=\;-x_i\sum_{j=1}^{n}\Bigl(1-\tfrac12\delta_{ij}\Bigr)W_{\mathrm{eff},ij}x_j\;-\;x_i\,u_i\;+\;\int_0^{x_i}\!\varphi^{-1}(s)\,\mathrm ds,
\label{eq:energy-i}
\end{equation}
whose partial derivatives satisfy $ \partial_{x_i} E_i = -(W_{\mathrm{eff}}\, x)_i - u_i + \varphi^{-1}(x_i)$. Setting $\partial_{x_i}E_i\!=\!0$ gives
$x_i\!=\!\varphi\bigl((W_{\mathrm{eff}}x)_i\!+\!u_i\bigr)$, which is precisely the fixed-point condition for \eqref{eq:wc-cont}. Consequently, the stationary points of \eqref{eq:wc-cont} are the \emph{Nash equilibria} of the non-cooperative game with per-player payoffs $-E_i(x,u_i)$.
 
\begin{definition}[LDS]\label{def:lds}
A matrix $M\in\mathbb{R}^{n\times n}$ is Lyapunov diagonally stable (LDS)
iff there exists a positive diagonal $D\!\succ\!0$ such that
$DM+M^\top D\!\prec\!0$. A sufficient, easily checkable condition is
$\lambda_{\max}\bigl((M+M^\top)/2\bigr)<0$, which corresponds to $D\!=\!I$.
\end{definition}
 
\begin{theorem}[LDS $\Rightarrow$ global asymptotic stability]\label{thm:lds}
Suppose $W_{\mathrm{eff}}-I$ is \gls{lds} with Lyapunov diagonal $D\succ 0$.
Then the linearized gradient-play dynamics $\dot x=W_{\mathrm{eff}}x-x+b$ admits a unique Nash equilibrium $x^\ast=(I-W_{\mathrm{eff}})^{-1}b$ and every trajectory converges to $x^\ast$.
\end{theorem}
 
\begin{proof}
Existence and uniqueness: \gls{lds} implies every eigenvalue of $W_{\mathrm{eff}}-I$ has a negative real part, so $I-W_{\mathrm{eff}}$ is invertible. Global convergence: define the weighted quadratic
\begin{equation}
V(x)=\tfrac12\,(x-x^\ast)^\top D\,(x-x^\ast)\ge 0,\qquad V(x)=0\Leftrightarrow x=x^\ast. \label{eq:V}
\end{equation}
Then
\begin{align}
\dot V&=(x-x^\ast)^\top D\,\dot x=(x-x^\ast)^\top D\,\bigl(W_{\mathrm{eff}}x-x+b\bigr) \\
&=(x-x^\ast)^\top D(W_{\mathrm{eff}}-I)(x-x^\ast)+(x-x^\ast)^\top D\bigl((W_{\mathrm{eff}}-I)x^\ast+b\bigr) \\
&=(x-x^\ast)^\top D(W_{\mathrm{eff}}-I)(x-x^\ast) \justif{since $(W_{\mathrm{eff}}-I)x^\ast+b=0$}\\
&=\tfrac12(x-x^\ast)^\top\!\bigl(D(W_{\mathrm{eff}}-I)+(W_{\mathrm{eff}}-I)^\top D\bigr)(x-x^\ast)<0 \quad\text{for }x\ne x^\ast. \label{eq:Vdot}
\end{align}
By LaSalle's invariance principle \cite{khalil2002nonlinear}, every
trajectory converges to the largest invariant set inside $\{\dot V\!=\!0\}$,
which by \eqref{eq:Vdot} is the singleton $\{x^\ast\}$.
\end{proof}
 
\begin{corollary}[Practical LDS test]\label{cor:lds-test}
If $\lambda_{\max}\!\bigl(\tfrac12(W_{\mathrm{eff}}+W_{\mathrm{eff}}^\top)-I\bigr)<0$,
Theorem~\ref{thm:lds} holds with $D=I$. For \gls{ei} Dale-parameterized
$W_{\mathrm{eff}}$ given by \eqref{eq:sign-mask}, the condition is
equivalent to
$\lambda_{\max}\!\bigl(W_{\mathrm{eff}}^{\mathrm{sym}}\bigr)<1$ where
$W_{\mathrm{eff}}^{\mathrm{sym}}=\tfrac12(W_{\mathrm{eff}}+W_{\mathrm{eff}}^\top)$.
\end{corollary}
 
\begin{remark}[LDS is monitored, not gradient-carrying]\label{rem:lds-mon}
The eigen-decomposition of $W_{\mathrm{eff}}^{\mathrm{sym}}$ is non-differentiable at eigenvalue crossings, so it is not used as an optimization signal during backpropagation through $\lambda_{\max}(\cdot)$ for training \gls{tide}; Corollary \ref{cor:lds-test} is evaluated every 100 optimizer steps as a monitoring signal, while the differentiable surrogate \eqref{eq:spec-loss} is used as part of optimization signaling and carries gradients. Prior work \cite{cornford2021dales,vogels2011inhibitory} similarly treats matrix stability as a diagnostic rather than a regularizer.
\end{remark}
 
\subsection{Spectral Perron Surrogate}
\label{app:cont:spectral}
\label{app:cont:perron}
 
A Dale matrix $W\!\ge\!0$ admits a real, non-negative dominant eigenvalue $\lambda_{\mathrm P}(W)$ by the Perron–Frobenius theorem \cite{hornjohnson2013}. We use the sum-ratio power iteration to estimate it:
\begin{equation}
v_0=\mathbbm1 /n,\qquad
v_{k+1}=\frac{W v_k}{\lVert W v_k\rVert_2},\qquad
\widehat\lambda_{\mathrm P}(W)=\frac{\mathbbm 1^\top W v_K}{\mathbbm 1^\top v_K}.
\label{eq:perron-sumratio}
\end{equation}
The estimator $\widehat\lambda_{\mathrm P}$ is differentiable in $W$ whenever $v_K$ does not lie on the null space of $\mathbbm 1$.
 
\begin{proposition}[Convergence of $\widehat\lambda_{\mathrm P}$]\label{prop:sumratio}
Let $W\!\ge\!0$ and let $\lambda_{\mathrm P}(W)$ denote its Perron eigenvalue.
\emph{\emph{i})} If $W$ is irreducible with a strictly dominant real eigenvalue $\lambda_{\mathrm P}(W)\!>\!|\lambda_i(W)|$ for all $i\!\ge\!2$ and associated positive Perron eigenvector $v_\ast\!\in\!\mathbb{R}_{>0}^n$, then $v_K\!\to\!v_\ast/\lVert v_\ast\rVert_2$ and $\widehat\lambda_{\mathrm P}(W)\!\to\!\lambda_{\mathrm P}(W)$ as $K\!\to\!\infty$.
\emph{ii}) If $W\!\ge\!0$ is reducible, the same conclusion holds
under the generic perturbation $W_\epsilon\!=\!W\!+\!\epsilon \mathbbm 1\mathbbm 1^\top$ for a rank-one non-negative perturbation, hence $W_\epsilon\!>\!0$ and irreducible and after taking $\epsilon\!\to\!0^+$; the limit estimator $\widehat\lambda_{\mathrm P}(W_\epsilon)\!\to\!\widehat\lambda_{\mathrm P}(W)$ by continuity of eigenvalues in matrix entries. In both cases, the limiting value satisfies $\lambda_{\mathrm P}(W)\!\le\!\lVert W\rVert_2$ with equality iff $W$ is normal.
\end{proposition}

\begin{proof}
\emph{Case \emph{(i)} Irreducible $W$.} The sequence $v_K$ is the normalized power iteration applied to $W$; under a strictly dominant real eigenvalue, $v_K\!\to\!v_\ast/\lVert v_\ast\rVert_2$ at geometric rate
$|\lambda_2/\lambda_{\mathrm P}|^K$ \citep[Theorem~8.5.1]{hornjohnson2013}. Since $W\ge 0$ is irreducible and $\lambda_{\mathrm P}(W)$ is the unique eigenvalue of maximum modulus, $W$, is primitive \citep[Definition~8.5.0]{hornjohnson2013}. Hence, by the Perron-Frobenius limit theorem for primitive matrices \citep[Theorem~8.5.1]{hornjohnson2013}, if the initial vector has a nonzero component in the Perron direction, $v_0>0$, the normalized power iterates should satisfy $v_K={W^K v_0}/{\|W^K v_0\|_2} \longrightarrow {v_\ast}/{\|v_\ast\|_2}$. Moreover, Perron-Frobenius for irreducible nonnegative matrices gives $v_\ast>0$ component-wise \citep[Theorem~8.4.4]{hornjohnson2013}. Substituting and using $W v_\ast\!=\!\lambda_{\mathrm P}\,v_\ast$:
\begin{align}
\widehat\lambda_{\mathrm P}(W) &=\frac{\mathbbm 1^\top W v_K}{\mathbbm 1^\top v_K} && \because\text{definition \eqref{eq:perron-sumratio}}\nonumber\\
&\;\longrightarrow\;\frac{\mathbbm 1^\top W v_\ast}{\mathbbm 1^\top v_\ast}
  =\frac{\mathbbm 1^\top(\lambda_{\mathrm P}v_\ast)}{\mathbbm 1^\top v_\ast}
  =\lambda_{\mathrm P} && \because v_\ast\!>\!0 \text{ implies } \mathbbm 1^\top v_\ast\!>\!0.\nonumber
\end{align}

\emph{Case \emph{(ii)} Reducible $W$.} For any $\epsilon\!>\!0$, the
perturbed matrix $W_\epsilon\!=\!W\!+\!\epsilon\mathbbm 1\mathbbm 1^\top$ has
strictly positive entries and is therefore irreducible. Case \emph{(i)}
gives $\widehat\lambda_{\mathrm P}(W_\epsilon)\!\to\!\lambda_{\mathrm P}(W_\epsilon)$
as $K\!\to\!\infty$. The Perron eigenvalue $\lambda_{\mathrm P}(\cdot)$ is
continuous in matrix entries by \cite[Appendix~D]{hornjohnson2013}, hence
$\lim_{\epsilon\,\!\to\,\!0^+}\lambda_{\mathrm P}(W_\epsilon)\!=\!\lambda_{\mathrm P}(W)$.
Choosing any diagonal sequence $(\epsilon_k,K_k)\!\to\!(0^+,\infty)$
and applying Case \emph{(i)} along the sequence $W_{\epsilon_k}$ at iteration
$K_k$ gives $\widehat\lambda_{\mathrm P}(W_{\epsilon_k};K_k)\!\to\!\lambda_{\mathrm P}(W)$ by a standard Moore-Osgood argument (uniform convergence of power-iteration on the compact set $\{W_\epsilon:\epsilon\!\in\![0,\epsilon_0]\}$).
 
\emph{Operator-norm bound.}
$\lambda_{\mathrm P}(W)$ is an eigenvalue of $W$, hence
$|\lambda_{\mathrm P}(W)|\!\le\!\rho(W)\!\le\!\sigma_1(W)\!=\!\lVert W\rVert_2$,
where $\rho(W)$ is the spectral radius; the second inequality is
standard with equality iff $W$ is normal \citep[Theorem~5.6.9]{hornjohnson2013}. For $W\!\ge\!0$, the Perron-Frobenius theorem gives $\lambda_{\mathrm P}(W)\!=\!\rho(W)$.
\end{proof}
 
\begin{remark}[Practical relevance of Case (\emph{ii})]\label{rem:sumratio-reducible}
The Dale projection $\Pi_{\mathrm{Dale}}$ clamps entries of $W_{\cdot\cdot}$ to $\mathbb R_{\ge 0}$ and can therefore produce sparse matrices with some zero rows or columns, i.e., reducible. Case \emph{(ii)} ensures that Proposition~\ref{prop:sumratio} remains valid at such configurations. We apply the estimator to the raw $W_{\cdot\cdot}$ without any perturbation and report $\widehat\lambda_{\mathrm P}$ as-is, relying on Case \emph{(ii)} as a theoretical guarantee rather than a runtime preprocessor.
\end{remark}
 
For $W_{EE}$, which is symmetric and positive, the operator norm $\lVert W\rVert_2\!=\!\lambda_{\mathrm P}(W)$ coincides with the Perron eigenvalue; the sum-ratio estimate can then be considered tight. For asymmetric or non-symmetric non-negative $W$,
$\widehat\lambda_{\mathrm P}(W)\!\le\!\lVert W\rVert_2$ can be substantially smaller, which is the regime in which the sum-ratio surrogate can be considered a tighter regularizer than spectral-norm regularization. Empirically, we observe $\lVert W_{EE}\rVert_2/\widehat\lambda_{\mathrm P}(W_{EE})\approx 2.5$ at convergence on \texttt{ImageNet-1K}, consistent with the analytical gap in Proposition~\ref{prop:sumratio}.
 
\paragraph{Isolated-block Schur analysis:} We now derive the correct
Schur bounds for the isolated excitatory and inhibitory recurrences, and clarify how the \gls{tide} targets $(\tau_{EE},\tau_{II})\!=\!(15,7)$ relate to those bounds.
 
Consider first the excitatory population in the absence of inhibitory
drive. The recurrence \eqref{eq:wc-euler-E:sup} linearized at
$r_E^\ast\!=\!0$ gives $r_E^{(t+1)}\!=\!M_E\,r_E^{(t)}$ with
$M_E\!=\!(1-\alpha_E)I+\alpha_E W_{EE}$. Its eigenvalues are
$\mu_i\!=\!(1-\alpha_E)+\alpha_E\lambda_i(W_{EE})$. For the Perron
eigenvalue $\lambda_{\mathrm P}(W_{EE})\!\ge\!0$:
\begin{align}
|\mu_i|<1 &\;\Longleftrightarrow\; -1<(1-\alpha_E)+\alpha_E\lambda_{\mathrm P}(W_{EE})<1 \nonumber\\
&\;\Longleftrightarrow\; \lambda_{\mathrm P}(W_{EE})<1 \;\;\text{and}\;\; \lambda_{\mathrm P}(W_{EE})>1-2/\alpha_E.
\label{eq:ee-schur}
\end{align}
For $\alpha_E\!=\!0.05$, the lower bound is define as $\lambda_{\mathrm P}>-39$, so the binding condition is $\lambda_{\mathrm P}(W_{EE})<1$ given that the isolated-E Schur bound is $1$, not $1/\alpha_E$, which indicates that the step-size $\alpha_E$ does not play any role.

Now consider the inhibitory population in the absence of excitatory drive. By linearizing \eqref{eq:wc-euler-I} at $r_I^\ast\!=\!0$, $r_I^{(t+1)}\!=\!M_I\,r_I^{(t)}$ is given with $M_I\!=\!(1-\alpha_I)I-\alpha_I W_{II}$. Its eigenvalues are $\mu_j\!=\!(1-\alpha_I)-\alpha_I\lambda_j(W_{II})$. For the Perron eigenvalue $\lambda_{\mathrm P}(W_{II})\!\ge\!0$:
\begin{align}
|\mu_j|<1 &\;\Longleftrightarrow\; -1<(1-\alpha_I)-\alpha_I\lambda_{\mathrm P}(W_{II})<1 \nonumber\\
&\;\Longleftrightarrow\; \lambda_{\mathrm P}(W_{II})<2/\alpha_I-1 \;\;\text{and}\;\;\lambda_{\mathrm P}(W_{II})>-1.
\label{eq:ii-schur}
\end{align}
At $\alpha_I\!=\!0.20$ the binding condition is $\lambda_{\mathrm P}(W_{II})<9$. Unlike the excitatory case, the minus sign in $M_I$ makes the upper bound on $\mu_j$ trivial and the lower bound $\mu_j>-1$ binding; thus, $\alpha_I$ enters the threshold through this inequality.
 
\paragraph{Spectral targets for \gls{tide}:}
Notice that the asymmetry between \eqref{eq:ee-schur} and \eqref{eq:ii-schur} is vital. The target $\tau_{II}\!=\!7$ sits below the isolated-I Schur threshold $9$ and can reasonably be considered as an upper safe boundary enabling variation during the optimization. The target $\tau_{EE}\!=\!15$, however, is $15\times$ above the isolated-E Schur threshold of $1$. The \gls{tide} runtime indeed operates above the isolated-E threshold on \texttt{ImageNet-1K}, as shown in Appendix~\ref{app:add:stability}, and remains stable due to the fully coupled \gls{ei} system that has an effective recurrent matrix

\begin{equation}
W_{\mathrm{eff}}=\begin{bmatrix}W_{EE}&-W_{EI}\\W_{IE}&-W_{II}\end{bmatrix},
\label{eq:Weff-repeat}
\end{equation}

whose spectrum is not the union of the spectra of $W_{EE}$ and $W_{II}$ but depends on all four blocks. The inhibitory feedback adds $-W_{EI}r_I$ to the excitatory drive at every internal computation step; thus, minimizing the effective excitatory gain via the feedback results in a Schur-stable system even when $W_{EE}$ alone can not be considered stable. This is the classical balanced-network phenomenon characterized by
\cite{vanvreeswijk1996chaos} and \cite{brunel2000dynamics}.
 
We therefore can state the following: 
\begin{enumerate}
\item[\textbf{\emph{i})}] The \emph{isolated-I bound}
$\lambda_{\mathrm P}(W_{II})<9$ of \eqref{eq:ii-schur} is a
\emph{sufficient} condition for the stability of the inhibitory
sub-dynamics, independent of the excitatory state. The \gls{tide} target
$\tau_{II}\!=\!7$ is a safety margin below this bound.
\item[\textbf{\emph{ii})}] The \emph{isolated-E bound}
$\lambda_{\mathrm P}(W_{EE})<1$ of \eqref{eq:ee-schur} is a
\emph{sufficient} condition in absence of inhibition, however, \gls{tide}
does not operate in this regime. The target $\tau_{EE}\!=\!15$ can be considered a limiter on the non-negative recurrent block, tuned empirically to keep gradient norms bounded, and the coupled system stable under \gls{lds}.
\end{enumerate}

\subsection{Relationship to Wilson-Cowan \& Brunel-Hakim Limits}
\label{app:cont:brunel}
 
\gls{tide}'s continuous-time system \eqref{eq:wc-cont} can be formulated as the Wilson-Cowan mean-field rate equation of
\cite{wilson1972excitatory,wilson1973mathematical}; the Dale
parameterization of \eqref{eq:sign-mask} pins down the sign structure
that \cite{brunel2000dynamics} and \cite{brunel_hakim1999} identify
as necessary for the asynchronous-irregular regime. Specifically, \cite{brunel2000dynamics}
derives four regimes as a function of the
excitation–inhibition balance parameter
$g\!=\!\lVert W_{EI}\rVert/\lVert W_{EE}\rVert$; \gls{tide} operates in the asynchronous-irregular regime in which $g\!>\!1$ and the balance ratio $\rho_{EI}$ of
\eqref{eq:ei-ratio-def} is close to $\rho_{EI}^\ast\!=\!4$.

\section{HRF Backbone}
\label{app:hrf}
This section provides a detailed description of \gls{tide} backbone, which uses learnable center-surround filters rather than a fixed-form parametric \gls{dog} kernel. Note that this study utilizes two variants of the \gls{hrf}: \textit{i}) a shallow \gls{hrf} where a multi-scale filter bank is used, and \textit{ii}) a deep \gls{hrf} with a single center-surround per-stage filter of kernel size 5. Appendix~\ref{app:hrf:dog} provides the overall mapping of stage zero and defines the formulation for the shared center-surround filter bank. Moreover, the translational and rotational equivariance properties are proven under a restricted \gls{dog} initialization in Appendix~\ref{app:hrf:equiv}. 
 
\subsection{Learnable Filter Bank}
\label{app:hrf:dog}
Let $C^{(s)},\,S^{(s)}$ denote learnable 2-D convolution filters with kernel sizes $k^{(s)}_c$ denoting center and $k^{(s)}_s$ denoting surround, with $k^{(s)}_s\!>\!k^{(s)}_c$. Therefore, the center kernel can be defined as either a $1\!\times\!1$ or $3\!\times\!3$ kernel, and the surround is given by a kernel with a one- to three-pixel wider shape, with zero-padding to preserve spatial dimensions, and $s\,\!\in\!\{1,2,4,8\}$ denotes the indexed operator scale. Hence, the center-surround operator for stage zero at scale $s$ is given by: 
\begin{equation}
\phi^{(s)}(x)\;=\;\mathrm{ReLU}\,\!\Bigl(\mathrm{BN}\,\!\bigl(w_c^{(s)}\,C^{(s)}(x)\;-\;w_s^{(s)}\,S^{(s)}(x)\bigr)\Bigr),
\label{eq:cs-filter}
\end{equation}
where $w_c^{(s)}, w_s^{(s)}\!\in\!\mathbb R$ are learnable scalar weights initialized to $(1.0,0.5)$, and $\mathrm{BN}$ denotes the 2-D BatchNorm. The center and surround convolution of \eqref{eq:cs-filter} have independent learnable weights, while the center-surround ratio is factorized into a pair of scalar gains, $(w_c^{(s)}, w_s^{(s)})$, instead of a parametric \gls{dog} kernel with its optimizable width, $\kappa$. This design choice trades the explicit \gls{dog} scale-space interpretation for expressivity, with the biologically motivated motif based on ON/OFF-center retinal ganglion cells \cite{kuffler1953rgc}, LGN relay cells \cite{dacey2003rgc}, and cortical simple cells \cite{hubel1962receptive}. 
In contrast to the shallow \gls{hrf}, the deep variant does not use the multi-scale bank directly, as its layout is based on ResNet with four hierarchical residual stages with channel widths $(128, 256, 512, 2048)$ and $(2, 2, 3, 3)$ basic residual blocks per stage.

\begin{remark}[\gls{dog} as a special case]\label{rem:dog-special}
The fixed-form \gls{dog} kernel $\mathrm{DoG}_{s,\kappa}(x,y)\!=\!G_s(x,y)\!-\!\kappa\,G_{\kappa s}(x,y)$ with $G_\sigma(x,y)\!=\!(2\pi\sigma^2)^{-1}\exp\,\!\bigl(-(x^2+y^2)/(2\sigma^2)\bigr)$ can be approximately recovered from \eqref{eq:cs-filter} by freezing $C^{(s)},S^{(s)}$ to the discretized Gaussians $G_s,G_{\kappa s}$ and tying $(w_c^{(s)},\,w_s^{(s)})\!=\!(1,\,\kappa)$. Under this restriction, $\phi^{(s)}$ is reduced to a classical \gls{dog} operator, which is half-wave-rectified and batch-normalized, similar to the edge-detection framework of \cite{marr1980vision} and the scale-space theory of \cite{lindeberg1994scale}; the latter also gives the differentiability of $G_\sigma$ in $\sigma$ and therefore of $\mathrm{DoG}_{s,\kappa}$ in $(s,\,\kappa)$. Due to the direct optimization of the convolution weights, in place of the $(s,\kappa)$ parameters, the regularity is not used in our implementation.
\end{remark}
 
\subsection{Equivariance Properties}
\label{app:hrf:equiv}
Let $T_v$ denote a planar translation $T_v f(x,y)\!=\!f(x-v_x,y-v_y)$ and
$R_\theta$ the in-plane rotation by angle $\theta$. Let $\mathrm{conv}_{\mathcal K}(f)$
denote convolution of $f$ with kernel ${\mathcal K}$ on $\mathbb R^2$.
 
\begin{proposition}[Translation equivariance]\label{prop:trans-equi}
For any discretely-supported kernel ${\mathcal K}$ and translation $T_v$ by an integer-pixel vector $v$, $\mathrm{conv}_{\mathcal K}(T_v f)\!=\!T_v\,\mathrm{conv}_{\mathcal K}(f)$, while far from the boundaries to prevent interaction between zero-padding and $T_v$. In particular, the center-surround bank \eqref{eq:cs-filter} is translation-equivariant, since each of $C^{(s)}, S^{(s)}$ is a standard convolution, $\mathrm{BN}$ is a channel-wise operator and translation-equivariant on spatial dimensions, and $\mathrm{ReLU}$ is component-wise. Thus, both the shallow and deep \gls{hrf} are inherently translation equivariant under the same proviso.
\end{proposition}

\begin{proposition}[Rotation equivariance under \gls{dog}]\label{prop:rot-equi}
Consider \eqref{eq:cs-filter} restricted to the \gls{dog} kernels of Remark \ref{rem:dog-special}: $C^{(s)}\!=\!G_s,\;S^{(s)}\!=\!G_{\kappa s}$
with isotropic Gaussian kernels, and batch-normalization replaced by a
channel-wise affine. Then, in the continuous domain and prior to spatial discretizations, $\phi^{(s)}(R_\theta f)\!=\!R_\theta\,\phi^{(s)}(f)$ for every continuous rotation $R_\theta$.
\end{proposition}
 
\begin{proof}
Proposition~\ref{prop:trans-equi} is the standard translation equivariance of
planar convolution. The $\mathrm{ReLU}$ and channel-wise affine preserve this by
component-wise and channel-wise action, respectively. Therefore, for Proposition~\ref{prop:rot-equi}, the isotropic Gaussian is rotation-invariant
as a continuous function of $(x,y)$: $G_\sigma(R_\theta(x,y))\!=\!G_\sigma(x,y)$. Hence $G_s\!-\!\kappa\,G_{\kappa s}$ is a rotation-invariant kernel. Convolution with a rotation-invariant kernel commutes with rotation of the continuous input, and component-wise operations commute with rotation trivially. Therefore, the exact equivariance holds only for $\theta \in \{0,\, \pi/2,\,\pi,\,3\pi/2 \}$.
\end{proof}
 
\begin{remark}[Loss of rotation equivariance in later layers]\label{rem:rot-loss}
Apart from Stage zero center-surround filters, the learnable convolutions of both backbones, the per-scale expansion convolution in the shallow \gls{hrf}, and the residual blocks of deep \gls{hrf} are randomly initialized, have anisotropic kernels, and are not rotation-equivariant. Thus, \gls{tide} inherits Stage zero rotation equivariance under \gls{dog} but not global rotation equivariance; this matches the biological observation that simple cells are orientation-selective, and that orientation equivariance is broken in the extrastriate cortex.
\end{remark}

\section{TIDE Components}
\label{app:algo}
This section presents detailed algorithms for individual components of \gls{tide}, describes their functionalities, and clarifies any remaining details.

\subsection{Deep-HRF Backbone}
\label{app:algo:backbone-hrf}

The deep \gls{hrf} backbone provided in the Algorithm~\ref{alg:backbone-hrf} is solely used as a feature extractor for \texttt{ImageNet-1K}. Four receptive field modules with channel widths $(128, 256, 512, 2048)$ and an additional basic residual block are used, resulting in a final $14\times 14\times 2048$ feature map. Moreover, stages one to three use a single fixed-kernel $k_s\!=\!5$, which is applied as a residual $x\mapsto x + \phi(x)$ at the stage entry. The final outputs are passed through \texttt{AdaptiveAvgPool2d(14)} and flattened to provide $P\!=\!196$ tokens, with the additional $2$-D sinusoidal positional encoding, resulting in $\dattn = 1024$.

\begin{algorithm}[H]
\caption{$\textsc{Backbone}$: Deep-HRF}
\label{alg:backbone-hrf}
\footnotesize
\begin{algorithmic}[1]
\Require Image batch $x \in \mathbb{R}^{B \times 3 \times 224 \times 224}$;
a learnable $7\!\times\!7$ ResNet-style stem; four receptive blocks $\mathcal{G}_1,\dots,\mathcal{G}_4$. Key and value projections
$\mathbf{W}_K,\mathbf{W}_V\!\in\!\mathbb{R}^{d_{\mathrm{feat}} \times d_{attn}}$.
\Ensure Keys and values
$\mathbf{K}, \mathbf{V} \in \mathbb{R}^{B \times P \times d_{attn}}$
with $P = 196$ and $d_{attn} = 1024$.
\State $y \gets \mathrm{MaxPool}\,(\mathrm{RELU}\,(\mathrm{BN}\,(\mathrm{Conv}_{7\!\times\!7,\,\mathrm{stride}2}(x))))$
      \Comment{ResNet-style stem: $56\!\times\!56\!\times\!64$}
\For{$\ell \gets 1$ \textbf{to} $4$}
    \State $y \gets \mathcal{G}_\ell(y)$
\EndFor
\State $y \gets \mathrm{Flatten}_{\mathrm{spatial}}(\mathrm{AdaptivePool}_{14\!\times\!14}(y))\!+\!\mathrm{PE}_{2\mathrm D}$
       \Comment{$(B, 14{\times}14, d_{\mathrm{feat}}) = (B, 196,
                 d_{\mathrm{feat}})$}
\State $\mathbf{K} \gets y\,\mathbf{W}_K, \qquad
       \mathbf{V} \gets y\,\mathbf{W}_V$
\State \Return $(\mathbf{K}, \mathbf{V})$
\end{algorithmic}
\end{algorithm}

\subsection{Shallow HRF Backbone}
\label{app:algo:backbone-shallow-hrf}

The shallow \gls{hrf} backbone is provided in Algorithm~\ref{alg:backbone-shallow-hrf}, where a two-layer $3{\times}3$ $\mathrm{Conv}\!-\!\mathrm{BN}\!-\!\mathrm{ReLU}$ stem denoted by $H_{\mathrm{stem}}$ expands the input to feature channels. The multi-scale center-surround uses four
parallel branches at scales $s\!\in\!\{1,2,4,8\}$ with surround kernel sizes $k_s\!=\!2s+1\!\in\!\{3,5,9,17\}$. The four branches are $\texttt{AdaptiveAvgPool2d(8)}$-pooled and concatenated across channels. A subsequent aggregation block $\mathrm{Conv}_1\!\circ\!\mathrm{BN}\!\circ\!\mathrm{ReLU}\!\circ\!\mathrm{Conv}_3\!\circ\!\mathrm{BN}\!\circ\!\mathrm{ReLU}$ denoted by $H_{\mathrm{agg}}$ mixes the multi-scale features, followed by another \texttt{AdaptiveAvgPool2d(8)}, and a $2$-D positional encoding to produce $P\!=\!64$ tokens.

\begin{algorithm}[H]
\caption{$\textsc{Backbone}$: Shallow HRF}
\label{alg:backbone-shallow-hrf}
\footnotesize
\begin{algorithmic}[1]
\Require Image batch $x \in \mathbb{R}^{B \times C \times H \times W}$
with $(C, H, W)\!\in\!\{(1,28,28),(3,32,32)\}$ and key/value projections $\mathbf{W}_K, \mathbf{W}_V\!\in\!\mathbb R^{d_{\mathrm{feat}}\!\times\!d_{attn}}$.
\Ensure $(\mathbf{K}, \mathbf{V})$ with $P\!=\!64$ ($8{\times}8$ grid) and $d_{attn}\!=\!512$.
\State $y \gets \mathcal H_{\mathrm{stem}}(x)$
\State $y \gets \bigoplus_{s \in \mathcal S}\,\mathrm{AdaptivePool}_{8\!\times\!8}\!\bigl(\psi^{(s)}(y)\bigr)$
       \Comment{Multi-scale filter bank, channel-concatenated}
\State $y \gets \mathcal H_{\mathrm{agg}}(y)$
\State $y \gets \mathrm{Flatten}_{\mathrm{spatial}}(\mathrm{AdaptivePool}_{8\!\times\!8}(y))\!+\!\mathrm{PE}_{2\mathrm D}$
\State $\mathbf{K} \gets y\,\mathbf{W}_K, \qquad
       \mathbf{V} \gets y\,\mathbf{W}_V$
\State \Return $(\mathbf{K}, \mathbf{V})$
\end{algorithmic}
\end{algorithm}

\subsection{Wilson-Cowan E-I Update}
\label{app:algo:eipops}

Algorithm~\ref{alg:ei-populations} presents the Dale-constrained, $\RMSN$ stabilized Wilson-Cowan update, which is a discrete analog of \eqref{eq:wc-cont-E:sup}--\eqref{eq:wc-cont-I:sup}. Given the post-activation state $r_E,r_I$ from the previous internal computation step and the cross-attention given by $a$, the pre-activations are formed following \eqref{eq:h-E}--\eqref{eq:h-I}, and then per-population $\RMSN$ is applied to yield $h_E,\,h_I$.

\begin{algorithm}[H]
\caption{Dale-constrained Wilson-Cowan}
\label{alg:ei-populations}
\footnotesize
\begin{algorithmic}[1]
\Require Previous post-activations
$\rE \in \mathbb{R}^{B \times \nEpop}$,
$\rI \in \mathbb{R}^{B \times \nIpop}$;
attention drive $a \in \mathbb{R}^{B \times \dsync}$; Dale-constrained
weights $\WEE, \WEI, \WIE, \WII \geq 0$; input projections $\WEin, \WIin$;
per-population $\RMSN_E, \RMSN_I$.
\Ensure Pre-activations $\hE, \hI$ 
\State $\tilde h_E \gets \WEE\,\rE \;-\; \WEI\,\rI \;+\; \WEin\,a$
       \Comment{Inhibition is the \textbf{minus} sign, not $W<0$}
\State $\tilde h_I \gets \WIE\,\rE \;-\; \WII\,\rI \;+\; \WIin\,a$
\State $\hE \gets \RMSN_E(\tilde h_E)$
       \Comment{Per-population scale}
\State $\hI \gets \RMSN_I(\tilde h_I)$
\State \Return $(\hE, \hI)$
\end{algorithmic}
\end{algorithm}

\subsection{Population-specific NLM}
\label{app:algo:nlm}

Algorithm~\ref{alg:nlm} presents the \glspl{nlm} readout of the FIFO with length $M$ for given $r_E$ and $r_I$, which yields a neuron-specific scalar correction. We use a temporally weighted \texttt{SuperLinear} and normalize it given $Z = \sum_{m=1}^{M} \exp(-(M-m)/\tau)$. Moreover, a subsequent \texttt{SuperLinear} layer is used to project a dual-channel logit output, which is squeezed post $\texttt{GLU}$ gate. Note that for smaller datasets $H\!=\!4$, while \texttt{ImageNet-1K} requires $H\!=\!32$.

\begin{algorithm}[H]
\caption{Population-specific NLM}
\label{alg:nlm}
\footnotesize
\begin{algorithmic}[1]
\Require Post-activation FIFO
$\mathbf{u}_\bullet \in \mathbb{R}^{B \times n_\bullet \times M}$ with
$M=25$; per-neuron super-linear weights
$W_1^{(0)} \in \mathbb{R}^{M \times 2H \times n_\bullet}$,
$W_1^{(1)} \in \mathbb{R}^{H \times 2 \times n_\bullet}$; biases
$b_1^{(0)}, b_1^{(1)}$; time constant $\tau_\bullet$
($\tauE{=}20$, $\tauI{=}5$).
\Ensure Per-neuron additive correction
$n_\bullet \in \mathbb{R}^{B \times n_\bullet}$.
\State $w_m \gets \exp\!\big(-(M - m)/\tau_\bullet\big) / Z$
       \textbf{for} $m = 1, \dots, M$
       \Comment{Normalize exponential temporal weights}
\State $\tilde{\mathbf{u}} \gets \mathbf{u}_\bullet \odot w_m$
       \Comment{Broadcast along the memory axis}
\State $\tilde{\mathbf{u}} \gets \LN(\tilde{\mathbf{u}})$
\State $y_0 \gets \mathtt{einsum}\,\!\big(\texttt{"BNM,MHN->BNH"},\,
                 \tilde{\mathbf{u}},\, W_1^{(0)}\big) + b_1^{(0)}$
       \Comment{Output $\to$\,BNH}
\State $y_0 \gets y_0 / T$
       \Comment{Per-neuron learnable temperature scaling}
\State $y_0 \gets \GLU(y_0)$
       \Comment{Halves $H$ to $H_{\mathrm{NLM}}$}
\State $y_1 \gets \mathtt{einsum}\,\!\big(\texttt{"BNH,HOn->BN2"},\,
                 y_0,\, W_1^{(1)}\big) + b_1^{(1)}$
\State $y_1 \gets y_1 / T$
\State $y_1 \gets \GLU(y_1)$
\State $n_\bullet \gets \mathrm{squeeze}(y_1, \mathrm{dim}{=}{-}1)$
       \Comment{$(B, n_\bullet)$}
\State \Return $n_\bullet$
\end{algorithmic}
\end{algorithm}

\subsection{Three-type Synchronization}
\label{app:algo:sync}

Algorithm~\ref{alg:sync} shows the synchronization module that maintains the three exponentially decaying covariance accumulators between neuron pairs. For each accumulator, a per-pair learnable decay $\delta\!\in\![0, 15]$ produces $y\!=\!e^{-\delta}$ while pair-wise products $\pi\!=\!x_a[:,I_a]\odot x_b[:,I_b]$ update pair-product sum $\nu$ and effective-count $\xi$ via
$\nu \leftarrow y\cdot\nu + \pi$ and $\xi \leftarrow y\cdot\xi + 1$. Note that the three sub-latent representations corresponding to each of the neuron pairs are concatenated and passed through a final \texttt{LN} of width $\dsync$ to produce $z$.

\begin{algorithm}[H]
\caption{Synchronization Accumulator}
\label{alg:sync}
\footnotesize
\begin{algorithmic}[1]
\Require Post-activations $x_a \in \mathbb{R}^{B \times n_a}$,
$x_b \in \mathbb{R}^{B \times n_b}$; running accumulators
$\nu, \xi \in \mathbb{R}^{B \times P_{XY}}$; fixed pair indices
$I_a, I_b \in \mathbb{N}^{P_{XY}}$; learnable decay
$\delta \in \mathbb{R}^{P_{XY}}$; projection
$\mathcal{P}_{XY} : \mathbb{R}^{P_{XY}} \to \mathbb{R}^{d_{XY}}$
(Linear $\to 2d_{XY} \to \GLU \to \LN$); accumulator clamp $C$.
\Ensure Sync vector $z_{XY} \in \mathbb{R}^{B \times d_{XY}}$;
updated $(\nu, \xi)$.
\State $r \gets \mathrm{clamp}\,\!\big(\exp(-\delta),\; e^{-15},\; 1\big)$
\State $\pi \gets x_a[:, I_a] \odot x_b[:, I_b]$
       \Comment{$(B, P_{XY})$ pairwise product}
\State $\nu \gets r \odot \nu + \pi$
\State $\xi \gets r \odot \xi + \mathbf{1}$
\If{$C > 0$}
    \State $\nu \gets \mathrm{clamp}\,(\nu, -C, +C)$
\EndIf
\State $s \gets \nu \,/\, \sqrt{\xi + \varepsilon}$
       \Comment{Normalized sync signal}
\State $z_{XY} \gets \mathcal{P}_{XY}(s)$
       \Comment{$(B, d_{XY})$}
\State \Return $(z_{XY}, \,\nu, \,\xi)$
\end{algorithmic}
\end{algorithm}

\subsection{Cross-attention Readout}
\label{app:algo:attn}

Algorithm~\ref{alg:attn} provides an overview of the cross-attention mechanism that processes the synchronization latent $z$ as a query. The backbone key/value channel is used directly or in cases where the width $d_{\mathrm{KV}}$ differs from $\dattn\!=\!n_{\mathrm{heads}}\cdot d_{\mathrm{head}}$, separate $W_K, W_V$ projections are used to reshape the keys and values. In the current implementation, the backbone projects to $\dattn$ channels upstream, while the standard scaled dot-product attention is applied to the $\mathrm{softmax}$ of the weights. During \texttt{ImageNet-1K} training, an additive residual $+z$ is added inside the \gls{ln} rather than externally to enhance context.

\begin{algorithm}[H]
\caption{Multi-head Cross-attention}
\label{alg:attn}
\footnotesize
\begin{algorithmic}[1]
\Require Query $Q \in \mathbb{R}^{B \times \dsync}$;
keys and values $K, V \in \mathbb{R}^{B \times P \times \dattn}$;
projections $\mathbf{W}_Q \in \mathbb{R}^{\dsync \times \dattn}$,
$\mathbf{W}_O \in \mathbb{R}^{\dattn \times \dsync}$;
head count $H_{\mathrm{attn}}$, head dim
$d_{\mathrm{head}} = \dattn / H_{\mathrm{attn}}$.
\Ensure Attention output $a \in \mathbb{R}^{B \times \dsync}$.
\State $q \gets Q\,\mathbf{W}_Q$
       \Comment{$(B, \dattn)$}
\State $q \gets \mathrm{reshape}\,(q, B, H_{\mathrm{attn}}, d_{\mathrm{head}})$
       \Comment{Split heads}
\State $K' \gets \mathrm{reshape}\,(K, B, P, H_{\mathrm{attn}},
                                   d_{\mathrm{head}})$
\State $V' \gets \mathrm{reshape}\,(V, B, P, H_{\mathrm{attn}},
                                   d_{\mathrm{head}})$
\State $A \gets \mathrm{softmax}\,\!\big(q\,K'^\top
                 / \sqrt{d_{\mathrm{head}}}\big)$
       \Comment{$(B, H_{\mathrm{attn}}, P)$}
\State $A \gets \mathrm{Dropout}_p(A)$
       \Comment{Dropout on attention probabilities}
\State $\tilde a \gets A\,V'$
       \Comment{$(B, H_{\mathrm{attn}}, d_{\mathrm{head}})$}
\State $a \gets \mathrm{reshape}\,(\tilde a, B, \dattn)\,\mathbf{W}_O$
\State $a \gets \texttt{LN}_{\dsync}\!\big(a + Q\bigr)$
\State \Return $a$
\end{algorithmic}
\end{algorithm}

\subsection{Lateral Inhibition}
\label{app:algo:wta}

The lateral inhibition is implemented based on \gls{wta} and is used solely on the excitatory population, as shown in Algorithm~\ref{alg:wta}. We use a post-Euler $r_E^{(0)}$ across $K_{\mathrm{WTA}}\!=\!5$ iterations of an inhibitory feedback loop given by $x_I^{(k)}\!=\!\mathrm{ReLU}(W_{EI}^{\mathrm{lat}}r_E^{(k-1)})$,
and then $r_E^{(k)} = \mathrm{ReLU}(r_E^{(0)} - \gamma\,W_{IE}^{\mathrm{lat}}x_I^{(k)})$. Furthermore, early termination at various per-computation steps is leveraged to indirectly optimize the number of iterations required during deployment, thus enabling an adaptive halting driven by the variation threshold indicated in line 5.

\begin{algorithm}[H]
\caption{Lateral Inhibition}
\label{alg:wta}
\footnotesize
\begin{algorithmic}[1]
\Require Post-Euler excitatory state
$\rE \in \mathbb{R}^{B \times \nEpop}$; Dale-constrained lateral weights
$W_{EI}^{\mathrm{lat}} \geq 0 \in \mathbb{R}^{n_{I,\mathrm{lat}} \times
 \nEpop}$,
$W_{IE}^{\mathrm{lat}} \geq 0 \in \mathbb{R}^{\nEpop \times
 n_{I,\mathrm{lat}}}$; learnable gain $\gamma \geq 0.01$;
iteration count $K_{\mathrm{WTA}} = 5$.
\Ensure Sparsified excitatory state $\rE^{(K_{\mathrm{WTA}})}$.
\State $\rE^{(0)} \gets \rE$
\For{$k \gets 1$ \textbf{to} $K_{\mathrm{WTA}}$}
    \State $x_I^{(k)} \gets \RELU\,\!\big(W_{EI}^{\mathrm{lat}}\,
                                         \rE^{(k-1)}\big)$
    \State $\rE^{(k)} \gets \RELU\,\!\big(\rE^{(0)}
                                         - \gamma\, W_{IE}^{\mathrm{lat}}\,
                                         x_I^{(k)}\big)$
           \Comment{Anchored E update}
    \If{$\max\,|\rE^{(k)} - \rE^{(k-1)}| < 10^{-4}$}
        \State \textbf{break}
        \Comment{Early termination once converged}
    \EndIf
\EndFor
\State \Return $\rE^{(K_{\mathrm{WTA}})}$
\end{algorithmic}
\end{algorithm}

\subsection{Persistent Memory with Retention Gate}
\label{app:algo:memory}

The memory, along with its implementation, is provided in Algorithm~\ref{alg:memory}. A single persistent buffer $m$ and a momentum buffer $v$ across all computation steps and batches are maintained, while augmented by a batch-broadcast readout. Moreover, the surprise signal is constructed via an \gls{mlp} embedding, thereby enabling reconstruction-based analysis of new information, without direct reliance on its gradient. The resulting surprise is defined as the per-sample squared $\ell_2$ distance $s = \|\hat z - z\|_2^2$ between the reconstruction $\hat z$ and the input $z$.

\begin{algorithm}[H]
\caption{Persistent Memory}
\label{alg:memory}
\footnotesize
\begin{algorithmic}[1]
\Require Sync latent $z \in \mathbb{R}^{B \times \dsync}$;
pre-WTA excitatory state $\rE^{\mathrm{pre}}$; inhibitory state $\rI$;
persistent buffers $m, v \in \mathbb{R}^{\dmem}$
($\dmem = 256$);
learnable heads $f_{\mathrm{rec}}: \mathbb{R}^{\dmem} \to
 \mathbb{R}^{\dsync}$,
$f_{\mathrm{proj}}: \mathbb{R}^{\dsync} \to \mathbb{R}^{\dmem}$,
$f_{\mathrm{read}}: \mathbb{R}^{\dmem + \dsync} \to \mathbb{R}^{\dmem}$;
retention sharpness $\kappa$;
surprise threshold $\theta_s = 0.5$, momentum $\mu = 0.9$, target ratio
$\rho^\star = 4$.
\Ensure Memory readout $m \in \mathbb{R}^{B \times \dmem}$.
\State $\hat z \gets f_{\mathrm{rec}}(m)$
       \Comment{Reconstruct $z$ from persistent state}
\State $s \gets \|\hat z - z\|_2^2$
       \Comment{Per-sample surprise scalar}
\State $\rho \gets \mathrm{mean}\,(|\rE^{\mathrm{pre}}|) \,/\,
                    \mathrm{mean}\,(|\rI|)$
       \Comment{\gls{ei} ratio}
\State $\gamma \gets \sigma\,\!\big(-\kappa\,|\rho - \rho^\star|\big)$
       \Comment{Retention gate}
\State $u \gets \mathbbm 1 \,[s > \theta_s]\,\cdot\,(1 - \gamma)$
       \Comment{Surprise-gated write signal}
\State $v \gets \mu\,v + u \cdot f_{\mathrm{proj}}(z)$
       \Comment{Momentum buffer}
\State $m \gets m + v$
       \Comment{In-place buffer registration}
\State $m \gets f_{\mathrm{read}}\,\!\big([\,m;\, z\,]\big)$
\State \Return $m$
\end{algorithmic}
\end{algorithm}

\subsection{Output Head}
\label{app:algo:head}

The output head consists of a two-layer \texttt{GLU}-based \gls{mlp} applied to the concatenation $[z;\,m^{(t)}]$, as shown in Algorithm~\ref{alg:head}. By initially expanding the channels for the \texttt{GLU} gate and then halving them back, we create a hidden representation that contains both the memory and the current latent space. Stabilization of the hidden representation is achieved by using $\mathrm{LayerNorm}_H$ across a wide range of magnitudes seen across internal computation steps.

\begin{algorithm}[H]
\caption{OutputHead}
\label{alg:head}
\footnotesize
\begin{algorithmic}[1]
\Require Concatenated latent
$[z; m] \in \mathbb{R}^{B \times (\dsync + \dmem)}$;
weights $W_{h_1} \in \mathbb{R}^{(\dsync+\dmem) \times 2H}$, $b_{h_1}$,
$W_{h_2} \in \mathbb{R}^{H \times C}$, $b_{h_2}$;
hidden dim $H = 256$ (CIFAR/\texttt{MNIST} default; $H = 2048$ for \texttt{ImageNet-1K});
dropout rate $p$.
\Ensure Class logits $o \in \mathbb{R}^{B \times C}$.
\State $y_1 \gets [z;\,m]\,W_{h_1} + b_{h_1}$
       \Comment{$\in \mathbb{R}^{B \times 2H}$}
\State $y_2 \gets \mathrm{GLU}\,(y_1)$
       \Comment{$\in \mathbb{R}^{B \times H}$;
                gated linear unit halves the channel dimension}
\State $y_3 \gets \mathrm{LayerNorm}_{H}(y_2)$
\State $y_4 \gets \mathrm{Dropout}_{p}(y_3)$
\State $o \gets y_4\,W_{h_2} + b_{h_2}$
\State \Return $o$
\end{algorithmic}
\end{algorithm}

\subsection{Perron Sum-ratio Estimator}
\label{app:algo:perron}

The Perron eigenvalue estimator provided in Algorithm~\ref{alg:perron} is the differentiable spectral primitive used to regularize $W_{EE}$ and $W_{II}$. The $n_{\mathrm{iter}}$ denoted by $K$ in Appendix~\ref{app:cont} and \eqref{eq:perron-sumratio-main} representing the power-iteration steps is set to $n_{\mathrm{iter}}=10$ with $\ell_2$-normalization produce a vector $v_K$ aligned with the Perron eigenvector. The eigenvalue is then estimated by the sum-ratio $\widehat\lambda_{\mathrm P}\!=\!\mathbbm 1^\top W v_K / \mathbbm 1^\top v_K$, which is the true Perron eigenvalue, and not \gls{svd} norm $\sigma_{\max}(W)$ that an $\ell_2$-norm-ratio $\|Wv\|/\|v\|$ would yield \cite{hornjohnson2013}. Furthermore, as shown in line 1, we initialize $v_0$ from a uniform start, $v_0\!=\!\mathbbm 1/n$, without using a warm start.

\begin{algorithm}[H]
\caption{PerronSumRatio}
\label{alg:perron}
\footnotesize
\begin{algorithmic}[1]
\Require Non-negative square matrix
$W \in \mathbb{R}_{\geq 0}^{n \times n}$; iteration count $n_{\mathrm{iter}}$.
\Ensure Estimate $\hat\lambda_{\mathrm{Perron}}(W)$.
\State $v_0 \gets \mathbbm 1\,/\, n$
\For{$k \gets 1$ \textbf{to} $n_{\mathrm{iter}}$}
    \State $v_k \gets W\,v_{k-1} \,/\, \|W\,v_{k-1}\|_2$
           \Comment{Normalized power iteration step}
\EndFor
\State $\hat\lambda \gets \dfrac{\mathbbm 1^\top\,W\,v_{n_{\mathrm{iter}}}}
                                {\mathbbm 1^\top\,v_{n_{\mathrm{iter}}}}$
       \Comment{Sum ratio, \eqref{eq:perron-sumratio}}
\State \Return $\hat\lambda$
\end{algorithmic}
\end{algorithm}

\section{Experimental Details}
\label{app:exp}
In this section, we present the preprocessing and the associated datasets used in this paper (Appendix~\ref{app:exp:data}). Appendix~\ref{app:exp:hparams} presents the architecture and training hyperparameters used during training. Finally, the compute resources and the low-variance estimators for \gls{tide} are provided in Appendices~\ref{app:exp:compute} and~\ref{app:exp:estimators}.

\subsection{Datasets \& Preprocessing}
\label{app:exp:data}
 
\gls{tide} is evaluated on five image-classification datasets, namely: \texttt{MNIST} \citep{lecun1998gradient}, \texttt{Fashion-MNIST}~\cite{xiao2017fashion}, \texttt{CIFAR-10} and \texttt{CIFAR-100}~\cite{krizhevsky2009cifar}, and \texttt{ImageNet-1K}~\cite{russakovsky2015imagenet}. The standard recommended train and validation splits for each dataset are used, and evaluation during training is performed on the provided test sets. The best evaluation checkpoints are used for all empirical studies, including the \gls{ood} analysis. Moreover, the preprocessing per dataset is as follows:

\begin{itemize}
    \item \texttt{MNIST}: $\mathrm{ToTensor}+\mathrm{Normalize((0.1307),\,(0.3081))}$
    \item \texttt{Fashion-MNIST}: $\mathrm{ToTensor}+\mathrm{Normalize((0.2860),\,(0.3530))}$
    \item \texttt{CIFAR-10/100}: $\mathrm{RandomCrop} (32,4)+\mathrm{RandomHorizontalFlip}+\mathrm{ToTensor}+\mathrm{Normalize}$, where for \texttt{CIFAR-10}, we used $(0.4914,0.4822,0.4465)/(0.2023,0.1994,0.2010)$ for per channel normalization, while $(0.5071,0.4867,0.4408)/(0.2675,0.2565,0.2761)$ are used for \texttt{CIFAR-100}. At test time for both \texttt{CIFAR-10} and \texttt{CIFAR-100}, only $\mathrm{ToTensor}+\mathrm{Normalize}$ is applied.
    \item \texttt{ImageNet-1K}: At training time, we used $\mathrm{RandomResizedCrop}(224)+\mathrm{RandomHorizontalFlip}+\mathrm{Normalize}\big((0.485,0.456,0.406),(0.229,0.224,0.225)\big)$ and at test time, $\mathrm{Resize}(256)+\mathrm{CenterCrop}(224)+\mathrm{Normalize}$ is used with the same normalization values as training time.
\end{itemize}

\subsection{Architecture \& Training Hyperparameters}
\label{app:exp:hparams}
Table~\ref{tab:hparams} lists the dataset-specific hyperparameters used across all studies in this paper. The provided values are identical across all random seeds, and the recommended dataset split is used for each dataset. Moreover, two backbone variants are used: small datasets use the shallow \gls{hrf}, while \texttt{ImageNet-1K} is trained and evaluated on the deep \gls{hrf} due to its feature complexity. As motivated by~\cite{darlow2025continuous}, \texttt{ImageNet-1K} uses larger \gls{nlm} hidden dimension, $H\!=\!32$, in contrast to other datasets where $H\!=\!4$. Additionally, to stabilize the training and minimize the computational resources used for ImageNet-1K, four additional strategies, including truncated \gls{bptt} with $K\!=\!25$, per-component gradient clipping, spectral regularization with $(\tau_{EE},\tau_{II})\!=\!(15,7)$, and mixed-precision AMP on the backbone, are used. Note that the Wilson-Cowan dynamics are kept in \texttt{float32} precision to avoid overflow~\cite{micikevicius2018mixed}.

\begin{table}[t!]
\centering\small
\setlength{\tabcolsep}{4.8pt}
\caption{Per-dataset training hyperparameters. Empty cells indicate the option is inactive. Shared constants are provided in \ref{tab:hparams-shared}.}
\label{tab:hparams}
\begin{tabular}{lccccc}
\toprule
\textbf{Parameter} & \texttt{MNIST} & \texttt{Fashion-MNIST} & \texttt{CIFAR-10} & \texttt{CIFAR-100} & \texttt{ImageNet-1K}\\
\midrule
$d_{\mathrm{model}}$ & 256 & 256 & 512 & 718 & 4096\\
$(n_E,n_I)$ & (205,51) & (205,51) & (410,102) & (574,144) & (3277,819)\\
Backbone & HRF & HRF & HRF & HRF & Deep-HRF\\
Positions $P$ & 64 & 64 & 64 & 64 & 196\\
$(d_{EE},d_{EI},d_{II})$ & (256,128,64) & (256,128,64) & (256,128,64) & (256,128,64) & (4096,2048,2048)\\
$n_{\mathrm{heads}}$ & 8 & 8 & 8 & 8 & 16\\
NLM hidden $H$ & 4 & 4 & 4 & 4 & 32\\
Output-head hidden & 256 & 256 & 256 & 256 & 2048\\
Cross-attn residual & --- & --- & --- & --- & on\\
Sync decay $r$ & 1 & 1 & 1 & 1 & $e^{-0.5}$\\
Sync accum.\ clamp & --- & --- & --- & --- & 100\\
Local batch & 64 & 64 & 256 & 256 & 64\\
Grad-accum & 1 & 1 & 1 & 1 & 4\\
DDP size & 1 & 1 & 1 & 1 & 4\\
Effective batch & 64 & 64 & 256 & 256 & 1024\\
Learning rate & $10^{-3}$ & $10^{-3}$ & $10^{-4}$ & $10^{-4}$ & $3{\cdot}10^{-4}$\\
LR warmup & 2K & 2K & 10K & 10K & 10K\\
Total steps & 50K-200K & 50K-200K & 600K & 600K & 100K-500K\\
Weight decay & $10^{-2}$ & $10^{-2}$ & $10^{-2}$ & $10^{-2}$ & $10^{-2}$\\
Dropout & 0.1 & 0.1 & 0.0 & 0.0 & 0.1\\
Mixed precision & --- & --- & --- & --- & AMP (Backbone)\\
Grad clip (global) & 1.0 & 1.0 & 1.0 & 1.0 & 20.0\\
Per-group clip & --- & --- & --- & --- & (2, 2, 50, 5, 5, 20)\\
TBPTT $K$ & 0 & 0 & 0 & 0 & 25\\
$\lambda_{\mathrm{spec}}$ & 0 & 0 & 0 & 0 & 0.1\\
$(\tau_{EE},\tau_{II})$ & --- & --- & --- & --- & (15, 7)\\
Curriculum $(t_s,T_w)$ & (1K,5K) & (1K,5K) & (1K,5K) & (1K,5K) & (10K,10K)\\
Seed & 42-46 & 42-46 & 42-46 & 42-46 & 42-46\\
\bottomrule
\end{tabular}
\end{table}

\begin{table}[!t]
\centering
\caption{Hyperparameters shared across all five datasets. The Euler
coefficients $\alpha_\bullet$ are derived from $\Delta t$ and the time constants $\tau_\bullet$.}
\label{tab:hparams-shared}
\small
\setlength{\tabcolsep}{35.2pt}
\renewcommand{\arraystretch}{1.05}
\begin{tabular}{@{}llc@{}}
\toprule
\textbf{Group} & \textbf{Parameter} & \textbf{Value} \\
\midrule
\multirow{4}{*}{\textit{Recurrent Dynamics}}
 & Internal computation steps $T$
 & $50$ \\
 & E-population fraction $f_E$
 & $0.8$ \\
 & Population time constants $(\tau_E, \tau_I)$
 & $(20,\, 5)\,\mathrm{ms}$ \\
 & Step size $\Delta t$ \;($\alpha_E = 0.05$, $\alpha_I = 0.20$)
 & $1\,\mathrm{ms}$ \\
\midrule
\multirow{2}{*}{\textit{Neuron-level Model}}
 & NLM memory length $M$
 & $25$ \\
 & NLM layers
 & $2$ \\
\midrule
\multirow{2}{*}{\textit{Lateral Inhibition}}
 & WTA iterations $K_{\mathrm{WTA}}$
 & $5$ \\
 & Lateral gain $\gamma$ (learnable; initial value)
 & $1.0$;\, $0.01$ \\
\midrule
\multirow{5}{*}{\textit{Memory}}
 & Buffer dimension $d_{\mathrm{memory}}$
 & $256$\\
 & Surprise threshold $\theta_s$
 & $0.5$ \\
 & Write step size $\eta_m$
 & $10^{-2}$ \\
 & Write momentum $\mu$
 & $0.9$ \\
 & Target \gls{ei} ratio $\rho^\ast_{EI}$
 & $4$ \\
\midrule
\multirow{4}{*}{\textit{Loss Weights}}
 & Task loss weight $\lambda_{\mathrm{task}}$
 & $1.0$ \\
 & \gls{ei} ratio weight $\lambda_{\mathrm{EI}}$
 & $10^{-2}$ \\
 & Game-theoretic weight $\lambda_{\mathrm{game}}$
 & $10^{-3}$ \\
 & Sync regularizer weight $\lambda_{\mathrm{sync}}$
 & $10^{-4}$ \\
\bottomrule
\end{tabular}
\end{table}

\subsection{Compute Resources}
\label{app:exp:compute}
Training and evaluation on the smaller datasets were run on a single NVIDIA V100-SXM2 GPU with 32GB VRAM, while \texttt{ImageNet-1K} required four 32GB V100 GPUs with PyTorch \gls{ddp}. Note that all training and evaluation were performed using CUDA~12.4 and PyTorch~2.6.0. Total wall-clock time required for training of each of the datasets per single seed is as follows: a) \texttt{MNIST}: $0.85$ steps/s, $65$ h for 50K steps. b) \texttt{Fashion-MNIST} $0.83$ steps/s, $67$ h for 50K steps. c) \texttt{CIFAR-10} $0.68$ steps/s, $246$ h for 600K steps. d) \texttt{CIFAR-100} $0.69$ steps/s, $207$ h for 300 K steps. e) \texttt{ImageNet-1K} $0.105$ steps/s, $287$ h for 100K steps.

\begin{figure}[!t]
\centering
\includegraphics[clip, trim=0.5cm 20.00cm 0.5cm 0.5cm, width=0.86\textwidth]{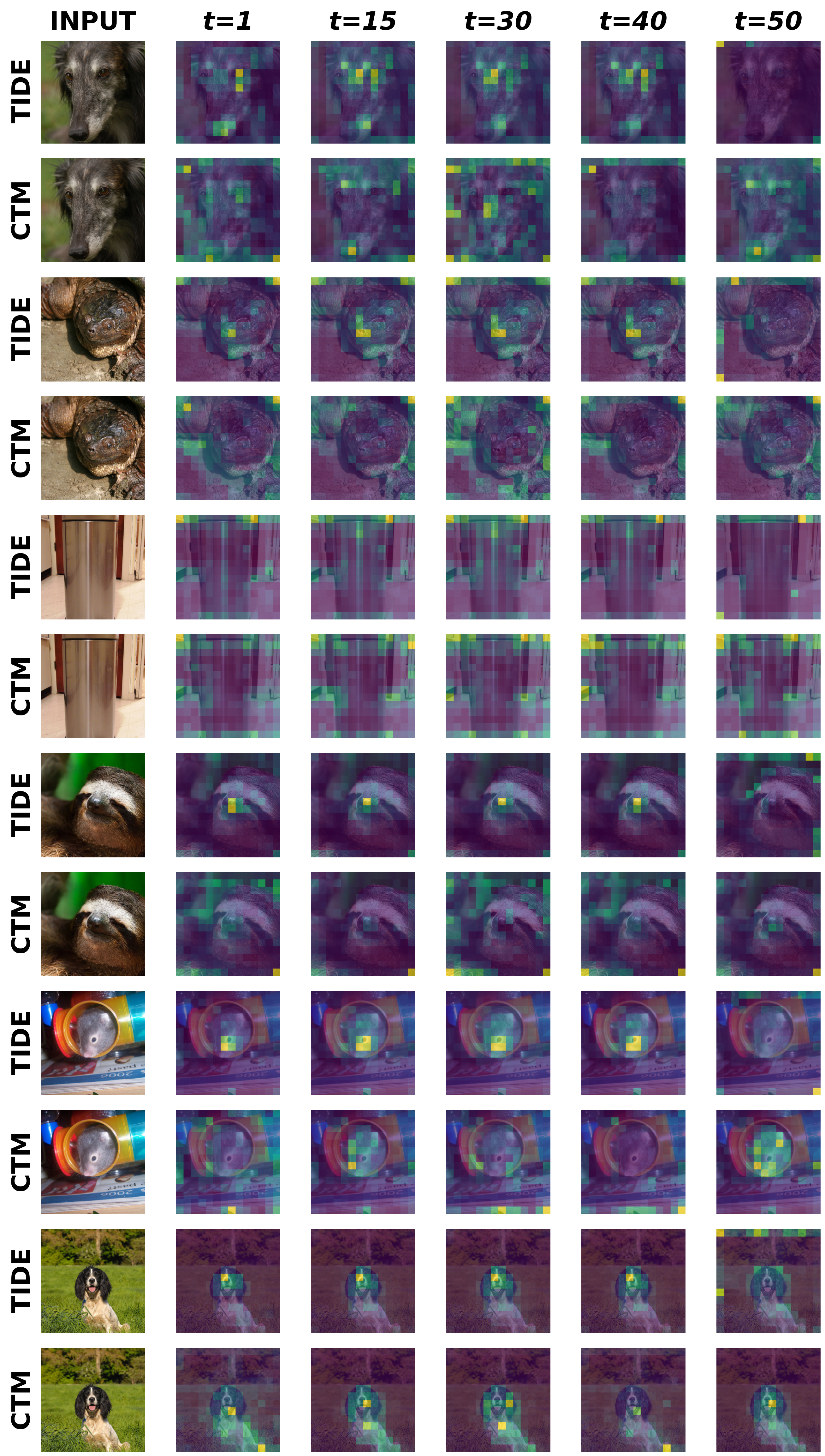}
\caption{Temporal evolution of mean attention as saliency per internal computation step for \gls{tide} versus \gls{ctm} for multiple cases of successful classification.}
\label{fig:saliancy_map}
\end{figure}

\begin{table}[!b]
\centering
\caption{Comparison between \gls{tide} and \gls{ctm} on five image classification datasets at equal training budget. The \emph{Best Seed} column reports the highest-performing run with its best/final \texttt{top-1} accuracy, $\mathrm{Acc@1}$ (\%), and \texttt{mean\,$\pm$\,std} across the multi seeded training for both architectures. \gls{ctm} is retrained based on~\cite{darlow2025continuous} for all the ablations with identical seeds. Note that ${}^\dagger$ denotes the mean \texttt{top-1} accuracy drop of more than $1$\%, and ${}^\ddagger$ indicates an effective batch size of 64. Moreover, the highest \texttt{top-1} accuracy for each dataset is in bold.}
\label{tab:main-bench-supp}
\scriptsize
\setlength{\tabcolsep}{2pt}
\renewcommand{\arraystretch}{1.05}
\newcommand{\sdfont}{\fontsize{6}{7}\selectfont}
\newcommand{\ms}[2]{#1\textsubscript{\sdfont\textpm#2}}
\newcommand{\nodag}{\hphantom{${}^{\dagger}$}}
\begin{tabular*}{\textwidth}{@{\extracolsep{\fill}}l@{\hskip 5pt}l@{\hskip 5pt}l@{\hskip 14pt}c@{\hskip 6.5pt}c@{\hskip 11pt}c@{\hskip 6.5pt}c@{\hskip 14pt}c@{\hskip 6.5pt}c@{\hskip 11pt}c@{\hskip 6.5pt}c@{}}
\toprule
 & & & \multicolumn{4}{c@{\hskip 14pt}}{\textbf{TIDE}} & \multicolumn{4}{c@{}}{\textbf{CTM}} \\
\cmidrule(lr){4-7}\cmidrule(l){8-11}
 & & & \multicolumn{2}{c@{\hskip 11pt}}{\textbf{\texttt{mean\,$\pm$\,std}}} & \multicolumn{2}{c@{\hskip 14pt}}{\textbf{Best Seed}} & \multicolumn{2}{c@{\hskip 11pt}}{\textbf{\texttt{mean\,$\pm$\,std}}} & \multicolumn{2}{c@{}}{\textbf{Best Seed}} \\
\cmidrule(lr){4-5}\cmidrule(lr){6-7}\cmidrule(lr){8-9}\cmidrule(l){10-11}
\textbf{Task} & \textbf{Steps} & $d_{\mathrm{model}}$ & best & final & best & final & best & final & best & final \\
\midrule
MNIST & 50K & 256 & \textbf{\ms{99.64}{0.03}} & \textbf{\ms{99.61}{0.04}}\nodag & \textbf{99.67} & \textbf{99.63}\nodag & {\ms{99.54}{0.03}} & {\ms{99.52}{0.02}} & {99.56} & {99.54\nodag} \\
& 100K & 256 & \textbf{\ms{99.54}{0.05}} & \ms{96.73}{4.67}${}^{\dagger}$ & \textbf{99.59} & 91.33${}^{\dagger}$ & \textbf{\ms{99.54}{0.05}} & \textbf{\ms{99.49}{0.07}} & \textbf{99.59} & \textbf{99.57}\nodag \\
& 200K & 256 & \textbf{\ms{99.60}{0.05}} & \textbf{\ms{99.53}{0.07}}\nodag & \textbf{99.64} & \textbf{99.60}\nodag & \ms{99.59}{0.06} & \ms{99.49}{0.08} & \textbf{99.64} & 99.57\nodag \\
\midrule
F-MNIST & 50K & 256 & \textbf{\ms{94.21}{0.21}} & \textbf{\ms{94.15}{0.25}}\nodag & \textbf{94.42} & \textbf{94.42}\nodag & {\ms{92.51}{0.11}} & {\ms{92.45}{0.19}} & {92.63} & {92.63\nodag} \\
& 200K & 256 & \textbf{\ms{94.12}{0.19}} & \textbf{\ms{93.85}{0.10}}\nodag & \textbf{94.23} & \textbf{93.96}\nodag & \ms{92.53}{0.24} & \ms{92.43}{0.27} & 92.80 & 92.73\nodag \\
\midrule
CIFAR-10 & 600K & 512 & \textbf{\ms{90.62}{0.06}} & \textbf{\ms{90.54}{0.07}}\nodag & \textbf{90.68} & \textbf{90.62}\nodag & \ms{86.23}{0.34} & \ms{85.98}{0.35} & 86.59 & 86.31\nodag \\
\midrule
CIFAR-100 & 600K & 718 & \ms{62.23}{0.16} & \ms{61.39}{0.82}\nodag & 62.38 & 61.97\nodag & \multirow{2}{*}{\ms{65.19}{0.48}} & \multirow{2}{*}{\textbf{\ms{65.01}{0.39}}} & \multirow{2}{*}{65.70} & \multirow{2}{*}{\textbf{65.44}\nodag} \\
 & 600K${}^{\ddagger}$ & 718 & \textbf{\ms{65.68}{0.18}} & \ms{64.50}{0.60}${}^{\dagger}$ & \textbf{65.79} & 65.08\nodag &  &  &  &  \\
\midrule
ImageNet-1K & 100K & 4096 & \textbf{\ms{67.22}{1.52}} & \textbf{\ms{67.01}{1.58}}\nodag & \textbf{68.74} & \textbf{68.59}\nodag & \ms{52.67}{1.79} & \ms{51.47}{1.62} & 54.46 & 53.09\nodag \\
& 200K & 4096 & \textbf{\ms{68.40}{1.91}} & \textbf{\ms{68.18}{2.25}}\nodag & \textbf{69.55} & \textbf{69.50}\nodag & \ms{62.78}{1.31} & \ms{61.81}{1.50} & 64.06 & 62.88${}^{\dagger}$ \\
 & 500K & 4096 & --- & --- & 69.90 & 69.87\nodag & --- & --- & \textbf{71.78} & \textbf{69.88}${}^{\dagger}$ \\
\bottomrule
\end{tabular*}
\end{table} 
\subsection{Low-variance Estimator}
\label{app:exp:estimators}
Let $w(\mathrm{step}) \in [0, 1]$ be a multiplicative weight for the per-step gradient estimator of~\eqref{eq:rb-curr}, thus, its squared deviation must be bounded component-wise by $(1 - w(\mathrm{step}))^2 \le 1$ as $w \equiv 1$. The cosine scheduled ramp \eqref{eq:curriculum} with $w(t) = \tfrac{1}{2}(1 - \cos(\pi t))$ on $t = (\mathrm{step} - t_s)/T_w \in [0, 1]$, admits the closed form $\frac{1}{T_w}\!\int_{t_s}^{t_s + T_w}\!(1 - w(s))^2\,\mathrm{d}s = 3/8$, where $t_s$ is the step at which the cosine warm-up initializes, and $T_w$ denote its length. Therefore, it can be stated that the cosine warm-up acts as a controlled, decaying perturbation of the immediate-on $w \equiv 1$ schedule, contributing on average $3/8$ squared deviation.
 
\section{Additional Experiments \& Ablations}
\label{app:add}
This section presents the main benchmark results, ablation studies, and
robustness analysis. Per-dataset benchmarks with their related training and evaluation metrics are reported, while the comparative analysis of their performance against \gls{ctm} as a baseline is included in Appendix~\ref{app:add:benchmarks}. Furthermore, the ablation study of hyperparameters on the \texttt{MNIST} and \texttt{Fashion-MNIST} datasets is reported across various subsections. Appendix~\ref{app:add:ei-ratio} focuses on the optimal \gls{ei} ratio, $\rho_{EI}^\ast\!=\!n_E/n_I$, and its effect on the resulted \texttt{top-1} accuracy. Appendix~\ref{app:add:tau} provides the ablation study of the population-specific time constants $\tau_E,\tau_I$, while Appendix~\ref{app:add:niter} is focused on the number of internal computation steps $T$, required for stability of training. Furthermore, the game-theoretic loss weight $\lambda_{\mathrm{game}}$ is studied in Appendix~\ref{app:add:game}, and the results for lateral inhibition iteration count $K_{\mathrm{WTA}}$ are provided in Appendix~\ref{app:add:wta}. Thereafter, we analyze the training stability of \texttt{ImageNet-1K} via its gradient-norm curves and spectral radii during training in Appendix~\ref{app:add:stability}. Further component-wise analysis of the architectures and their contribution to the performance of \gls{tide} is reported in Appendix~\ref{app:add:loo}. Additional analysis is performed to investigate the robustness of the backbone in Appendix~\ref{app:add:backbone}. Finally, we provide a detailed comparative analysis of the performance of \gls{ctm} and \gls{tide} under perturbations, while investigating their robustness in \gls{ood} cases in Appendix~\ref{app:add:robust}.

\begin{figure}[!t]
\centering
\includegraphics[clip, trim=0.5cm 0.50cm 0.5cm 0.5cm, width=\textwidth]{figures/final/ablation/ctm_vs_tide_benchmark_grid_full_longest.png} 
\caption{Comparison between \gls{tide} and \gls{ctm} based on their training and validation curves, while providing \texttt{top-1} accuracy (\%), and \texttt{mean\,$\pm$\,std} total losses across all the trained seeds.}
\label{fig:eval_benchmark}
\end{figure}

\subsection{Main Benchmark}
\label{app:add:benchmarks}
The \texttt{top-1} accuracy of \gls{tide} alongside \gls{ctm} is reported in Table~\ref{tab:main-bench-supp}. We trained both \gls{tide} and \gls{ctm} with the identical hyperparameters per dataset provided in Table~\ref{tab:hparams}. To directly compare the convergence of each of the architectures, both are trained with similar training budget to analyze their sample efficiency. Note that all the provided results and metrics are based on our own re-run of the \gls{ctm} under identical training. In addition to quantitative analysis, qualitative results are provided, showing the visual temporal evolution of attention for \gls{tide} and \gls{ctm} to enable direct comparison of the behavior of both architectures (cf. Figure~\ref{fig:saliancy_map}). We additionally analyzed the performance of ResNet-18 as a backbone for all datasets as part of leave-one-out ablation study (cf. Appendix~\ref{app:add:loo}), which shows significant performance degradation compared to both variants of \gls{hrf}.

The remaining discussion focuses on the multi-seeded training of \gls{tide} and the results achieved compared with \gls{ctm}. Note that the backbone differs significantly between \gls{tide} and \gls{ctm} as the main aim of \gls{tide} is to achieve bio-plausible end-to-end training by utilizing \gls{hrf} and its two variants. We use an identical $\dmodel$ for both \gls{tide} and \gls{ctm} while investigating each dataset as shown in Table~\ref{tab:main-bench-supp}. Additionally, evaluation curves based on the combined results of multi-seeded training of \gls{tide} and \gls{ctm} are provided in Figure~\ref{fig:eval_benchmark}.

It has been observed that \gls{tide} can achieve higher evaluation accuracy and exceeds \gls{ctm} on \texttt{MNIST} (99.67 vs.\ 99.64), \texttt{Fashion-MNIST} (94.42 vs.\ 92.80), \texttt{CIFAR-10} (90.68 vs.\ 86.59), \texttt{CIFAR-100} (65.79 vs.\ 65.70) and \texttt{ImageNet-1K}-$200$\,K (69.55 vs.\ 64.06), where the \gls{ctm}-$200$\,K baseline is under-trained relative to \gls{ctm} $500$\,K-step result of $71.78\%$ which is achieved via our re-training and is also reported by~\cite{darlow2025continuous}. We further investigate backbone selection by initially replacing it with ResNet-18, which overall reduces \texttt{top-1} accuracy across all datasets, as shown in Appendix~\ref{app:add:loo}. We conjecture this is due to the projection enforcing Dale's principle and its poor interaction with batch-normalization statistics inherited from the ResNet stem. Thus, for the \texttt{ImageNet-1K} dataset, a revised ResNet-style deep \gls{hrf} is devised to address this issue.

As shown in Figure~\ref{fig:attention_grid}, we perform a diversity-based analysis in which we randomly select $1000$ images across all classes in \texttt{ImageNet-1K} to compare \gls{tide} and \gls{ctm}. \gls{tide} achieved $70.4\%$ against the $500$\,K steps trained \gls{ctm}'s $67.1\%$ \texttt{top-1} accuracy. Figure~\ref{fig:attention_grid} shows the decoupled of classification from the internal steps for both models, where \gls{tide} outperforms \gls{ctm} by achieving a higher accuracy rate. Furthermore, it has been observed that \gls{tide} has higher mean certainty than \gls{ctm} ($0.89$ vs.\ $0.80$), indicating that \gls{ei}-based dynamics produce more decisive readouts than \gls{ctm} by relying on its \gls{wta} and surprise-gated memory. Both \gls{ctm} and \gls{tide} have a prediction agreement of $79.1\%$ for correctly classified images, indicating substantial agreement well above chance.

\begin{figure}[!t]
\centering
\includegraphics[width=\textwidth]{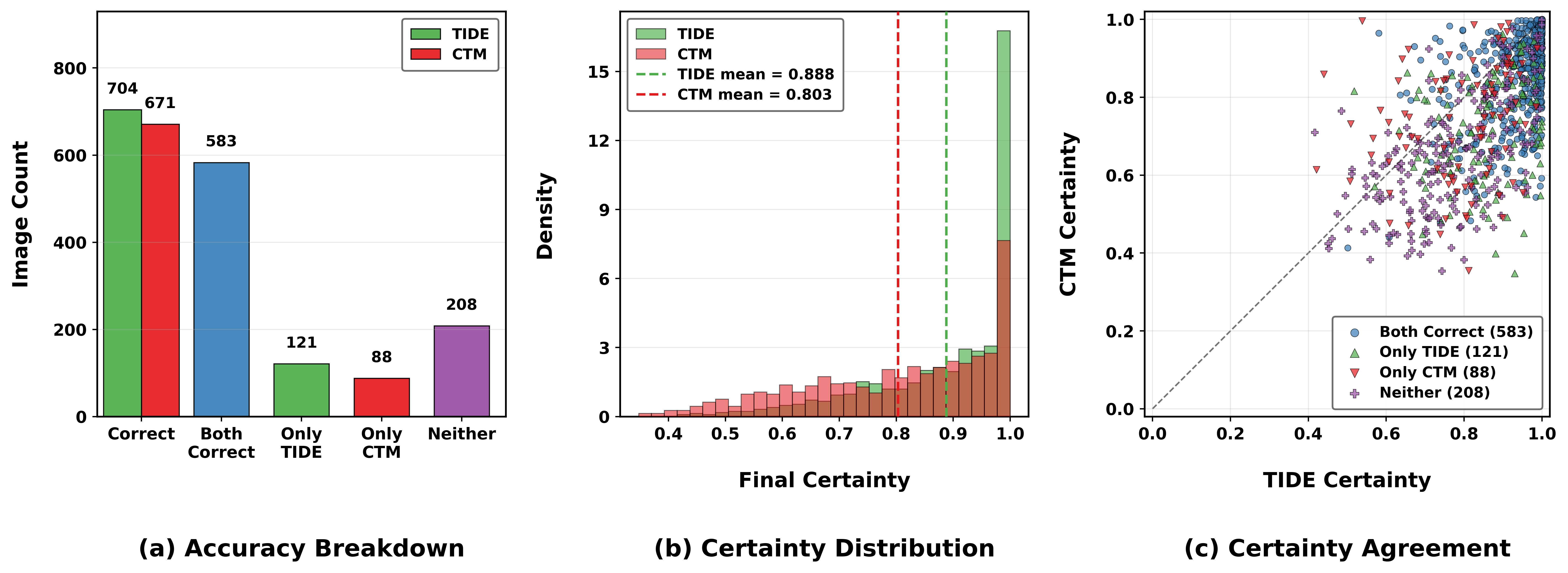}
\caption{\gls{tide} versus \gls{ctm} accuracy and certainty on a diverse
$1000$ randomly selected images subset of the \texttt{ImageNet-1K} validation set per each class. (a) Accuracy breakout and comparison across \gls{tide} and \gls{ctm}. (b) Certainty distribution for accurate classification and \texttt{mean} confidence comparison. (c) Agreement analysis based on all classifications across the randomly selected samples.}
\label{fig:attention_grid}
\end{figure}

\begin{table}[!b]
\centering
\caption{\gls{ei} population ratio ablation study. Baseline
$n_E/d_{\mathrm{model}} = 0.8$ is in bold. Final \gls{ei} activity ratio (post-training) is reported to facilitate analysis of stability, given the target has been set as $\rho_{EI}^{\ast} = 4.0$.}
\label{tab:ei-sweep}
\footnotesize
\setlength{\tabcolsep}{8.5pt}
\renewcommand{\arraystretch}{1.05}
\begin{tabular}{@{}lccccc@{}}
\toprule
& \multicolumn{3}{c}{\textbf{\texttt{MNIST}}}
& \multicolumn{2}{c@{}}{\textbf{\texttt{Fashion-MNIST}}} \\
\cmidrule(lr){2-4} \cmidrule(lr){5-6}
$n_E/d_{\mathrm{model}}$
 & \textbf{best} & \textbf{final} & \textbf{$\rho_{EI}$}
 & \textbf{best} & \textbf{final} \\
\midrule
0.6           & 99.53          & 98.95          & 3.94          & 93.53          & 93.51 \\
0.7           & 99.55          & 99.39          & 4.01          & 93.67          & 93.61 \\
\textbf{0.8}  & \textbf{99.53} & \textbf{99.51} & \textbf{4.01} & \textbf{93.53} & \textbf{86.19} \\
0.9           & 99.55          & 99.28          & 4.01          & 93.49          & 93.34 \\
\bottomrule
\end{tabular}
\end{table}
\subsection{E-I Ratio Analysis}
\label{app:add:ei-ratio}
 
Table~\ref{tab:ei-sweep} provides the results of the \gls{ei} ratio ablation study, where we analyze four population split configurations,  $\rho_{EI}^\ast\!=\!n_E/n_I$, on \texttt{MNIST} and \texttt{Fashion-MNIST}. Values
$\{0.6,0.7,0.8,0.9\}$ correspond to $n_E/d_{\mathrm{model}}$, hence $\rho_{EI}^\ast\!=\!n_E/(d_{\mathrm{model}}-n_E)\!\in\!\{1.5,2.33,4.0,9.0\}$.
The biologically-motivated default ratio $80{:}20$ ($n_E/d_{\mathrm{model}}\!=\!0.8$) achieves
the baseline \texttt{MNIST} result of $99.51\%$ and is within $0.02$ percentage points of the best variant ($n_E/d\!=\!0.7$) on \texttt{MNIST}.
 
\begin{figure}[!t]
  \centering
  \includegraphics[clip, trim=0.5cm 0.25cm 0.5cm 0.5cm, width=0.85\textwidth]{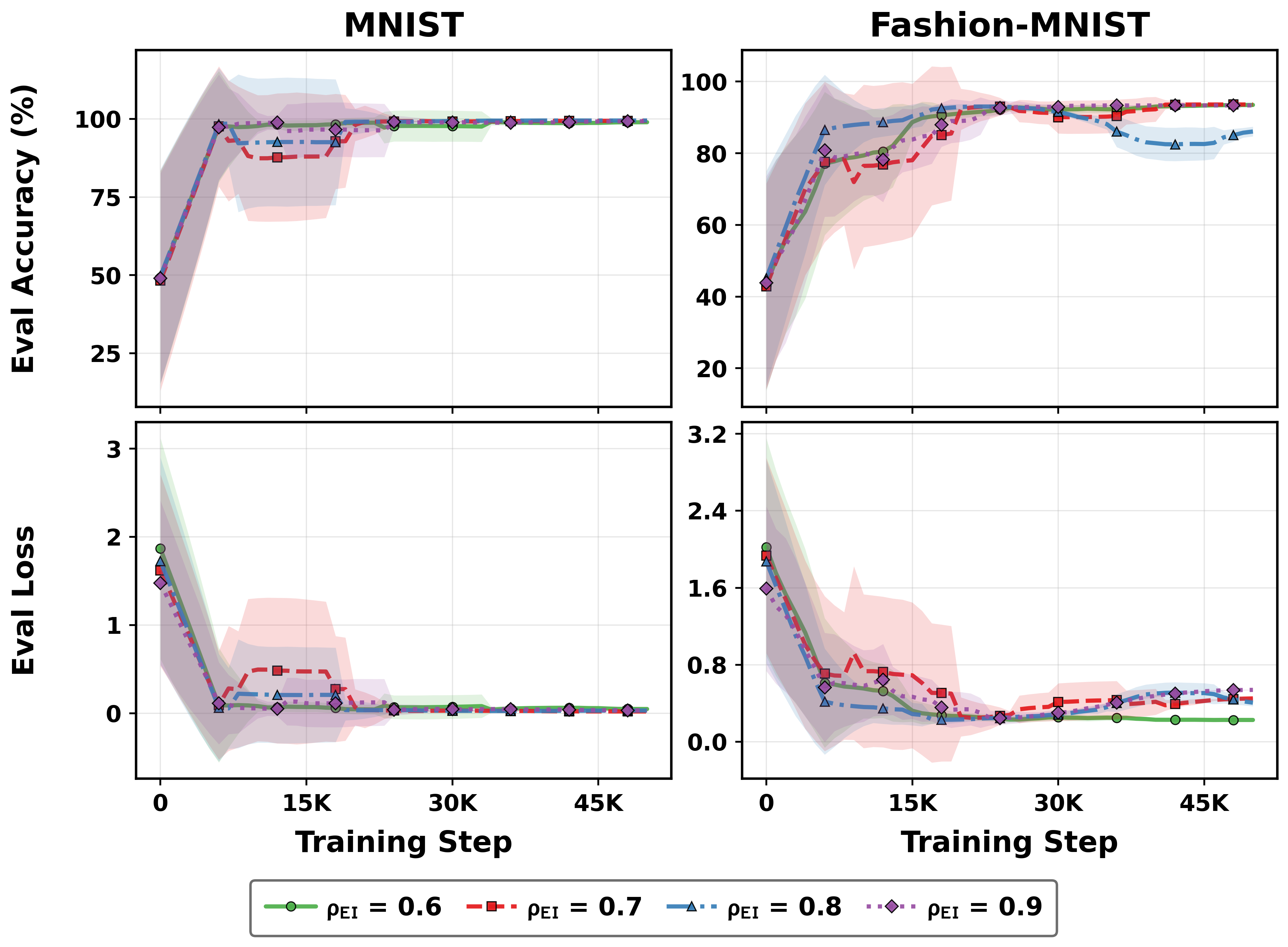}
  \caption{Evaluation curves during ablation study to analyze \gls{ei} population ratio on \texttt{MNIST} and \texttt{Fashion-MNIST}.}
  \label{fig:ei_ablation}
\end{figure}

As illustrated in Figure~\ref{fig:ei_ablation}, the \texttt{Fashion-MNIST} run with the default \gls{ei} ratio exhibits post-peak drift ($86.19\%$) that disappears at $n_E/d\!\in\!\{0.6,0.7,0.9\}$. Across all configurations, $\rho_{EI}$ values for \texttt{MNIST} column in Table~\ref{tab:ei-sweep} track the $80{:}20$ desired target ratio to within $2\%$, indicating that the activity-ratio regularizer $\mathcal L_{\mathrm{EI}}$ (\ref{eq:ei-loss}) is effective regardless of the population split.

\subsection{Time-constant Analysis}
\label{app:add:tau}
 
Table~\ref{tab:tau-sweep} provides the ablation study on the excitatory and inhibitory membrane time constants. The biological regime requires $\tau_I\!<\!\tau_E$; we verify this by performing two parallel ablation studies where either $\tau_I\!=\!5$ ms is held fixed, or $\tau_E\!=\!20$ ms is set as a constant, as shown in Figures~\ref{fig:taue_ablation} and \ref{fig:taui_ablation}.

\begin{table}[!b]
\centering
\caption{Population time-constant ablation (ms). Baseline values
$(\tau_E, \tau_I) = (20, 5)\,\mathrm{ms}$ are in bold.}
\label{tab:tau-sweep}
\footnotesize
\setlength{\tabcolsep}{8.5pt}
\renewcommand{\arraystretch}{1.05}
\begin{tabular}{@{}lcccc@{\hskip 12pt}lcccc@{}}
\toprule
& \multicolumn{2}{c}{\textbf{\texttt{MNIST}}}
& \multicolumn{2}{c@{\hskip 12pt}}{\textbf{\texttt{Fashion-MNIST}}}
& & \multicolumn{2}{c}{\textbf{\texttt{MNIST}}}
& \multicolumn{2}{c@{}}{\textbf{\texttt{Fashion-MNIST}}} \\
\cmidrule(lr){2-3} \cmidrule(lr){4-5} \cmidrule(lr){7-8} \cmidrule(lr){9-10}
$\tau_E$ (ms) & \textbf{best} & \textbf{final} & \textbf{best} & \textbf{final}
& $\tau_I$ (ms) & \textbf{best} & \textbf{final} & \textbf{best} & \textbf{final} \\
\midrule
10           & 98.69          & 95.97          & 93.91          & 89.45
& 3           & 99.58          & 99.56          & 93.64          & 93.50 \\
15           & 99.59          & 99.47          & 93.80          & 93.62
& \textbf{5}  & \textbf{99.58} & \textbf{99.52} & \textbf{93.15} & \textbf{92.56} \\
\textbf{20}  & \textbf{99.62} & \textbf{99.61} & \textbf{94.23} & \textbf{93.73}
& 7           & 99.49          & 99.48          & 94.00          & 93.88 \\
25           & 99.61          & 99.61          & 94.25          & 94.05
& 10          & 99.55          & 99.55          & 93.78          & 93.73 \\
30           & 99.54          & 99.50          & 92.95          & 92.85
&             &                &                &                &      \\
\bottomrule
\end{tabular}
\end{table}

\begin{figure}[!t]
  \centering
  \includegraphics[clip, trim=0.5cm 0.25cm 0.5cm 0.5cm, width=0.85\textwidth]{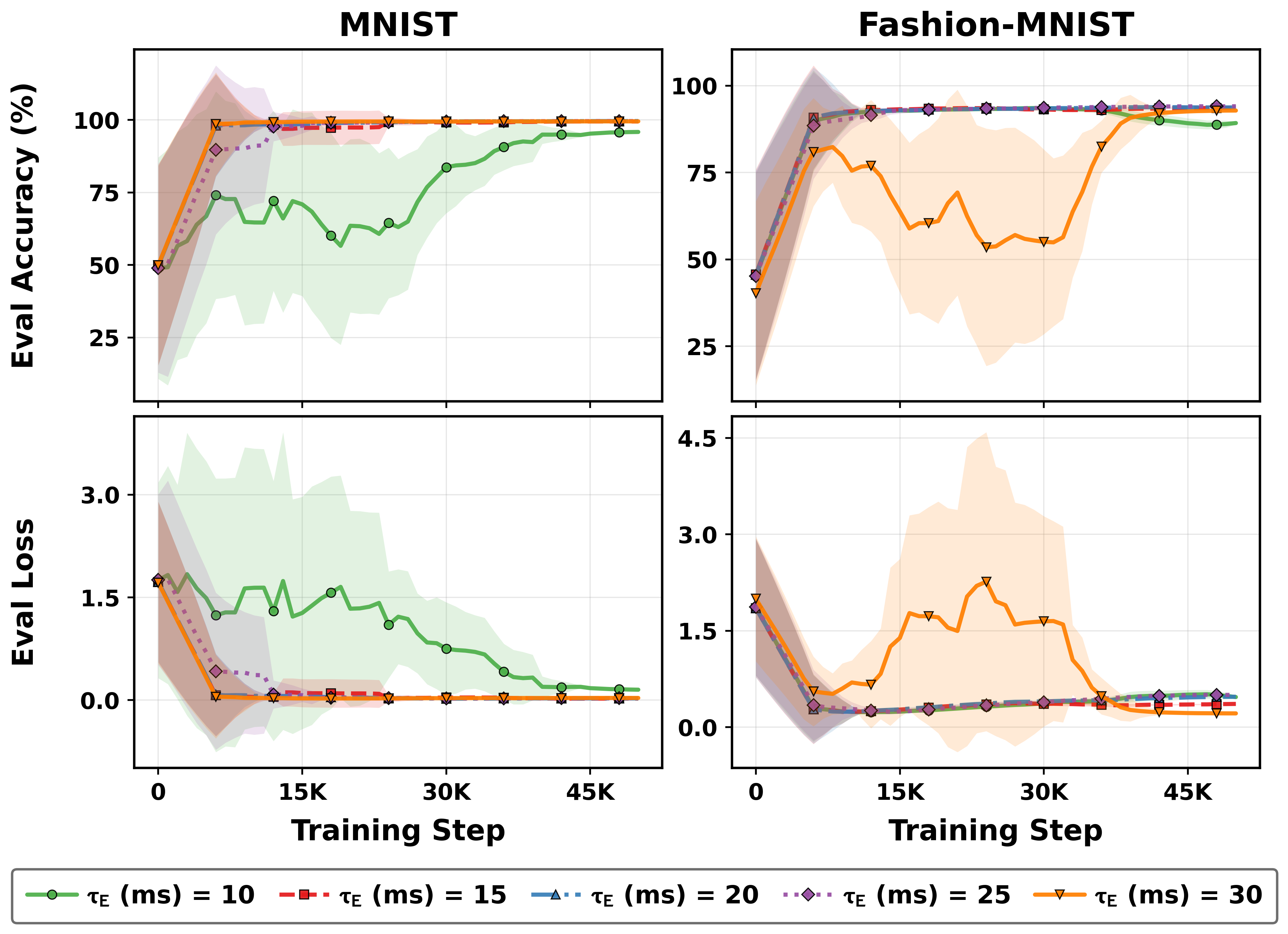}
  \caption{Evaluation curves during ablation study to analyze $\tau_E$ impact on \texttt{MNIST} and \texttt{Fashion-MNIST}.}
  \label{fig:taue_ablation}
\end{figure}

\begin{figure}[!t]
  \centering
  \includegraphics[clip, trim=0.5cm 0.25cm 0.5cm 0.5cm, width=0.85\textwidth]{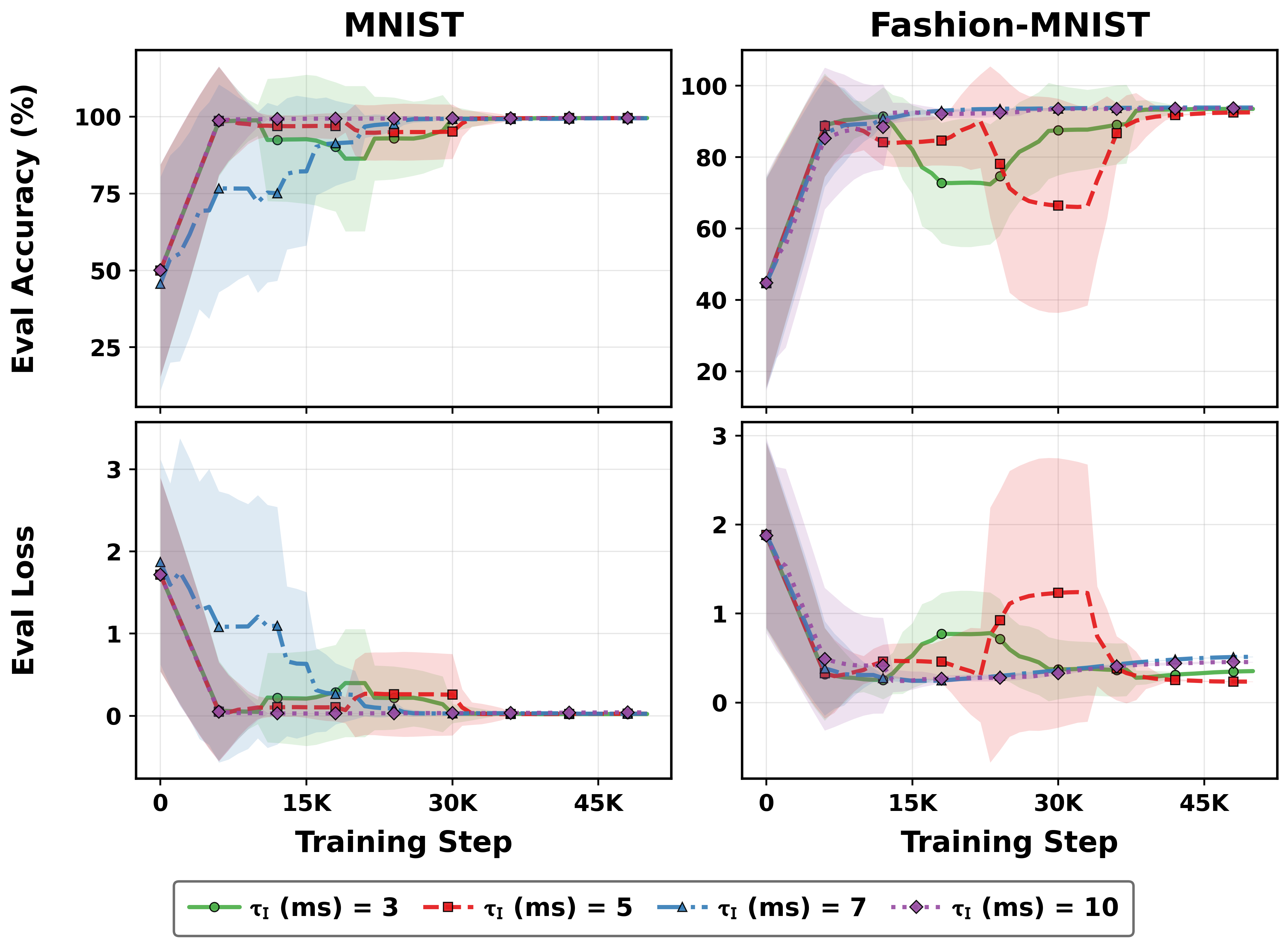}
  \caption{Evaluation curves during ablation study to analyze $\tau_I$ impact on \texttt{MNIST} and \texttt{Fashion-MNIST}.}
  \label{fig:taui_ablation}
\end{figure}

A noticeable performance degradation is observed when $\tau_E\!=\!10$ ms; both \texttt{MNIST} and \texttt{Fashion-MNIST} accuracies decrease by $1$ and $4$ percentage points, respectively. This is a direct result of $\Delta t/\tau_E\!=\!\alpha_E\!=\!0.1$ rather than $0.05$ when $\tau_E/\tau_I\!=\!2$, which causes an overshoot in the inhibitory response, visibly destabilizing the training, as shown in Figure~\ref{fig:taue_ablation}.

During the $\tau_I$ ablation, a mild preference for $\tau_I\!=\!7$ on \texttt{Fashion-MNIST} is observed as shown in Table~\ref{tab:tau-sweep}. Although no catastrophic instability was observed during the $\tauI$ ablation, given $\tau_I\!>\!\alpha_I\,\tau_E/2\!=\!0.5$ ms, which satisfies the stability condition outlined in Appendix~\ref{app:cont:linear}, a sudden performance degradation was observed for both \texttt{MNIST} and \texttt{Fashion-MNIST} when $\tau_I\!=\!5$ as illustrated by Figure~\ref{fig:taui_ablation}.

\subsection{Iteration-depth Analysis}
\label{app:add:niter}

An ablation analysis of the internal computation steps $T$ is performed while all other parameters are kept fixed (cf. Table~\ref{tab:niter-sweep}). Furthermore, we provide a wall-clock comparison for the given epoch based on the estimated time per classification, with varying duration $T$, which directly affects both training duration and the inference phase. Moreover, as illustrated by Figure~\ref{fig:iter_ablation}, longer internal computation steps do not guarantee higher performance, which was observed for both \texttt{MNIST} and \texttt{Fashion-MNIST}.

\begin{table}[!b]
\centering
\caption{Internal computation steps, $T$, ablation. Baseline $T = 50$ is in
bold. The Time factor column reports
training wall-clock relative to the baseline.}
\label{tab:niter-sweep}
\footnotesize
\setlength{\tabcolsep}{8.5pt}
\renewcommand{\arraystretch}{1.05}
\begin{tabular}{@{}lccccc@{}}
\toprule
& \multicolumn{2}{c}{\textbf{\texttt{MNIST}}}
& \multicolumn{2}{c}{\textbf{\texttt{Fashion-MNIST}}}
& \\
\cmidrule(lr){2-3} \cmidrule(lr){4-5}
$T$ & \textbf{best} & \textbf{final} & \textbf{best} & \textbf{final}
& \textbf{Time factor} \\
\midrule
10           & 99.58          & 99.54          & 94.17          & 94.00          & $\times\,0.24$ \\
25           & 99.58          & 99.51          & 93.90          & 93.86          & $\times\,0.52$ \\
\textbf{50}  & \textbf{99.60} & \textbf{99.57} & \textbf{93.86} & \textbf{93.78} & $\times\,1.00$ \\
75           & 99.54          & 98.11          & 93.86          & 70.42          & $\times\,1.51$ \\
100          & 98.77          & 97.34          & 92.05          & 24.45          & $\times\,2.02$ \\
\bottomrule
\end{tabular}
\end{table}

\begin{figure}[!t]
  \centering
  \includegraphics[clip, trim=0.5cm 0.25cm 0.5cm 0.5cm, width=0.85\textwidth]{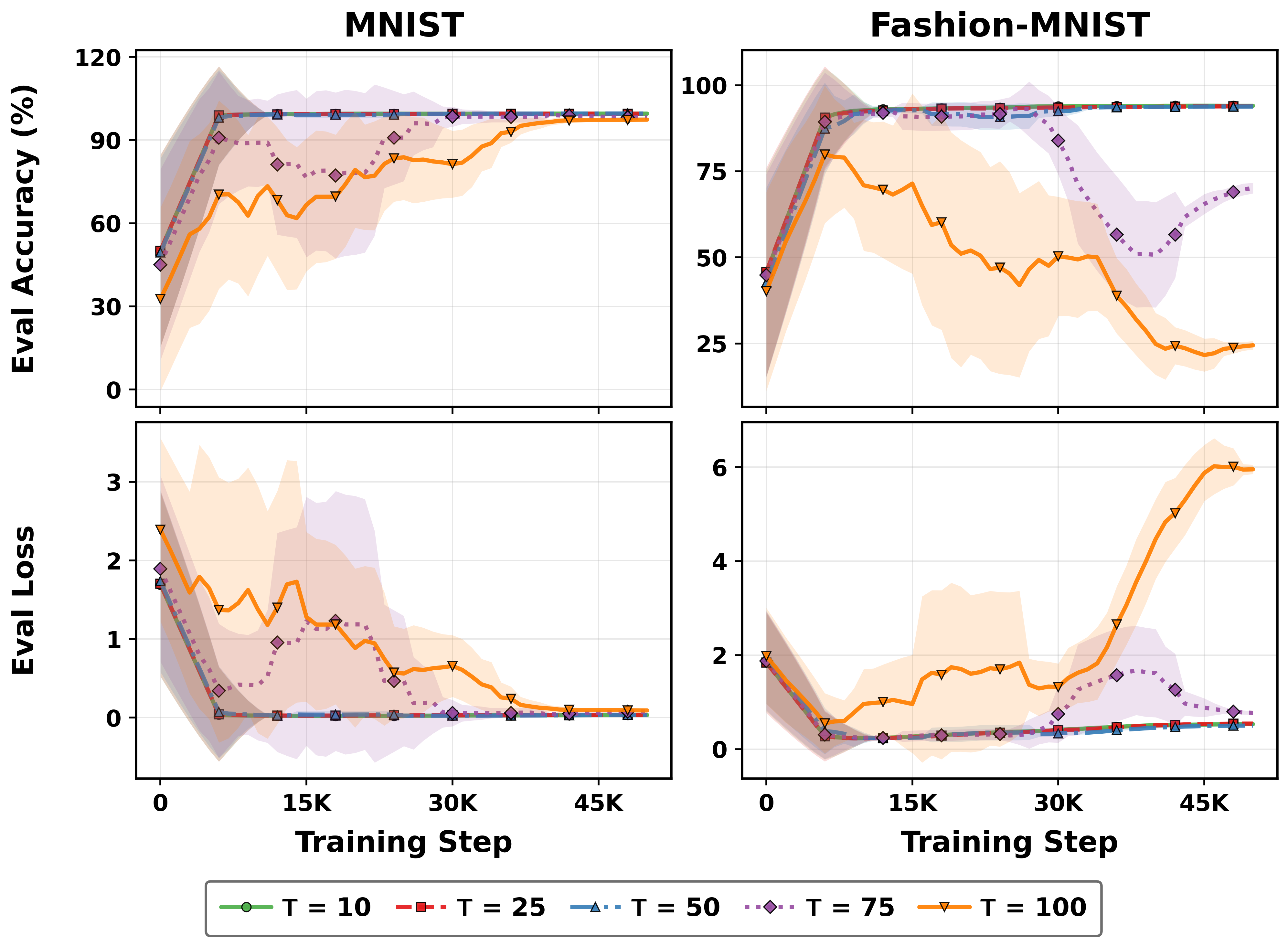}
  \caption{Evaluation curves during ablation study to analyze computation step $T$ and its effect on \texttt{MNIST} and \texttt{Fashion-MNIST}.}
  \label{fig:iter_ablation}
\end{figure}

It has been shown that the best accuracy is unaffected when $T \in [10,\,50]$ for both datasets, and the accuracy deviation is $\!<\!0.31$\% within this range. As mentioned, for $T\!=\!100$, performance degrades substantially, especially, for \texttt{Fashion-MNIST}, where the final accuracy collapses entirely to $24.45\%$, as shown in Figure~\ref{fig:iter_ablation}. We attribute this instability and collapse to two major compounding effects: i) The \gls{nlm} memory length limit, $M\!=\!25$, which might not be ideal for longer $T$ due to loss of information. ii) The \gls{bptt} horizon $K\!=\!0$ of Algorithm~\ref{alg:tide} for smaller image datasets such as \texttt{MNIST} and \texttt{Fashion-MNIST}, which increases the accumulated gradient variance and further destabilizes the optimization given longer $T$. Therefore, as with \texttt{ImageNet-1K}, either truncated-\gls{bptt} with $K \le T/4$ or jointly scaling $M$ is required to ensure stability.
 
\subsection{Game-theoretic Loss Weight Analysis}
\label{app:add:game}

\begin{table}[!b]
\centering
\caption{Game-theoretic loss weight $\lambda_{\mathrm{game}}$ ablation.
Baseline $\lambda_{\mathrm{game}} = 10^{-3}$ is in bold. The final
\gls{ei} activity ratio $\rho_{EI}$ with target $4.0$ is reported in italics
for interpretability.}
\label{tab:game-sweep}
\footnotesize
\setlength{\tabcolsep}{8.5pt}
\renewcommand{\arraystretch}{1.05}
\begin{tabular}{@{}lcccccc@{}}
\toprule
& \multicolumn{3}{c}{\textbf{\texttt{MNIST}}}
& \multicolumn{3}{c@{}}{\textbf{\texttt{Fashion-MNIST}}} \\
\cmidrule(lr){2-4} \cmidrule(lr){5-7}
$\lambda_{\mathrm{game}}$
 & \textbf{best} & \textbf{final} & $\boldsymbol{\rho_{EI}}$
 & \textbf{best} & \textbf{final} & $\boldsymbol{\rho_{EI}}$ \\
\midrule
$0$                & 99.64          & 99.60          & \emph{4.00}          & 93.75          & 93.30          & \emph{3.99} \\
$\mathbf{10^{-3}}$ & \textbf{99.59} & \textbf{99.53} & \emph{\textbf{3.99}} & \textbf{93.61} & \textbf{93.46} & \emph{\textbf{3.98}} \\
$10^{-2}$          & 99.64          & 99.59          & \emph{0.01}          & 94.15          & 93.95          & \emph{4.00} \\
$10^{-1}$          & 99.62          & 99.61          & \emph{0.01}          & 94.03          & 93.91          & \emph{3.99} \\
\bottomrule
\end{tabular}
\end{table}

\begin{figure}[!t]
  \centering
  \includegraphics[clip, trim=0.5cm 0.25cm 0.5cm 0.5cm, width=0.85\textwidth]{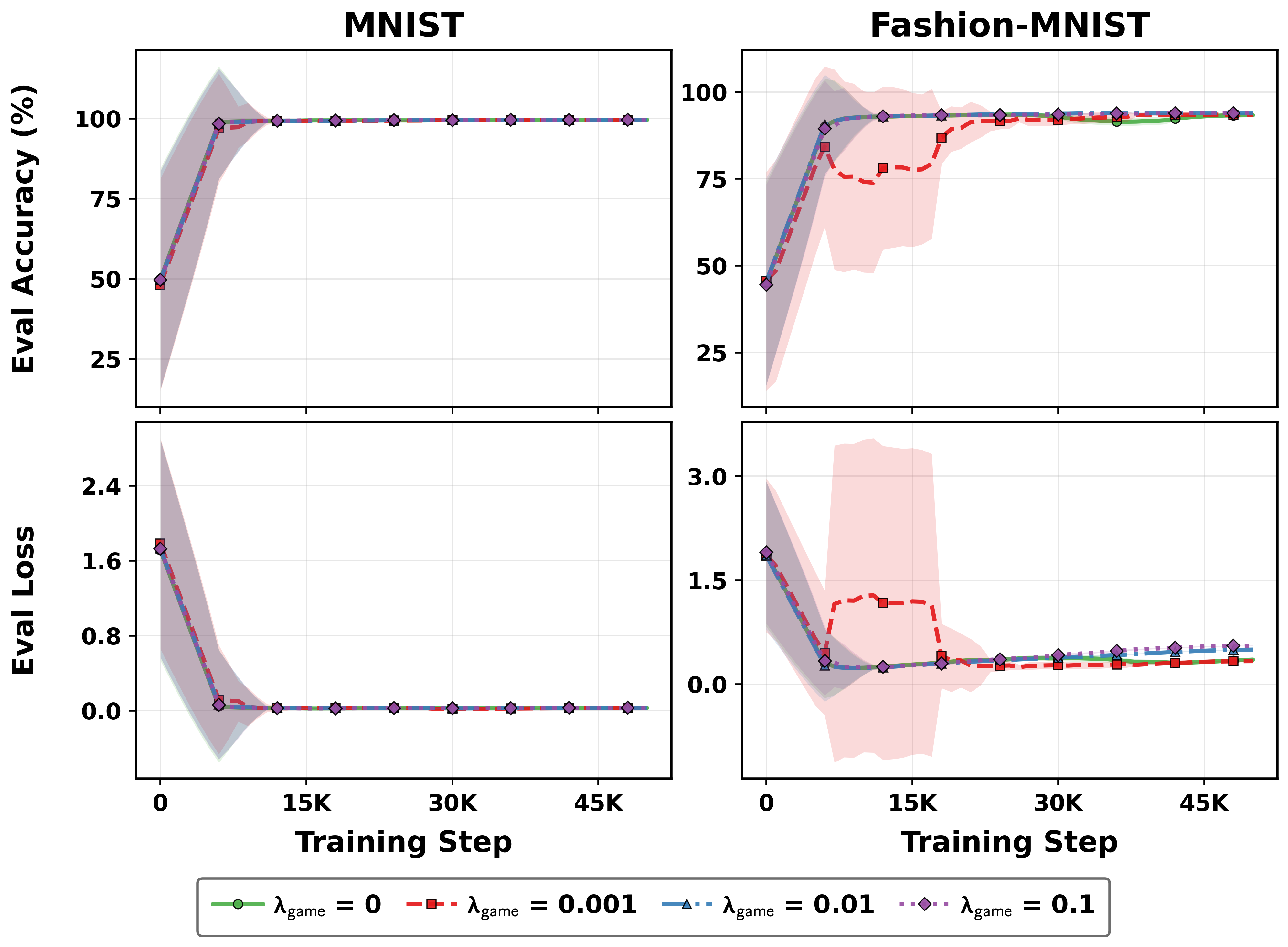}
  \caption{Evaluation curves during ablation study to analyze the impact of $\lambda_{\mathrm{game}}$ on \texttt{MNIST} and \texttt{Fashion-MNIST}.}
  \label{fig:game_ablation}
\end{figure}

The auxiliary game-theoretic weight $\lambda_{\mathrm{game}}$ of \eqref{eq:game-loss} is further studied as shown in Table~\ref{tab:game-sweep}. Figure~\ref{fig:game_ablation} illustrates that the task accuracy remains performant across all four configurations \texttt{MNIST}; however, significant \gls{ei} population collapse is observed at $\lambda_{\mathrm{game}}\!\in\!\{10^{-2},\,10^{-1}\}$ to $\rho_{EI}\!=\!0.01$, indicating complete deactivation of inhibitory pathways. Therefore, the game-theoretic loss, $\mathcal L_{\mathrm{game}}$, must be weighted comparably to the task loss, $\mathcal L_{\mathrm{task}}$, to prevent the regularizer from overwhelming the task loss and fully silencing the inhibitory population. 

Due to the inherent preference for $\mathcal L_{\mathrm{task}}$, the best and final \texttt{top-1} accuracy is not directly affected for either \texttt{MNIST} or \texttt{Fashion-MNIST}. Moreover, as shown in Table~\ref{tab:game-sweep}, the collapse of $\rho_{EI}$ is not only dependent on the magnitude of $\mathcal L_{\mathrm{game}}$ but also dataset dependent, as no significant changes were observed for \texttt{Fashion-MNIST}. Note that due to the neuron-inspired assumption of \gls{tide}, $\lambda_{\mathrm{game}}$ must be $\le 10^{-3}$ for stability of \gls{ei} circuit.
 
\subsection{Lateral Inhibition Iterations}
\label{app:add:wta}
The lateral inhibition circuit and its internal iterations are investigated to analyze the effect of $K_{\mathrm{WTA}}$ on the \texttt{top-1} accuracy and stability of the \gls{tide}, while obtaining the optimal value for $K_{\mathrm{WTA}}$ as shown in Table~\ref{tab:wta-sweep}.

\begin{table}[!b]
\centering
\caption{Lateral-inhibition iteration count $K_{\mathrm{WTA}}$ ablation
study. Baseline $K_{\mathrm{WTA}} = 5$ is in bold. $K_{\mathrm{WTA}} = 0$ (no lateral inhibition) is omitted because the
corresponding runs are provided in Appendix~\ref{app:add:loo}.}
\label{tab:wta-sweep}
\footnotesize
\setlength{\tabcolsep}{8.5pt}
\renewcommand{\arraystretch}{1.05}
\begin{tabular}{@{}lcccc@{}}
\toprule
& \multicolumn{2}{c}{\textbf{\texttt{MNIST}}}
& \multicolumn{2}{c@{}}{\textbf{\texttt{Fashion-MNIST}}} \\
\cmidrule(lr){2-3} \cmidrule(lr){4-5}
$K_{\mathrm{WTA}}$
 & \textbf{best} & \textbf{final}
 & \textbf{best} & \textbf{final} \\
\midrule
1            & 99.53          & 99.46          & 94.10          & 94.10 \\
3            & 99.58          & 99.49          & 93.67          & 93.64 \\
\textbf{5}   & \textbf{99.42} & \textbf{98.99} & \textbf{93.06} & \textbf{92.54} \\
10           & 99.49          & 99.47          & 94.20          & 94.07 \\
\bottomrule
\end{tabular}
\end{table}

\begin{figure}[!t]
  \centering
  \includegraphics[clip, trim=0.5cm 0.25cm 0.5cm 0.5cm, width=0.85\textwidth]{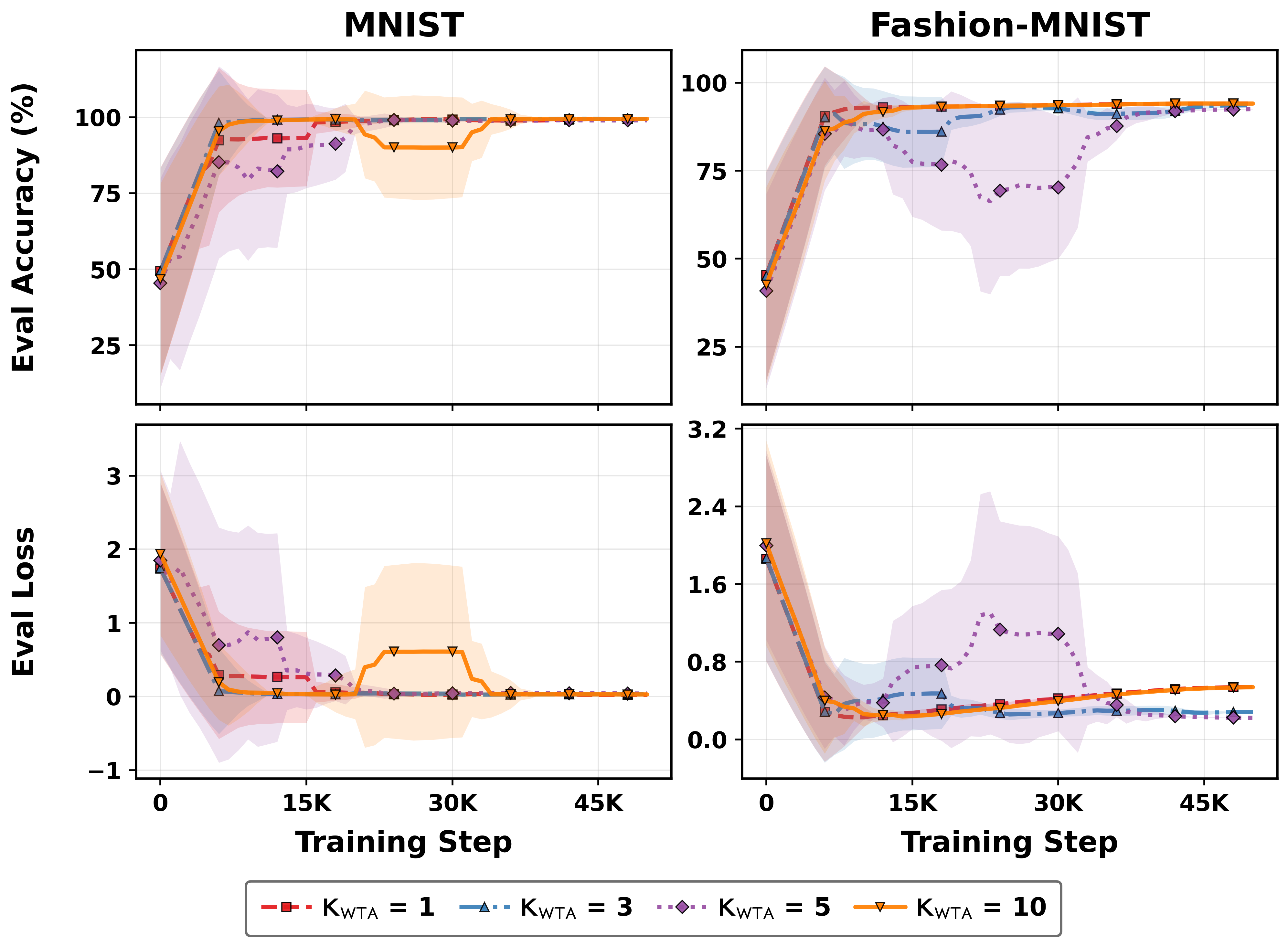}
  \caption{Evaluation curves during ablation study to analyze the impact of $K_{\mathrm{WTA}}$ on \texttt{MNIST} and \texttt{Fashion-MNIST}.}
  \label{fig:wta_ablation}
\end{figure}

As shown in Figure~\ref{fig:wta_ablation}, the default $K_{\mathrm{WTA}}\!=\!5$ does not perform well on both \texttt{MNIST} and \texttt{Fashion-MNIST}. Furthermore, the best performance is achieved when $K_{\mathrm{WTA}}\!=\!\{3, 10\}$; thus improving the \texttt{top-1} accuracy by $+1.14$\% and $+0.02$\% for \texttt{Fashion-MNIST} and \texttt{MNIST}, respectively.

Given that the lateral inhibition is based on $\mathrm{RELU}$ without a reliance on $\mathrm{Softmax}$, while $W_{EI}^{\mathrm{lat}}$ is initialized to a uniform distribution bounded by $1/\sqrt{n_E}$ for the given training length of $50$\,K steps, the dynamics typically converge to $K_{\mathrm{WTA}}\!\le\!3$. The early termination signal is typically triggered within three steps; thus, increasing $K_{\mathrm{WTA}}$ to higher values without decreasing the termination tolerance will only result in longer, idle iterations during inference.

The resulting failure mode given $K_{\mathrm{WTA}}\!=\!5$ is mainly due to the location of the baseline in a local minimum; however, $K_{\mathrm{WTA}}\!=\!5$ is retained given further experimentation (cf. Table~\ref{tab:ablation-compact}) has demonstrated that this value can be used across all datasets without performance degradation.
Consequently, we used $K_{\mathrm{WTA}}\!=\!5$ across all our benchmarks (cf. Table~\ref{tab:hparams-shared}).

\subsection{Training-stability Diagnostics}
\label{app:add:stability}

Given the stability issues in \gls{ei} circuits highlighted in \cite{betteti2025competition}, further analysis was conducted to investigate training stability, potential failure modes, and the approaches needed to achieve stable results. 

\begin{figure}[!t]
\centering
\includegraphics[width=0.53\textwidth]{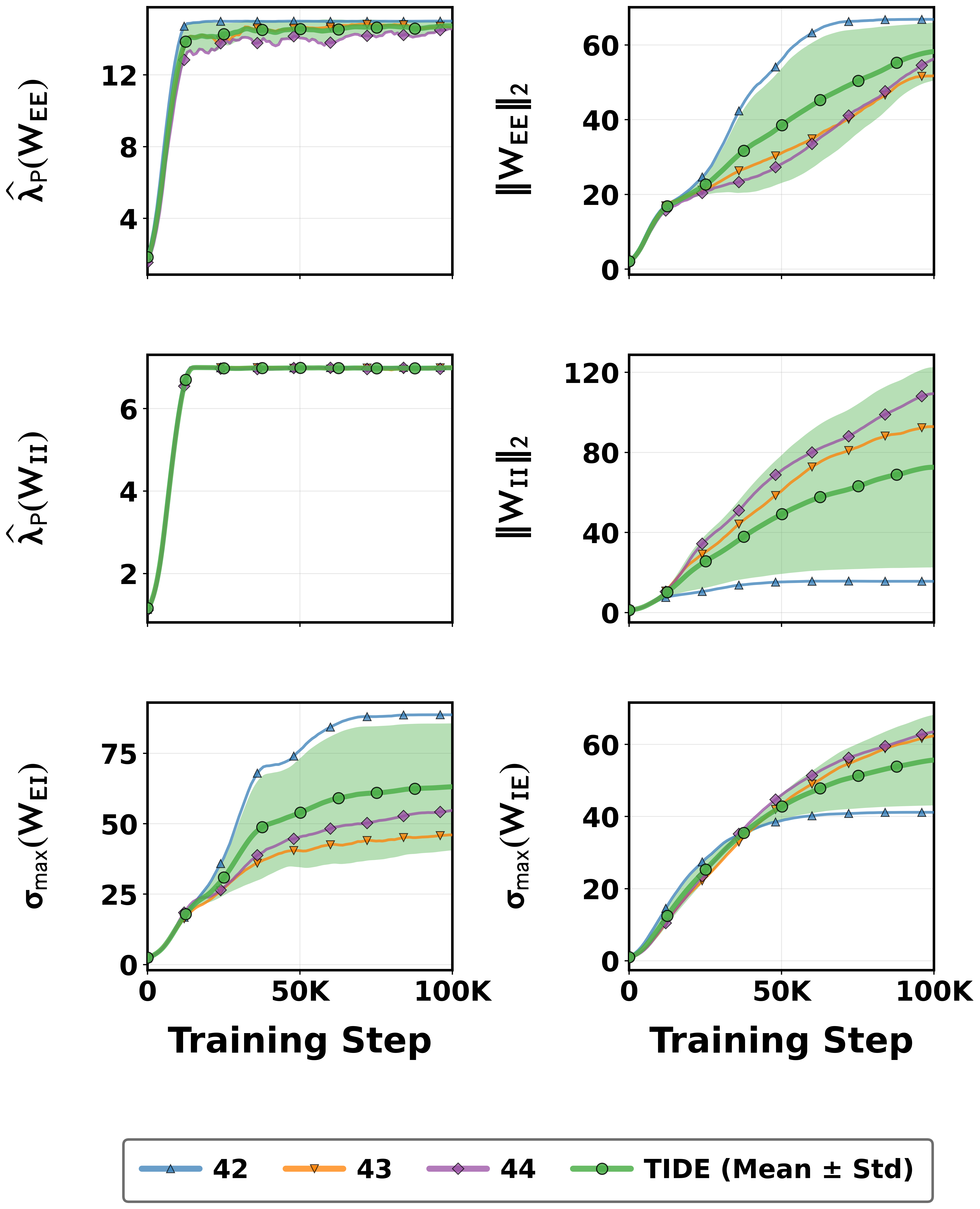}
\caption{Spectral values during $100$\,K multi-seeded \texttt{ImageNet-1K} training, Rows~1--2: Perron estimates $\widehat\lambda_{\mathrm P}$ (\textit{left}) and spectral norm $\|\cdot\|_2$ (\textit{right}) for $W_{EE}$ and $W_{II}$. Row~3: cross-population $\sigma_{\max}(W_{EI})$ and $\sigma_{\max}(W_{IE})$.}
\label{fig:imagenet-spectral}
\end{figure}

\begin{figure}[!b]
\centering
\includegraphics[width=0.53\textwidth]{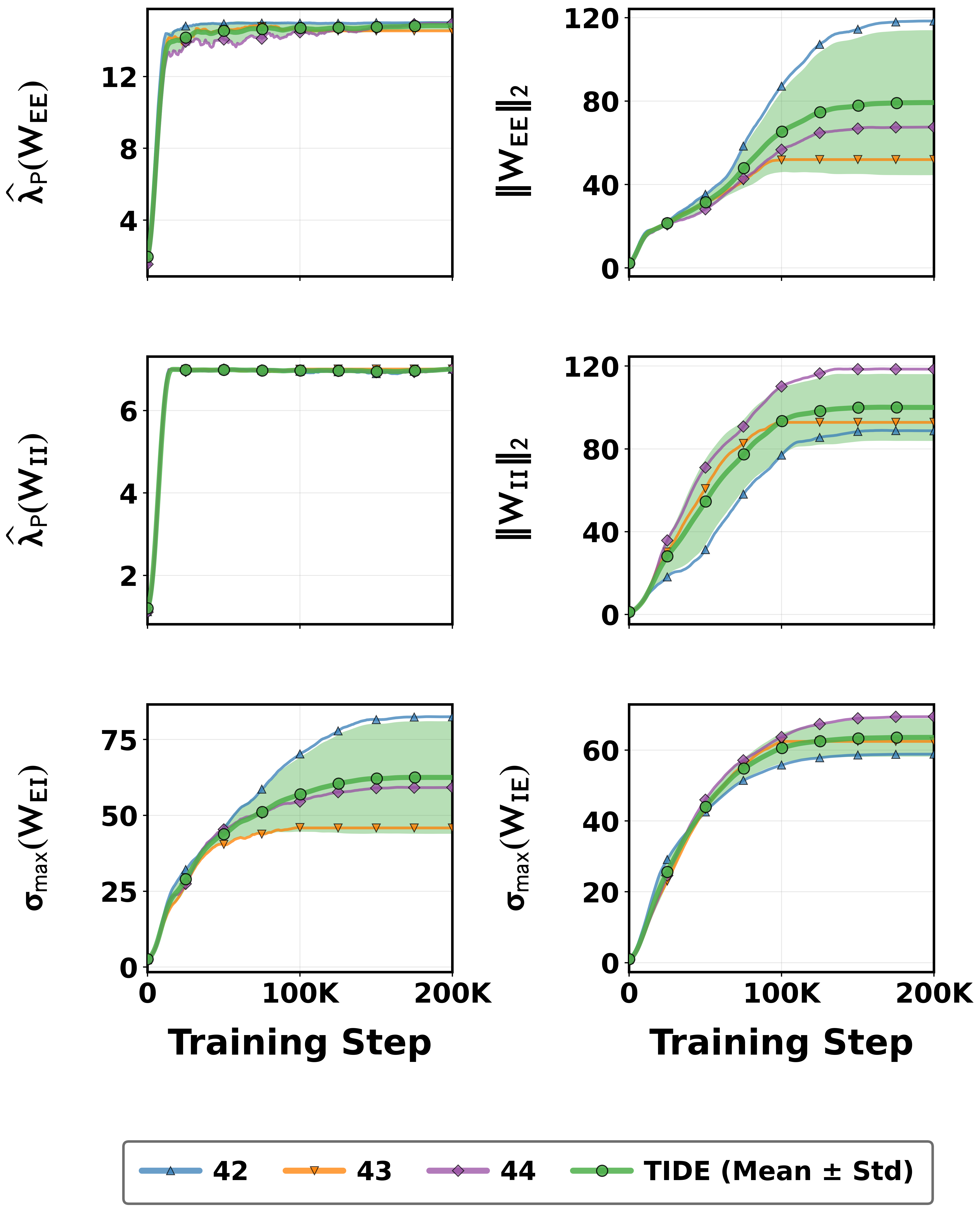}
\caption{Spectral values during $200$\,K multi-seeded \texttt{ImageNet-1K} training, Rows~1--2: Perron estimates $\widehat\lambda_{\mathrm P}$ (\textit{left}) and spectral norm $\|\cdot\|_2$ (\textit{right}) for $W_{EE}$ and $W_{II}$. Row~3: cross-population $\sigma_{\max}(W_{EI})$ and $\sigma_{\max}(W_{IE})$.}
\label{fig:imagenet-spectral2}
\end{figure}

As illustrated in Figures~\ref{fig:imagenet-spectral} and~\ref{fig:imagenet-spectral2}, the Perron radii converged to desired value without overshooting for both $100$\,K and $200$\,K multi-seeded \texttt{ImageNet-1K} training, where $\widehat\lambda_{\mathrm P}(W_{EE}) = 14.69 \pm 0.27$ and $\widehat\lambda_{\mathrm P}(W_{II}) = 6.998 \pm 0.006$, since the $\mathrm{ReLU}(\widehat\lambda - \tau)^2$ results in a hinge equilibrates for the penalty gradient against the task loss gradient at the boundary, $\tau_{EE} = 15$ and $\tau_{II} = 7$ are achieved. The soft-hinge gradient remains active, providing non-vanishing pressure toward the \gls{lds}-compatible region defined in \ref{def:lds}. Furthermore, the regularized spectral norms grow throughout training while exceeding the Perron clamps at $\|W_{EE}\|_2 = 58.6 \pm 7.6$, $\|W_{II}\|_2 = 72.7 \pm 50.1$, $\sigma_{\max}(W_{EI}) = 63.2 \pm 22.5$, and $\sigma_{\max}(W_{IE}) = 55.8 \pm 12.7$, where $\|W\|_2 = \sigma_{\max}(W)$ denotes the largest singular value of $W$, i.e., the spectral norm characterizing the worst-case amplification of any input vector by $W$. However, this growth does not violate the \gls{lds} condition of Corollary~\ref{cor:lds-test}, so local asymptotic stability of the Wilson-Cowan recurrence (Theorem~\ref{thm:lds}) is preserved.
Together with empirical observations from multi-seeded training across various datasets, the stability mechanisms described in Appendix~\ref{app:cont} are effective at providing stable learning in asymmetric architectures.

\subsection{Architectural Component Analysis}
\label{app:add:loo}
To analyze the contribution of each component to the overall performance of \gls{tide}, we perform a leave-one-out ablation, as shown in Table~\ref{tab:leave}. This study shows the direct compounded effect of biologically-motivated constraints: removing any component lowers architecture's performance, and the highest \texttt{top-1} accuracy is achieved only when all components are included. 

\paragraph{Setup:} All architectural variants are trained and evaluated under identical training budget and hyperparameter settings as indicated in Tables~\ref{tab:hparams} and~\ref{tab:hparams-shared}. Furthermore, when investigating the spectral regularizer loss and its effect on the stability of the final \texttt{top-1} accuracy for the remaining datasets, $\lambda_{spec}$ is set to $0.1$ (cf. Table~\ref{tab:leave}). 

\paragraph{Spectral regularizer:} With the additional stabilization from the regularizer, \gls{tide} achieves overall better performance and its \texttt{top-1} accuracy on both \texttt{MNIST} and \texttt{Fashion-MNIST} increases as shown in Table~\ref{tab:leave}. However, the stabilization observed when comparing Figure~\ref{fig:loo-spectral} with Figure~\ref{fig:eval_benchmark} comes at the cost of lower \texttt{top-1} accuracy on both \texttt{CIFAR-10} and \texttt{CIFAR-100}.

\begin{figure}[!t]
\centering
\includegraphics[clip, trim=0.5cm 0.250cm 0.5cm 0.5cm, width=\textwidth]{figures/final/leave_one_out/ctm_vs_tide_loo_add_spectral.png}
\caption{Comparison between \gls{tide} with its spectral regularization and \gls{ctm}. Both \texttt{top-1} accuracy (\%), and \texttt{mean\,$\pm$\,std} total losses across all the trained seeds are provided.}
\label{fig:loo-spectral}
\end{figure}

\subsection{Backbone Analysis \& Robustness}
\label{app:add:backbone}
 
To separate the contribution of \gls{hrf} from that of \gls{tide}'s recurrent dynamics, each backbone used in this study is trained together with a linear classifier (cf. Table~\ref{tab:backbone-linear}) to analyze their performance on unperturbed datasets. Additional robustness analysis are performed on all the variants under several perturbation severities by using \texttt{CIFAR-10-C} and \texttt{CIFAR-100-C}~\cite{hendrycks2019robustness} (Figures~\ref{fig:robustness_cifar10} and~\ref{fig:robustness_cifar100}) to further investigate the robustness gains against the conventional ResNet-18 backbone.
 
\paragraph{Setup:} Each linear classifier uses the averaged $64$ value tokens from its backbone and applies a single linear layer trained via cross-entropy loss for $50$\,K steps. To compare \gls{tide}'s recurrent dynamics against the linear classifier, we train all the models under an identical training budget. Note that all the following studies are performed on perturbed datasets, and we use only models that are solely trained on unperturbed data to investigate generalization and \gls{ood} robustness.  

\paragraph{CIFAR-10-C \& CIFAR-100-C:} \texttt{CIFAR-10-C} and \texttt{CIFAR-100-C} consist of the $10$\,K test images of \texttt{CIFAR-10} and \texttt{CIFAR-100}~\cite{krizhevsky2009cifar}, respectively, at $32{\times}32$ resolution with $10$ and $100$ classes, corrupted with the $15$ standard corruption similar to \texttt{ImageNet-C}~\cite{hendrycks2019robustness}. As shown in Figure~\ref{fig:robustness_cifar10}, \gls{tide} outperforms shallow \gls{hrf} by an average of $+4.28\%$ across all perturbations. Similarly,  Figure~\ref{fig:robustness_cifar100} illustrates that \gls{tide} can outperform its counter part by an average of $+4.99\%$ across \texttt{CIFAR-100-C}. 
 
\begin{figure}[!t]
  \centering
  \includegraphics[clip, trim=2.9cm 0.25cm 0.75cm 0.74cm, width=\textwidth]{figures/final/hrf/cifar10_c_heatmap_linear.png}
  \caption{Robustness analysis using CIFAR-10-C~\cite{hendrycks2019robustness} for corrupted images with perturbations. Top: \textit{Left} panel presents results for shallow \gls{hrf} backbone with linear classifier, while the \textit{center} panel reports the \texttt{top-1} accuracy under perturbations of the linear classifier with ResNet-18 backbone, and \textit{right} panel presents the results for the deep \gls{hrf} backbone. Bottom: \textit{Left} panel presents results for \gls{tide}, and \textit{right} panel reports the differences between \gls{tide} and the linear classifier with the shallow \gls{hrf} backbone.
  }
  \label{fig:robustness_cifar10}
\end{figure}

\begin{figure}[!t]
  \centering
  \includegraphics[clip, trim=2.9cm 0.25cm 0.75cm 0.74cm, width=\textwidth]{figures/final/hrf/cifar100_c_heatmap_linear.png}
  \caption{Robustness analysis using CIFAR-100-C~\cite{hendrycks2019robustness} for corrupted images with perturbations. Top: \textit{Left} panel presents results for shallow \gls{hrf} backbone with linear classifier, while the \textit{center} panel reports the \texttt{top-1} accuracy under perturbations of the linear classifier with ResNet-18 backbone, and \textit{right} panel presents the results for the deep \gls{hrf} backbone. Bottom: \textit{Left} panel presents results for \gls{tide}, and \textit{right} panel reports the differences between \gls{tide} and the linear classifier with the shallow \gls{hrf} backbone.
  }
  \label{fig:robustness_cifar100}
\end{figure}

\begin{table}[!b]
 \centering
 \caption{Performance analysis of different backbones with the linear classifier on unperturbed datasets. The table provides the mean best and final \texttt{top-1} accuracy, $\mathrm{Acc@1}$ (\%), as the \texttt{mean\,$\pm$\,std} across the seeds for evaluation during training. Moreover, the highest \texttt{top-1} accuracy for each dataset is in bold.}
 \label{tab:backbone-linear}
 \footnotesize
 \setlength{\tabcolsep}{30pt}
 \renewcommand{\arraystretch}{1.05}
 \newcommand{\sdfont}{\fontsize{8}{7}\selectfont}
 \newcommand{\ms}[2]{#1\textsubscript{\sdfont\textpm#2}}

 \begin{tabular}{@{}l@{\hskip 6pt}c@{\hskip 4pt}c@{\hskip 6pt}c@{\hskip 4pt}c@{\hskip 6pt}c@{\hskip 4pt}c@{\hskip 6pt}c@{\hskip 4pt}c@{}}
 \toprule

  & \multicolumn{2}{@{\hskip .5\tabcolsep}c@{\hskip .5\tabcolsep\hskip 6pt}}{\textbf{MNIST}}
 & \multicolumn{2}{@{\hskip .5\tabcolsep}c@{\hskip .5\tabcolsep\hskip 6pt}}{\textbf{F-MNIST}}
 & \multicolumn{2}{@{\hskip .5\tabcolsep}c@{\hskip .5\tabcolsep\hskip 6pt}}{\textbf{CIFAR-10}}
 & \multicolumn{2}{@{\hskip .5\tabcolsep}c@{\hskip .5\tabcolsep}}{\textbf{CIFAR-100}} \\

 \cmidrule(lr){2-3}\cmidrule(lr){4-5}\cmidrule(lr){6-7}\cmidrule(lr){8-9}
 \textbf{Configuration}
 & best & final & best & final & best & final & best & final \\
 \midrule

 Shallow HRF
 & \ms{\textbf{99.59}}{\textbf{0.01}} & \ms{99.54}{0.03}
 & \ms{\textbf{94.12}}{\textbf{0.06}} & \ms{\textbf{94.06}}{\textbf{0.10}}
 & \ms{\textbf{86.14}}{\textbf{0.15}} & \ms{\textbf{86.09}}{\textbf{0.08}}
 & \ms{\textbf{52.98}}{\textbf{0.24}} & \ms{\textbf{52.96}}{\textbf{0.22}} \\

 ResNet-18
 & \ms{99.52}{0.04} & \ms{99.48}{0.03}
 & \ms{93.52}{0.14} & \ms{93.42}{0.12}
 & \ms{82.45}{0.20} & \ms{82.39}{0.23}
 & \ms{50.60}{0.09} & \ms{50.49}{0.12} \\

 Deep HRF
 & \ms{\textbf{99.59}}{\textbf{0.02}} & \ms{\textbf{99.55}}{\textbf{0.04}}
 & \ms{92.44}{0.18} & \ms{92.39}{0.21}
 & \ms{85.52}{0.10} & \ms{85.45}{0.16}
 & \ms{51.46}{0.29} & \ms{51.35}{0.31} \\

 \bottomrule
 \end{tabular}
\end{table}

\subsection{Robustness \& OOD Evaluation}
\label{app:add:robust}

\begin{figure}[!t]
  \centering
  \includegraphics[clip, trim=2.9cm 0.25cm 0.75cm 0.74cm, width=\textwidth]{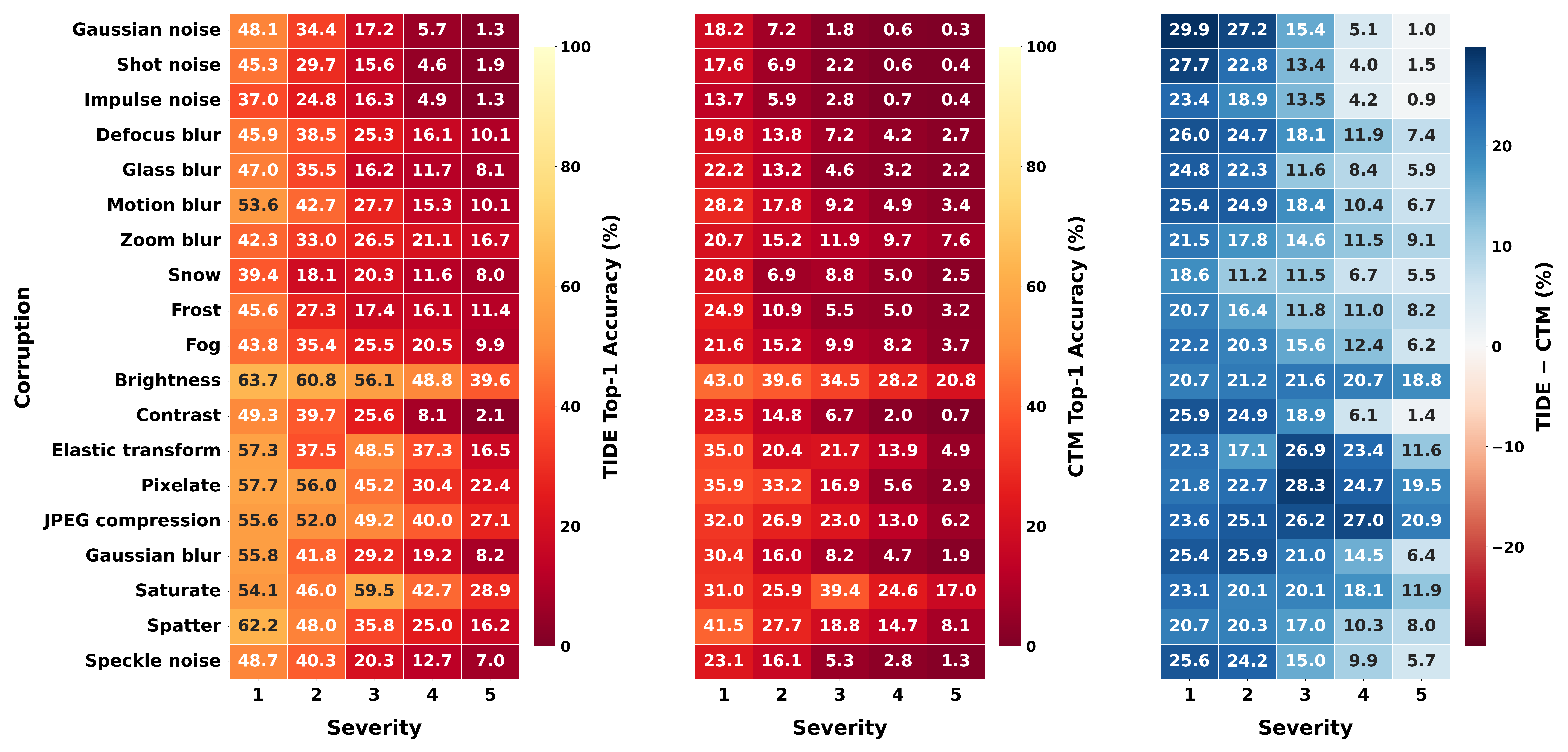}\\[3pt]
  \includegraphics[clip, trim=2.9cm 0.25cm 0.75cm 0.74cm, width=\textwidth]{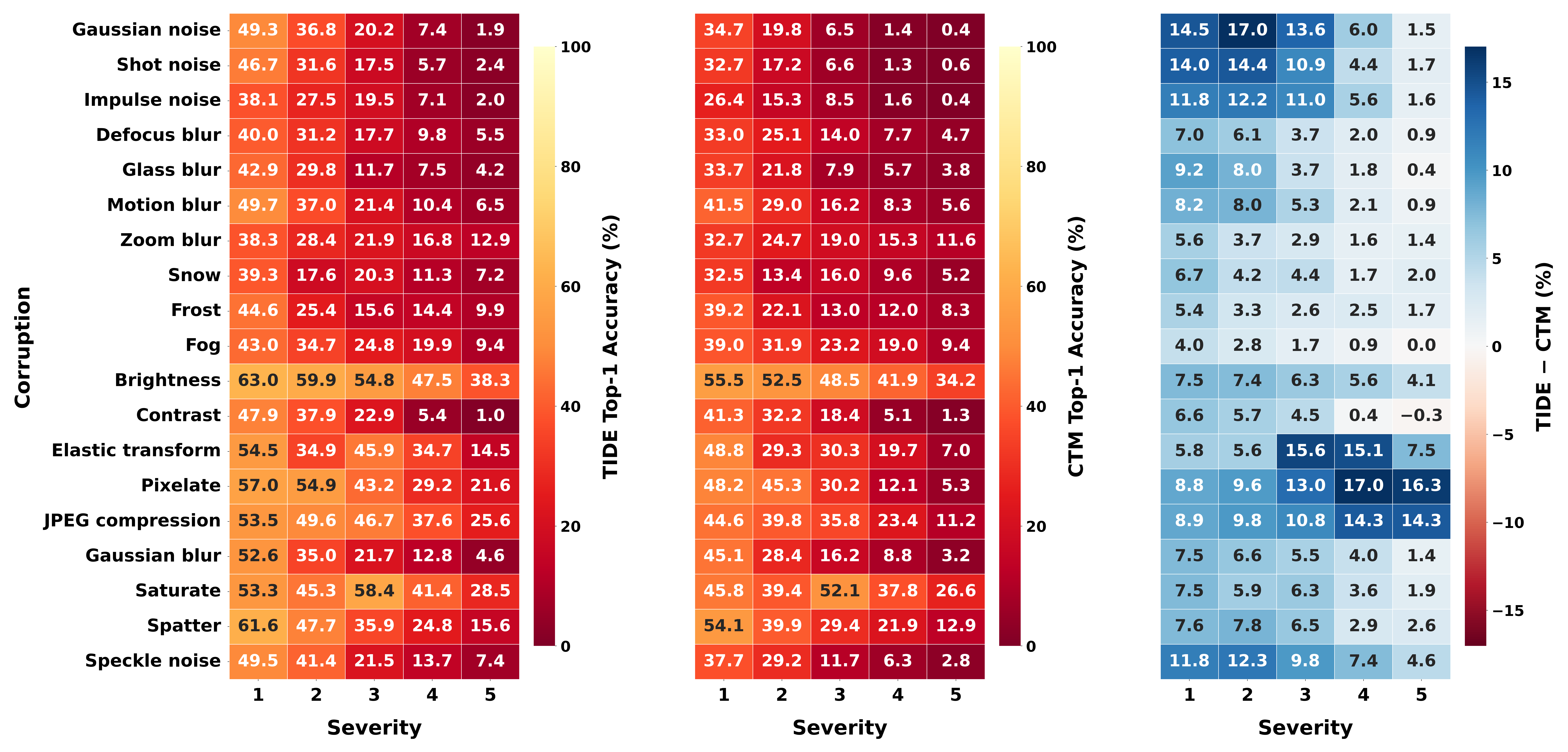}
  \caption{Robustness analysis using ImageNet-C~\cite{hendrycks2019robustness} for corrupted images with perturbations. \textit{Left} panel presents results for \gls{tide}, \textit{center} panel corresponds to \gls{ctm}, and \textit{right} panel reports differences.
  Top: $100$\,K training budget. Bottom: $200$\,K training budget.}
  \label{fig:ood-imagenet-c2}
\end{figure}
 
\begin{figure}[!t]
  \centering
  \includegraphics[width=0.82\textwidth]{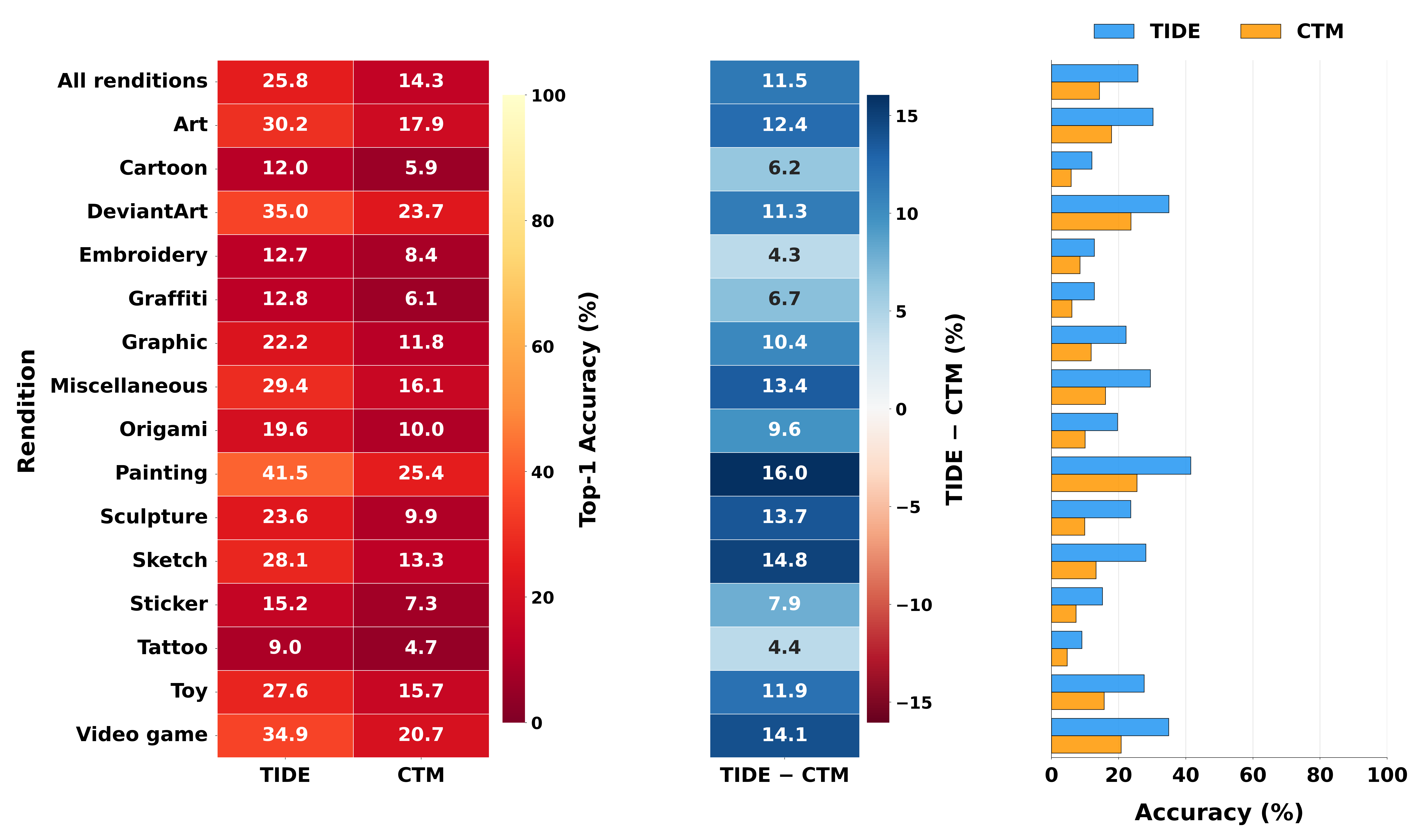}\\[2pt]
  \includegraphics[width=0.82\textwidth]{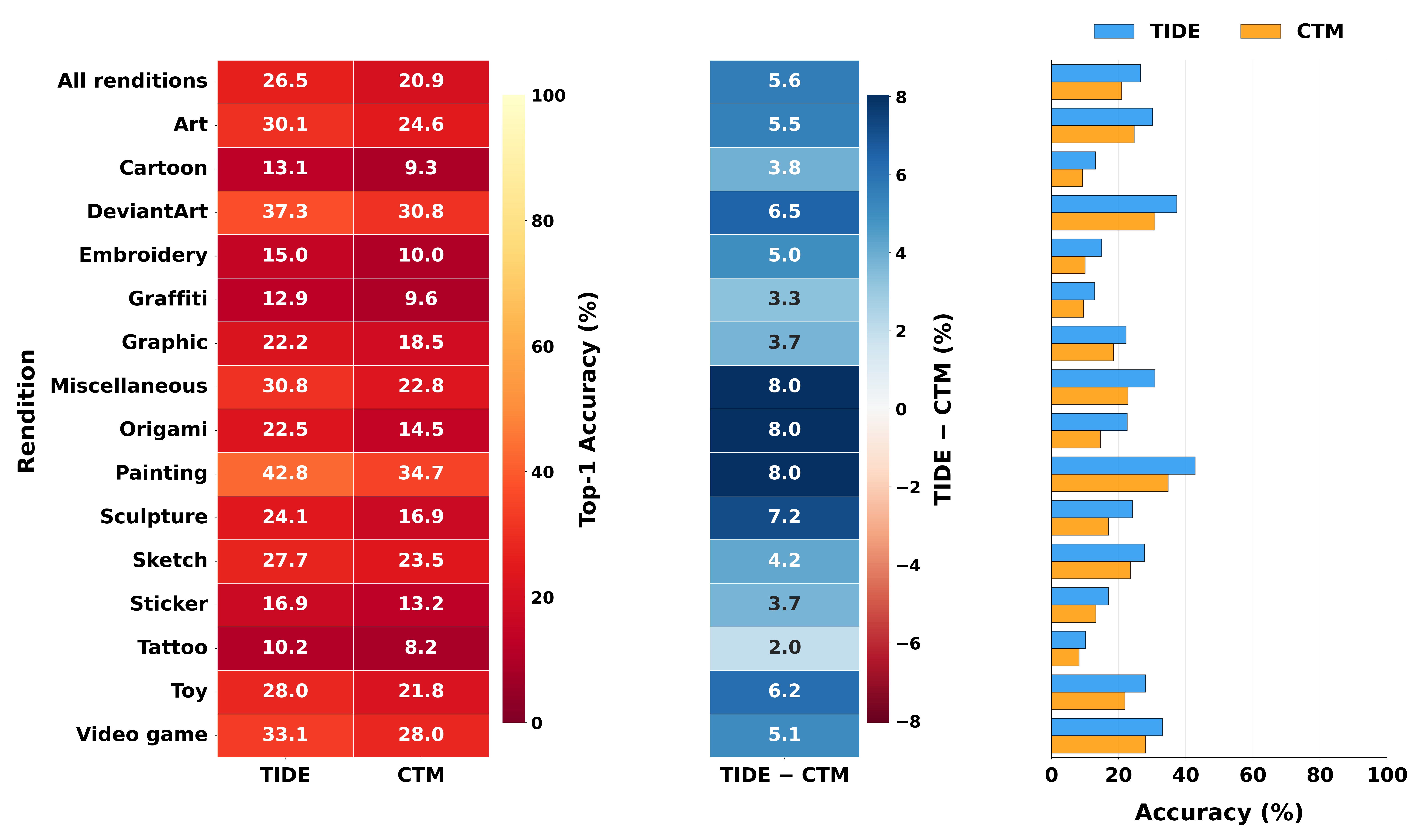}\\[2pt]
  \includegraphics[width=0.82\textwidth]{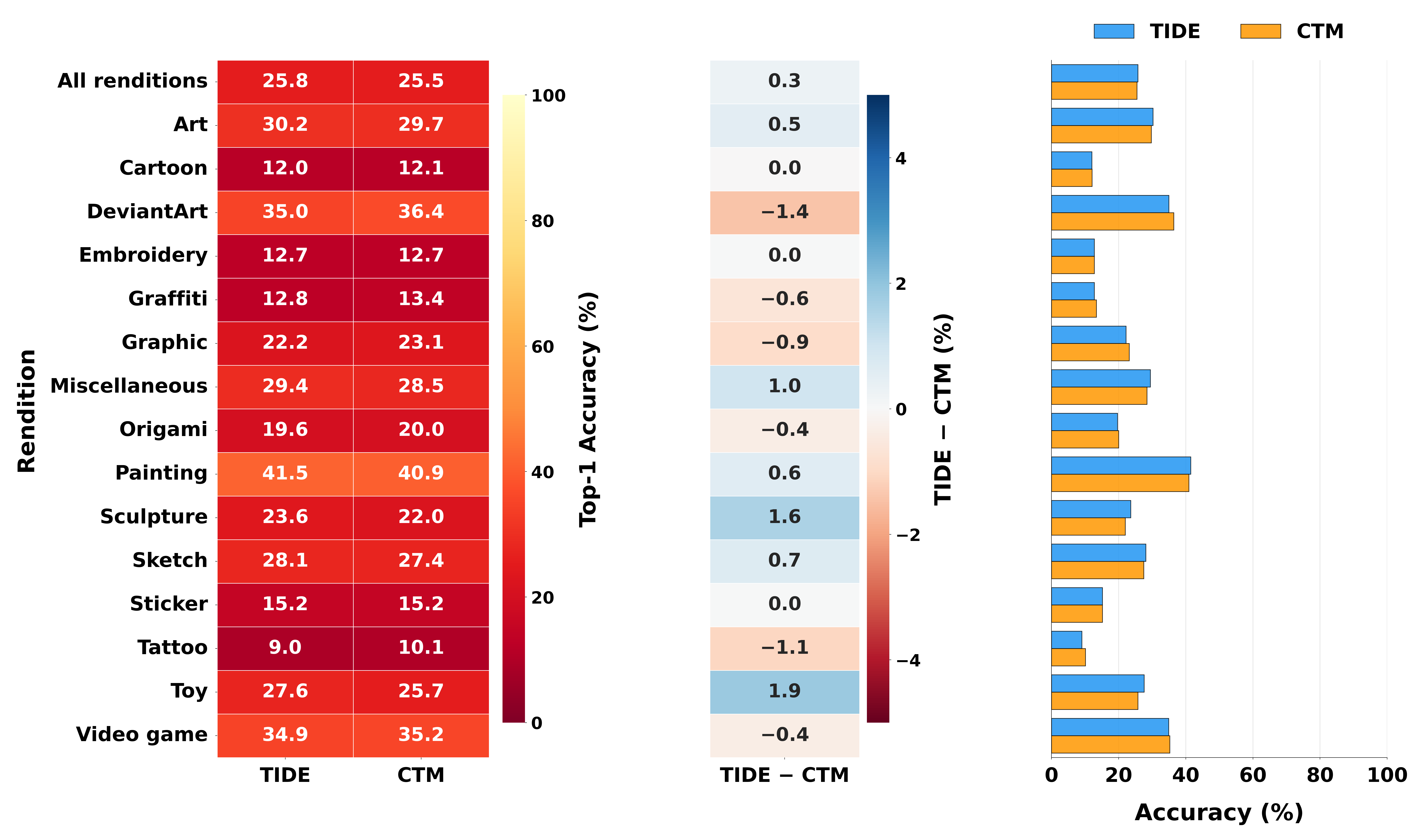}
  \caption{Robustness analysis on \texttt{ImageNet-R} \cite{hendrycks2021many} in the $100$\,K (top), $200$\,K (middle) and $500$\,K (bottom) comparisons. \textit{Left} panel presents \texttt{top-1} (\%) for \gls{tide} and \gls{ctm}, \textit{center} panel reports the differences and \textit{right} panel is the accuracy bar.}
  \label{fig:ood-imagenet-r2}
\end{figure}

\begin{figure}[!t]
  \centering
  \includegraphics[width=0.75\textwidth]{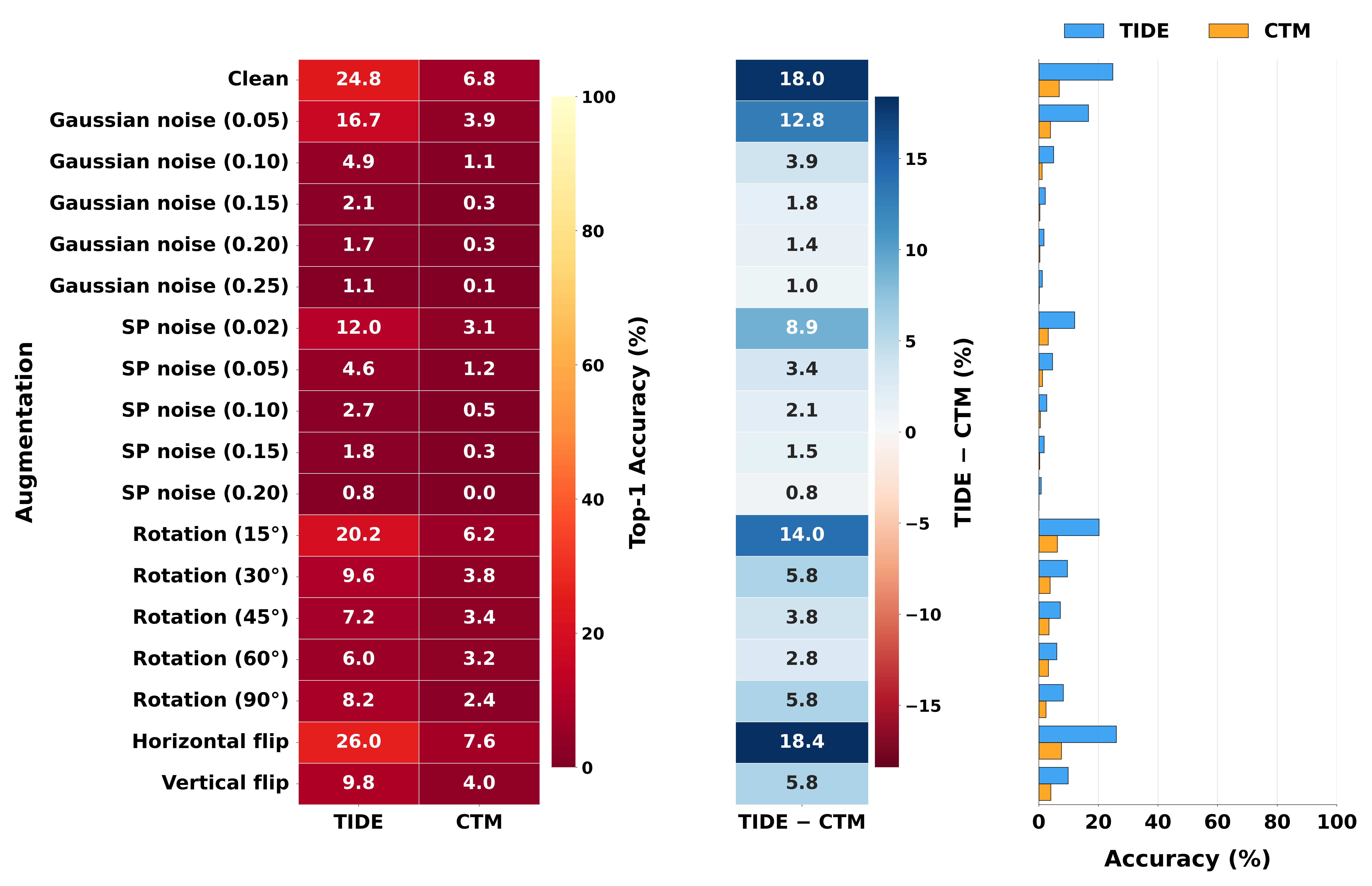}\\[2pt]
  \includegraphics[width=0.75\textwidth]{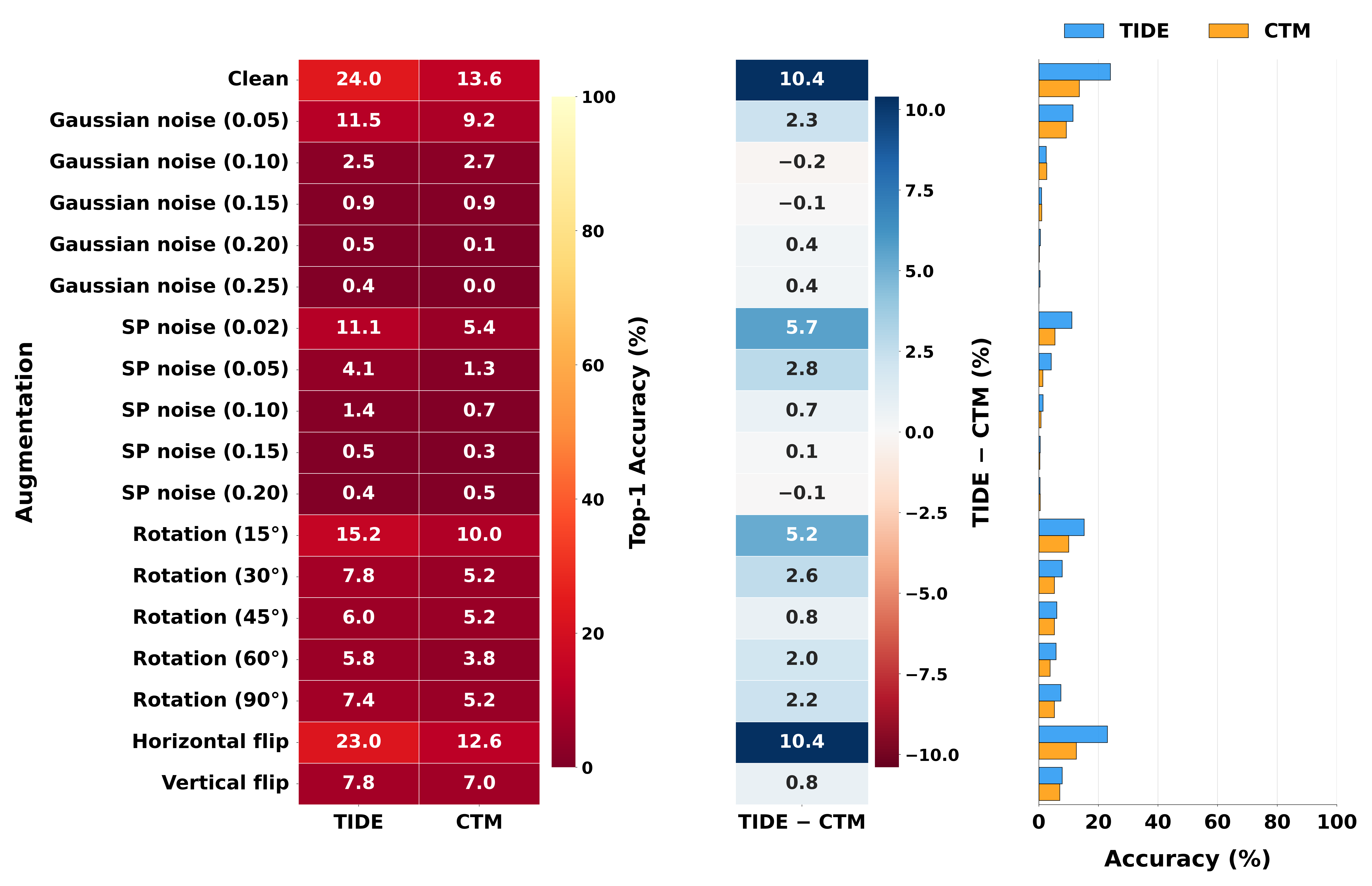}\\[2pt]
  \includegraphics[width=0.75\textwidth]{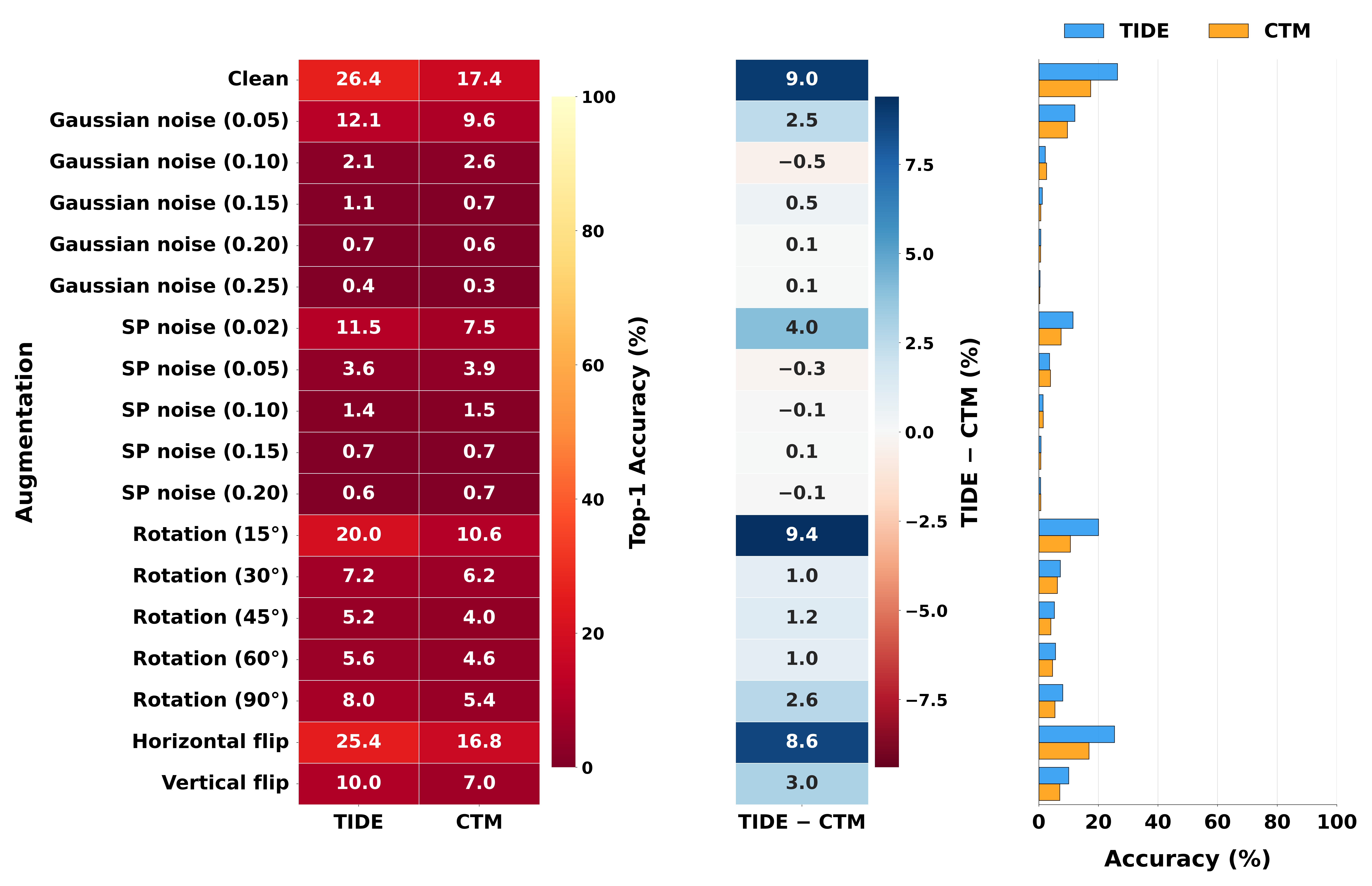}
  \caption{\gls{ood} analysis on \texttt{Tiny-ImageNet} \cite{le2015tiny} in the $100$\,K (top), $200$\,K (middle) and $500$\,K (bottom) comparisons. \textit{Left} panel present \texttt{top-1} (\%) for \gls{tide} and \gls{ctm}, \textit{center} panel reports the differences and \textit{right} panel is the accuracy bar.}
  \label{fig:ood-imagenet-tiny2}
\end{figure}

The previous experiments in Appendix~\ref{app:add:benchmarks} primarily report \texttt{top-1} accuracy on unperturbed data. To further investigate if leveraging neuro-inspired motifs provides measurable benefits for the learning outcomes beyond \gls{ctm}, we evaluate the robustness by investigating $100$\,K, $200$\,K and $500$\,K steps checkpoint for \gls{tide} trained solely on \texttt{ImageNet-1K} against the similarly trained \gls{ctm} checkpoints on three established \gls{ood} benchmarks: \texttt{ImageNet-C} \cite{hendrycks2019robustness}, \texttt{ImageNet-R} \cite{hendrycks2021many}, and \texttt{Tiny-ImageNet} \cite{le2015tiny} without retraining, fine-tuning, or test-time adaption or augmentation. Note that Section~\ref{sec:result} reports only the $500$\,K comparison for \texttt{ImageNet-C}.

\paragraph{Setup:} Both models are evaluated in inference mode with batch sizes $32$ under the standard \texttt{ImageNet-1K} preprocessing conditions stated in Appendix~\ref{app:exp:data}. The \texttt{top-1} best accuracy for both models across all training budget are selected for direct comparison and analysis.

\paragraph{ImageNet-C:} \texttt{ImageNet-C} consists of the $50$\,K validation images of \texttt{ImageNet-1K} modified under different perturbation severity levels~\cite{hendrycks2019robustness}. Following the ablation protocol from \cite{hendrycks2021many}, we evaluate \texttt{top-1} accuracy of \gls{tide} and \gls{ctm} as shown in Figure~\ref{fig:ood-imagenet-c2}. Overall, \gls{tide} has higher mean \texttt{top-1} accuracy than \gls{ctm} under all training budgets: $16.8\%$ at $100$\,K ($30.6\%$ against $13.9\%$), and $6.4\%$ at $200$\,K ($28.9\%$ against $22.5\%$).

\paragraph{ImageNet-R:} \texttt{ImageNet-R} consists of $30$\,K renditions with various styles and categories, including art, graphic, painting, drawn from a $200$-class subset of \texttt{ImageNet-1K}. Following the ablation protocol outlined by \cite{hendrycks2021many}, we evaluate \texttt{top-1} accuracy of \gls{tide} and \gls{ctm} as shown in Figure~\ref{fig:ood-imagenet-r2}. Overall, \gls{tide} has higher \texttt{top-1} accuracy than \gls{ctm} over all renditions under all training budgets: $11.5\%$ at $100$\,K ($25.75\%$ against $14.3\%$), $5.6\%$ at $200$\,K ($26.5\%$ against $20.9\%$), and $0.29\%$ at $500$\,K ($25.75\%$ against $25.46\%$). \gls{ctm} shows better robustness on tattoos and deviantart as illustrated in Figure~\ref{fig:ood-imagenet-r2}. Overall, \gls{tide} provides more robustness against \gls{ood} and shifted distributions via rendition.

\paragraph{Tiny-ImageNet:} \texttt{Tiny-ImageNet} consists of $10\,$K validation images at $64{\times}64$ resolution, evaluated against the $1000$-way \texttt{ImageNet-1K} softmax via the \texttt{TinyImageNet}$\to$\texttt{ImageNet-1K} mapping \cite{le2015tiny}. Both models are evaluated zero-shot on $17$ corruption augmentations applied post-upsampling, thereby enabling a comparison of architectural robustness against the bilinear $64{\to}224$ upsampling artifact under additional augmentations. Overall, \gls{tide} has higher mean \texttt{top-1} accuracy compared to \gls{ctm} over all augmentations for all training budgets: $5.5\%$ at $100$\,K ($8.0\%$ against $2.4\%$), $2.1\%$ at $200$\,K ($6.3\%$ against $4.1\%$) and $1.9\%$ at $500$\,K ($6.8\%$ against $4.9\%$). Moreover, as illustrated in Figure~\ref{fig:ood-imagenet-tiny2}, \gls{tide} outperforms \gls{ctm} on all augmentations, specifically horizontal flip and $15^o$ rotation. \gls{tide} retains $32.1\%$, $26.1\%$ and $25.8\%$ of its clean \texttt{top-1} accuracy against $35.8\%$, $30.3\%$ and $28.0\%$ for \gls{ctm}, where retention is defined as the mean \texttt{top-1} accuracy over all the perturbations divided by the clean \texttt{top-1}.
 
\section{Open Questions \& Limitations}
\label{app:open}
 
In this section, we present several remaining open issues and highlight the limitations of the proposed framework. Moreover, future work must address these issues both theoretically and empirically to further investigate the credibility of asymmetric-based architectures in achieving integrated robustness under a variety of domains.

\textbf{Global strict-convexity in \gls{tide}:} Given that Theorem~\ref{thm:pgd-dale} assumes $\mathcal L$ is $\mu$-strongly convex over $\mathcal W_{\mathrm{Dale}}$, which does not hold for the full non-convex training loss, the $\mathcal O(\log 1/\epsilon)$ convergence rate carries over only locally to a strict local minimum. Therefore, at most the convergence guarantee applies only to a local minimum with a positive-definite Hessian.

\textbf{Exact \gls{lds} for \gls{tide} at convergence:} Corollary~\ref{cor:lds-test} is sufficient for condition monitoring at runtime, and for \texttt{ImageNet-1K}, it has been observed that the soft spectral discussed in Appendix~\ref{app:add:stability} stabilizes at its target Perron radii throughout training, which is consistent with \gls{lds} being at convergence. However, we do not prove that this must hold, and the symmetrized effective matrix $\tfrac12(W_{\mathrm{eff}}+W_{\mathrm{eff}}^\top)-I$ must be negative-definite. Thus, leaving the training-time \gls{lds} enforcement as an open issue that requires further investigation.

\textbf{Balance regime:} Given that the Definition~\ref{def:balance} is formal, and not derived directly from first principles for \gls{tide}, it remains open whether \gls{tide} dynamics can drive the model into the region defined by \cite{vanvreeswijk1996chaos}. 

\textbf{Rotation equivariance of \gls{hrf}:} As noted in Remark~\ref{rem:rot-loss}, the implemented \gls{hrf} backbones are not rotation equivariant, and only Stage zero of the \gls{hrf} backbone under \gls{dog} initialization can be considered as rotation equivariant. To achieve rotation equivariance, a multi-scale bank for the shallow \gls{hrf} and a per-stage bank for the deep \gls{hrf} are required. Moreover, a rotation-equivariant deep \gls{hrf} would need additional steerable filter banks in all its residual stages.

\textbf{Internal computation steps:} It has been observed that increasing $T$ as in Appendix~\ref{app:add:niter} causes instability rather than improvement. A principled scaling of $T$ with sequence length requires further investigation for a more adaptive processing of input data.

\textbf{Limitations:} The main limitation of both \gls{tide} and \gls{ctm} is their required sequential feedforward data processing, which prohibits shorter runtime and prevents more robust investigation. Therefore, further investigation is required to assess the feasibility of leveraging neuromorphic hardware and its dynamics, not only to enhance data preprocessing but also to implement a Look-up-Table-based approach to further reduce the training duration of the proposed framework. Given that we have already lowered the training boundary by around $50$\,\% to achieve results comparable to \gls{ctm}, indicating the soundness of our approach, further studies are still needed to assess the reachability of an even sparser representation. Additionally, due to the lack of comparable implementations using similar approaches, detailed benchmarking within the same architectural family remains limited and challenging. Although \gls{ctm} uses LSTM as its baseline, it is not a direct comparison; however, their results can be directly translated to \gls{tide} given that we use \gls{ctm} as our baseline.
 
\end{document}